\documentclass[english]{article}

\usepackage{microtype}
\usepackage{graphicx}
\usepackage{subfigure}
\usepackage{booktabs} 
\usepackage{amsmath} 
\usepackage{amssymb}  
\usepackage{amsthm} 
\usepackage{color}
\usepackage{xcolor}
\usepackage{caption}
\newcommand{\norm}[1]{\left\lVert#1\right\rVert}
\newtheorem{theorem}{Theorem}

\newtheorem{lemma}{Lemma}
\newtheorem{remark}{Remark}

\newtheorem{proposition}{Proposition}
\newtheorem{assumption}{Assumption}
\usepackage{booktabs}
\usepackage{algorithm}
\usepackage{algorithmic}
\usepackage{hyperref}
\usepackage{cleveref}

\usepackage[round]{natbib}

\usepackage{apalike}

\usepackage{geometry}
\hypersetup{
	colorlinks,linkcolor=red,anchorcolor=blue,citecolor=blue}
\usepackage{multirow}

\usepackage{authblk}

\makeatletter
\newcommand{\printfnsymbol}[1]{%
	\textsuperscript{\@fnsymbol{#1}}%
}
\makeatother

\newcommand{\lTV}[1]{\left\|#1\right\|_{TV}}

\newcommand{\mf}{\mathcal F}

\newcommand{\lone}[1]{\left|#1\right|}

\allowdisplaybreaks[3]

\usepackage{hyperref}

\begin{document}

\title{Non-asymptotic Convergence of Adam-type Reinforcement Learning Algorithms under Markovian Sampling}

\author[1]{Huaqing Xiong\thanks{equal contribution}}
\author[1]{Tengyu Xu\printfnsymbol{1}}
\author[1]{Yingbin Liang}
\author[2]{Wei Zhang}
\affil[1]{Department of Electrical and Computer Engineering, The Ohio State University}
\affil[2]{Department of Mechanical and Energy Engineering, Southern University of Science and Technology}
\affil[ ]{\{xiong.309, xu.3260, liang.889\}@osu.edu;\quad zhangw3@sustech.edu.cn}

\date{}

\maketitle




\vspace{-1cm}
\begin{abstract}
Despite the wide applications of Adam in reinforcement learning (RL), the theoretical convergence of Adam-type RL algorithms has not been established. This paper provides the first such convergence analysis for two fundamental RL algorithms of policy gradient (PG) and temporal difference (TD) learning that incorporate AMSGrad updates (a standard alternative of Adam in theoretical analysis), referred to as PG-AMSGrad and TD-AMSGrad, respectively. Moreover, our analysis focuses on Markovian sampling for both algorithms. We show that under general nonlinear function approximation, PG-AMSGrad with a constant stepsize converges to a neighborhood of a stationary point at the rate of $\mathcal{O}(1/T)$ (where $T$ denotes the number of iterations), and with a diminishing stepsize converges exactly to a stationary point at the rate of $\mathcal{O}(\log^2 T/\sqrt{T})$. Furthermore, under linear function approximation, TD-AMSGrad with a constant stepsize converges to a neighborhood of the global optimum at the rate of $\mathcal{O}(1/T)$, and with a diminishing stepsize converges exactly to the global optimum at the rate of $\mathcal{O}(\log T/\sqrt{T})$. Our study develops new techniques for analyzing the Adam-type RL algorithms under Markovian sampling.
\end{abstract}

\section{Introduction}\label{sec:intro}

Reinforcement learning (RL) aims to study how an agent learns a policy through interacting with its environment to maximize the accumulative reward for a task. RL has so far  accomplished tremendous success in various applications such as playing video games~\cite{mnih2013playing}, bipedal walking~\cite{castillo2018reinforcement} and online advertising~\cite{pednault2002sequential}, to name a few. 
In general, there are two widely used classes of RL algorithms: policy-based methods and value function based methods.

For the first class, policy gradient (PG) \cite{sutton2000policy} is a basic algorithm which has motivated many advanced policy-based algorithms including actor-critic~\cite{konda2000actor}, DPG~\cite{silver2014deterministic}, TRPO~\cite{schulman2015trust}, PPO~\cite{schulman2017proximal}, etc. The idea of the vanilla PG \cite{sutton2000policy} is to parameterize the policy and optimize a target accumulated reward function by (stochastic) gradient descent.
The asymptotic convergence and finite-time (i.e., non-asymptotic) analysis have been characterized for PG in various scenarios, which will be further discussed in \Cref{sec:relatedwork}.

For the second class of value function based algorithms, temporal difference (TD) learning~\cite{sutton1988learning} is a fundamental algorithm which has motivated more advanced algorithms such as Q-learning~\cite{watkins1992q},  SARSA~\cite{rummery1994line}, etc. The vanilla TD \cite{sutton1988learning} typically parameterizes the value function of an unknown policy and iteratively finds the true value function or its estimator by following the (projected) Bellman operation, which can also be interpreted to be analogous to a stochastic gradient descent (SGD) update. 
The asymptotic convergence and finite-time analysis have been established for TD in various scenarios, which will be discussed in \Cref{sec:relatedwork}. 
Despite extensive exploration, all the existing theoretical studies of PG and TD have focused on algorithms that adopt the SGD-type updates without adaption on the stepsize. In practice, however, the adaptive momentum estimation (Adam) method \cite{kingma2014adam} has been more commonly used in RL \cite{bello2017neural,stooke2018accelerated}. It is so far unclear whether RL algorithms that incorporate the Adam-type updates have provable converge guarantee.
{\em The goal of this paper is to provide the first non-asymptotic analysis to characterize the convergence rate of the Adam-type PG and TD algorithms. }
In addition, we develop new technical tools to analyze the Adam-type algorithms under {\em Markovian sampling}. Such analysis is not available in the existing studies of the Adam-type algorithms in optimization which usually assume independent and identically distributed (i.i.d.) sampling.



\subsection{Our Contribution}


The main contribution of the paper lies in establishing the first non-asymptotic analysis for the Adam-type PG and TD algorithms under the Markovian sampling model.

From the technical standpoint, although the existing analysis of the Adam-type algorithms in conventional optimization may provide useful tools, analysis of such algorithms in RL problems has new challenges. (1) Samples in TD learning are drawn from an unknown transition kernel in a non-i.i.d.\ manner, and hence a bias of the gradient estimation arises. Such a bias is also dependent on the AMSGrad updates, which further complicates the analysis. We develop techniques to bound the bias by the stepsize even under adaptive momentum updates.
(2) The sampling process in PG (with a nonlinear approximation architecture) is not only non-i.i.d., but also follows time-varying parametric policies. Thus, it is even more challenging to bound such a bias with the AMSGrad updates. In response to this, we develop new techniques to characterize the bias error which can also be controlled by the stepsize.


As for the results, we provide the first convergence guarantee for RL algorithms (including PG and TD) incorporating the update rule of AMSGrad (referred to as PG-AMSGrad and TD-AMSGrad, respectively), and our techniques also lead to an improved result for the vanilla PG. (1) First, we show that under general nonlinear function approximation, PG-AMSGrad with a constant stepsize\footnote{The ``stepsize" here refers to the basic stepsize $\alpha$ in the AMSGrad update \eqref{eq:AMSGrad4}. The overall learning rate of the algorithm is determined by the basic stepsize $\alpha$, hyperparameters $\beta_1$ and $\beta_2$, and the first and second moments of gradients as given in \eqref{eq:AMSGrad1}-\eqref{eq:AMSGrad4}, and is hence adaptive during the AMSGrad iteration.} converges to a neighborhood of a stationary point at a rate of $\mathcal{O}(1/T)$ (where $T$ denotes the number of iterations), and with a diminishing stepsize converges exactly to a stationary point at a rate of $\mathcal{O}(\log^2 T/\sqrt{T})$. (2) Furthermore, under linear function approximation, TD-AMSGrad with a constant stepsize converges to a neighborhood of the global optimum at a rate of $\mathcal{O}(1/T)$, and with a diminishing stepsize converges exactly to the global optimum at a rate of $\mathcal{O}(\log T/\sqrt{T})$. (3) In addition, by adapting our technical tools to analyze the vanilla PG with the SGD update under Markovian sampling, we obtain an orderwisely better computational complexity than the existing works.

\subsection{Related Work}\label{sec:relatedwork}

Due to the rapidly growing theoretical studies on RL, we only review the most relevant studies below.

\textbf{Convergence analysis of PG:} 
Asymptotic convergence of PG based on stochastic approximation (SA) was established in \cite{williams1992simple,baxter2001infinite,sutton2000policy,kakade2002natural,pirotta2015policy,tadic2017asymptotic}. In specific RL problems such as LQR, PG has been proved to converge to the optimal policy \cite{fazel2018global,malik2018derivative,tu2018gap}. Under specific policy function approximation classes, \cite{bhandari2019global,agarwal2019optimality} also showed that PG can find the optimal policy. Under the general nonconvex approximation, \cite{shen2019hessian,papini2018stochastic,papini2019smoothing,xu2019improved,xu2020sample} characterized the convergence rate for PG and variance reduced PG to a stationary point in the finite-horizon scenario, and \cite{zhang2019global,karimi2019non} provided the convergence rate for PG in the infinite-horizon scenario. \cite{wang2019neural} studied PG with neural network function approximation in an overparameterized regime. Convergence analysis has been also established for the variants of PG, such as actor-critic algorithms \citep{bhatnagar2008incremental,bhatnagar2009natural,bhatnagar2010actor,luo2019natural,yang2019provably,fu2020actor}, TRPO \citep{shani2019adaptive}, TRPO/PPO \citep{liu2019neural}, etc.
This paper studies the infinite-horizon scenario, but focuses on the Adam-type PG, which has not been studied in the literature. Our analysis is also applicable to other variants of PG.

\textbf{Convergence analysis of TD:} Originally proposed in \cite{sutton1988learning}, TD learning with function approximation aroused great interest in analyzing its convergence. While a general TD may not converge as pointed out in \cite{baird1995residual,gyorfi1996averaged}, the authors in \cite{tsitsiklis1997analysis} provided conditions to ensure asymptotic convergence of TD with linear function approximation under i.i.d.\ sampling. Other results on asymptotic convergence using the tools from linear SA were provided in \cite{kushner2003stochastic,benveniste2012adaptive}. Non-asymptotic convergence was established for TD under i.i.d.\ sampling in, e.g.,  \cite{dalal2018finite,bhandari2018finite,lakshminarayanan2018linear}, and under Markovian sampling in, e.g., \cite{bhandari2018finite,srikant2019finite,hu2019characterizing}. The convergence rate of TD with nonlinear function approximation has recently been studied in \cite{cai2019neural} for overparameterized neural networks using i.i.d.\ samples.
In contrast to the aforementioned work on TD with the SGD updates, this paper studies Adam-type TD under Markovian sampling.

\textbf{Adaptive reinforcement learning algorithms:} Adaptivity has been applied to RL algorithms to improve the performance.  \citep{shani2019adaptive} used an adaptive proximity term to study the convergence of TRPO. An adaptive batch size was adopted to improve the policy performance \citep{papini2017adaptive} and reduce the variance \citep{ji2019faster} of PG. The abovementioned papers did not study how adaptive learning rates can affect the performance of PG or TD. More recently, a concurrent work \citep{sun2020adaptive} which is posted a few days after our paper provided an analysis of TD(0) and TD($\lambda$) when incorporating adaptive gradient descent (AdaGrad) updates. However, this paper focuses on a more popular class of adaptive algorithms: Adam-type methods, and provide the first convergence guarantee when they are applied to PG and TD.



\textbf{Convergence analysis of Adam-type algorithms in conventional optimization:} Adam was proposed in~\cite{kingma2014adam} for speeding up the training of deep neural networks, but the vanilla Adam was shown not to converge in \cite{reddi2019convergence}. Instead AMSGrad was proposed as a slightly modified version to justify the theoretic performance of Adam and its regret bounds were characterized in \cite{reddi2019convergence,tran2019convergence} for online convex optimization. Recently, Adam/AMSGrad was proved to converge to a stationary point for certain general nonconvex optimization problems~\cite{zou2018sufficient,zhou2018convergence,chen2018convergence}.
Our study provides the first convergence guarantee for the Adam-type algorithms in the RL settings, where the Markovian sampling poses the key difference and challenge in our analysis from those in conventional optimization.

\section{Preliminary}\label{sec:pre}

In this section, we provide the necessary background for the problems that we study in this paper.

\subsection{Markov Decision Process}

We consider the standard RL settings, where an agent interacts with a (possibly stochastic) environment (e.g. process or system dynamics). This interaction is usually modeled as a discrete-time discounted Markov Decision Processes (MDPs), described by a tuple $(\mathcal{S},\mathcal{A},\mathbb{P},R, \gamma,\zeta)$, where $\mathcal{S}$ is the state space, $\mathcal{A}$ is the action space, $\mathbb P:\mathcal{S}\times \mathcal{A} \times \mathcal{S}\mapsto [0,1]$ is the probability kernel for the state transitions, e.g., $\mathbb{P}(\cdot|s, a)$ denotes the probability distribution of the next state given the current state $s$ and action $a$. In addition, $R: \mathcal{S}\times \mathcal{A}\mapsto[0,R_{\max}]$ is the reward function mapping station-action pairs to a bounded subset of $\mathbb{R}$, $\gamma\in (0,1)$ is the discount factor, and $\zeta$ denotes the initial state distribution. The agent's decision is captured by the policy $\pi:=\pi(\cdot|s)$ which characterizes the density function over the action space $\mathcal{A}$ at the state $s\in\mathcal{S}$. 
We denote $\nu:=\nu_{\pi}$ as the stationary distribution of the transition kernel $\mathbb P$ for a given $\pi$. In addition,
we define the $\gamma$-discounted stationary visitation distribution of the policy $\pi$ as $\mu_{\pi}(s)=\sum_{t=1}^\infty \gamma^t P_{\zeta,\pi}(s_t=s)$. Further, we denote $\mu_{\pi}(s,a) =\mu_{\pi}(s)\pi(a|s) $ as the (discounted) state-action visitation distribution.

In this paper, we assume that the state space $\mathcal{S}$ is countably infinite and the action space $\mathcal{A}$ is finite with possibly large cardinality $|\mathcal{A}|$. Hence, we estimate the policy and the value function corresponding to a certain unknown policy by parameterized function classes as we introduce in Section \Cref{sec:AMSGradPGalg,sec:AMSGradTDalg}, respectively.

\subsection{Update Rule of AMSGrad}\label{subsec:AMSGrad}

Although Adam \cite{kingma2014adam} has gained great success in practice since being proposed, it is shown not to converge even in the simple convex setting \cite{reddi2019convergence}. Instead, a slightly modified version called AMSGrad~\cite{reddi2019convergence} is widely used to understand the success of adaptive momentum optimization algorithms. 
Given a gradient $g_t$ at time $t$, the generic form of AMSGrad is given by
\begin{align}
    m_{t} &= (1-\beta_{1})m_{t-1} + \beta_{1} g_t; \label{eq:AMSGrad1}\\
    v_{t} &= (1-\beta_2)\hat v_{t-1} + \beta_2 g_t^2;\label{eq:AMSGrad2}\\
    \hat v_t &= \max (\hat{v}_{t-1}, v_t),\ \hat{V}_t = diag(\hat{v}_{t,1},\dots,\hat{v}_{t,d});\label{eq:AMSGradhatV}\\
    \theta_{t+1} &= \theta_t - \alpha_t \hat{V}_t^{-\frac{1}{2}}m_t,\label{eq:AMSGrad4}
\end{align}
where $\alpha_t$ is the stepsize, and $\beta_1,\beta_2$ are two hyper-parameters. In addition, $m_t$, $v_t$ given in \eqref{eq:AMSGrad1} and \eqref{eq:AMSGrad2} are viewed as the estimation of the first moment and second moment, respectively, which play important roles in adapting the learning rate as in \eqref{eq:AMSGrad4}. Compared to Adam, the main difference of AMSGrad lies in \eqref{eq:AMSGradhatV}, which guarantees the sequence $\hat{v}_t$ to be non-decreasing whereas Adam does not require this. Such difference is considered to be a central reason causing the non-convergent behavior of Adam \cite{reddi2019convergence,chen2018convergence}.  

\subsection{Notations} 
We use $\|x\|:=\|x\|_{2}=\sqrt{x^Tx}$ to denote the $\ell_2$-norm of a vector $x$, and use $\|x\|_{\infty}=\underset{i}{\max}|x_i|$ to denote the infinity norm. When $x,y$ are both vectors, $x/y, xy, x^2, \sqrt{x}$ are all calculated in the element-wise manner, which are used in the update of AMSGrad. We denote $[n]=\{1,2,\dots,n\}$, and $\lceil x\rceil\in \mathbb{Z} $ as the integer such that $\lceil x\rceil - 1 \leq x < \lceil x\rceil$.

\section{Convergence of PG-AMSGrad under Markovian Sampling}\label{sec:AMSGradPG}

In this section, we study the convergence of an Adam-type policy gradient algorithm (PG-AMSGrad) with nonlinear function approximation and under non-i.i.d.\ sampling. 

\subsection{Policy Gradient and PG-AMSGrad}\label{sec:AMSGradPGalg}

Consider a reinforcement learning problem which aims to find a policy that maximizes the expected accumulative reward. We assume that the policy is parameterized by $\theta\in\mathbb{R}^d$ and form a policy class $\Pi := \{\pi_\theta |\theta\in\mathbb{R}^d\}$, which in general is a nonlinear function class. The policy gradient method is usually used to solve the following {\em infinite-horizon} optimization problem:
\begin{equation}\label{eq:pgLoss}
    \underset{\theta\in\mathbb{R}^d}{\text{maximize}}\quad
    J(\theta) = \mathbb{E}\left[\sum_{t=1}^{\infty} \gamma^t R(s_t, a_t) \right].
\end{equation}
The gradient of $J(\theta)$ with respect to $\theta$ is captured by the policy gradient theorem for infinite-horizon MDP with the discounted reward \cite{sutton2000policy}, and is given by
\begin{equation}\label{eq:expetedPG}
    \nabla _{\theta} J(\theta) = \underset{\mu_{\theta}}{\mathbb{E}}\left[ Q^{\pi_\theta}(s,a)\nabla_\theta \log(\pi_\theta (a|s) )  \right],
\end{equation}
where the expectation is taken over the discounted state-action visitation distribution $\mu_{\theta} := \mu_{\pi_\theta}(s,a)$, and $Q^{\pi_\theta}(s,a)$ denotes the Q-function for an initial state-action pair $(s,a)$ defined as 
\begin{equation}\nonumber
    Q^{\pi_\theta}(s,a) = \mathbb{E}\left[\sum_{t=1}^{\infty} \gamma^t R(s_t, a_t) \Big\lvert s_1=s, a_1=a\right].
\end{equation}
In addition, we refer to $\nabla_{\theta}\log \pi_{\theta}(a|s)$ as the score function corresponding to the policy $\pi_{\theta}$.

Since the transition probability is unknown, the policy gradient in \eqref{eq:expetedPG} needs to be estimated via sampling. 
The Q-function $Q^{\pi_\theta}(s,a)$ and the score function are typically estimated by independent samples.
First, at each time $t$, we draw a sample trajectory to provide an estimated Q-function $\hat Q^{\pi_\theta}(s,a)$ based on the algorithm EstQ \cite{zhang2019global} (see Algorithm \ref{alg:EstQ} in Appendix \ref{sec:appdixEstQ} for details). Such an estimator has been shown to be unbiased \cite{zhang2019global}. That is, if we use $O^q$ to denote the randomness including the samples and horizon in EstQ, then we have
\begin{equation}
    \underset{O^q}{\mathbb{E}} \hat Q^{\pi_\theta}(s,a) = Q^{\pi_\theta}(s,a),\quad\forall (s,a).
\end{equation}
Second, we take samples $\{(s_t,a_t)\}$ to estimate the score function by following the policy $\pi_{\theta_t}$ and the transition function proposed in \cite{konda_2002} which is given by 
$$\hat P(\cdot|s_t,a_t) = \gamma\mathbb{P}(\cdot|s_t,a_t) + (1-\gamma)\zeta(\cdot),$$ 
where $\zeta(\cdot)$ denotes the initial distribution and $\mathbb{P}$ is the transition probability from the original MDP. It has been shown in \cite{konda_2002} that such a transition probability guarantees the MDP to converge to the state-action visitation distribution.

Then, the gradient estimator to approximate $\nabla_{\theta} J(\theta)$ at time $t$ is given by
\begin{equation}\label{eq:PGgradient}
    g_t\!:=\!g(\theta_t;s_t,a_t)\! =\! \hat Q^{\pi_{\theta_t}}(s_t,a_t)\!\nabla_{\theta_t} \log(\pi_{\theta_t} (a_t|s_t) ).
\end{equation}

We then apply such a gradient estimator $g_t$ to update the policy parameter by the AMSGrad update given in \eqref{eq:AMSGrad1}-\eqref{eq:AMSGrad4}, and obtain PG-AMSGrad as in Algorithm \ref{alg:AMSGradPG}.

\begin{algorithm}[h]
 	\caption{PG-AMSGrad} \label{alg:AMSGradPG} 
 	\begin{algorithmic}[1]
 		\STATE 	{\bf Input:}   $\alpha, \theta_{1}, \beta_1, \beta_2, m_0 = 0, \hat v_0 = 0, t=1, s_1\sim \zeta(\cdot), a_1\sim\pi_{\theta_1}(\cdot|s)$.
		\WHILE{ not converge }
		\STATE Assign stepsize $\alpha_t$.
		\STATE Obtain $\hat Q^{\pi_{\theta_t}}(s_t,a_t)\leftarrow \text{EstQ}(s_t,a_t,\theta_t)$.
		\STATE Compute $g_t=\hat Q^{\pi_{\theta_t}}(s_t,a_t)\nabla_{\theta_t} \log(\pi_{\theta_t} (a_t|s_t) )$.
		\STATE $m_{t} = (1-\beta_{1})m_{t-1} + \beta_{1} g_t$.
		\STATE $v_{t} = (1-\beta_2)\hat v_{t-1} + \beta_2 g_t^2$.
		\STATE $\hat v_t = \max (\hat{v}_{t-1}, v_t),\ \hat{V}_t = diag(\hat{v}_{t,1},\dots,\hat{v}_{t,d}) $.
		\STATE $\theta_{t+1} = \theta_t - \alpha_t \hat{V}_t^{-\frac{1}{2}}m_t$.
		\STATE $t\leftarrow t+1$.
		\STATE Sample $s_{t}\sim\hat P(\cdot|s_{t-1},a_{t-1}),a_{t}\sim\pi_{\theta_t}(\cdot|s_{t})$.
		\ENDWHILE
 	\end{algorithmic}
\end{algorithm}

We note that the gradient estimator obtained in \eqref{eq:PGgradient} is {\em biased}, because the score function is estimated by a sequence of Markovian samples. We will show that such a biased gradient estimator is in fact computationally more efficient than the unbiased estimator used in the existing literature \cite{zhang2019global} (see Section \ref{subsec:comparisonPG}). Our main technical novelty here lies in developing techniques to analyze the biased estimator under the AMSGrad update for PG.

\subsection{Technical Assumptions}

In the following, we specify some technical assumptions in our convergence analysis.

We consider a general class of parameterized policy functions that satisfy the following assumption.
\begin{assumption}\label{asp:policyLip}
    Assume that the parameterized policy $\pi_{\theta}$ is differentiable with respect to $\theta$, and the score function $\nabla_{\theta}\log \pi_{\theta}(a|s)$ corresponding to $\pi_{\theta}(\cdot|s)$ exists. In addition, we assume both the policy function and the score function are Lipschitz continuous with the parameters $L_\pi$ and $L$, respectively, i.e., for all $ \theta_1,\theta_2\in\mathbb R^d$,
    $$|\pi_{\theta_1} (a|s) - \pi_{\theta_2} (a|s)| \leq L_\pi \norm{\theta_1 - \theta_2};
    $$
    $$\norm{\nabla_{\theta_1}\log (\pi_{\theta_1}(a|s))\! -\! \nabla_{\theta_2}\log (\pi_{\theta_2}(a|s))} \leq L\norm{\theta_1 - \theta_2}.
    $$
    Further, we also assume that the score function is uniformly bounded by $c_\Theta$ for any $(s,a)\in\mathcal{S}\times\mathcal{A}$, i.e.,
    $$ \norm{\nabla_{\theta}\log (\pi_{\theta}(a|s))} \leq c_\Theta, \quad\forall \theta.
    $$
\end{assumption}
This assumption is standard in the existing literature that studies PG with nonconvex function approximation \cite{zhang2019global,xu2019improved,papini2018stochastic}. 

In Algorithm \ref{alg:AMSGradPG}, we sample a data trajectory using the transition kernel $\hat P$ and the policy $\pi_{\theta_t}$. Such a sequence of samples are non-i.i.d.\ and follow a Markovian distribution. 
We assume that the MDP and the policies we consider satisfy the following standard mixing property.
\begin{assumption}\label{asp:markov}
    For any $\theta\in\mathbb R^d$, there exist constant $\sigma > 0$ and $\rho \in (0,1)$ such that
    $$ \underset{s\in\mathcal{S}}{\sup}\ \lTV{P(s_t\in \cdot|s_1 = s) - \mu_\theta(\cdot)}\leq \sigma\rho^t \quad \forall t,
    $$
    where $\lTV{\mu_1-\mu_2}$ denotes the total-variation norm (or the total-variation distance between two probability measures $\mu_1$ and $\mu_2$). 
\end{assumption}
This assumption holds for irreducible and aperiodic Markov chains \cite{mitrophanov2005sensitivity}, and is widely adopted in the theoretical analysis of RL algorithms under Markovian sampling settings \cite{bhandari2018finite,Chen2019finiteQ,zou2019finite,karimi2019non}. 


\subsection{Convergence of PG-AMSGrad}

In this section, we provide the convergence analysis of PG-AMSGrad as given in Algorithm \ref{alg:AMSGradPG}. We first consider the case with a constant stepsize, and then provide the result with a diminishing stepsize.

Although AMSGrad has been studied in conventional optimization, our analysis of PG-AMSGrad mainly deals with the following new challenges arising in RL. First, samples here are generated via an MDP and distributed in a non-i.i.d.\ fashion. Thus the gradient estimator is biased and we need to control the bias with a certain upper bound scaled by the stepsize. Second, the sampling distribution also changes over time, which causes additional complication. Thus, our technical development mainly handles the above two challenges under the adaptive momentum update rule of AMSGrad. We next provide the convergence results that we obtain and relegate all the proofs to the appendices. 

We first provide the Lipschitz properties for the true policy gradient and its estimator, which are useful for establishing the convergence. Recall that in Algorithm \ref{alg:AMSGradPG}, the gradient estimator $g_t=\hat Q^{\pi_{\theta_t}}(s_t,a_t)\nabla_{\theta_t} \log(\pi_{\theta_t} (a_t|s_t) )$ at time $t$ is obtained by using the Q-function estimator generated by the EstQ algorithm (see Appendix \ref{sec:appdixEstQ}). Note that $\hat Q^{\pi_\theta}(s,a)$ is an unbiased estimator of $Q^{\pi_\theta}(s,a)$ for all $(s,a)$ \cite{zhang2019global}, and the samples for estimation are independent of those for other steps in PG-AMSGrad except the initial sample. Taking expectation over the randomness in EstQ at time $t$ (denoted as $O^q_t$), we obtain an estimator $\nabla_{\theta_t}\tilde{J}(\theta_t;s_t,a_t)$ defined as
\begin{align}
    &\nabla_{\theta_t}\tilde{J}(\theta_t;s_t,a_t) := \underset{O^q_t}{\mathbb{E}} [g_t]= \underset{O^q_t}{\mathbb{E}}\left[\hat Q^{\pi_{\theta_t}}(s_t,a_t)\nabla_{\theta_t} \log(\pi_{\theta_t} (a_t|s_t) )\right]= Q^{\pi_{\theta_t}}(s_t,a_t)\nabla_{\theta_t} \log(\pi_{\theta_t} (a_t|s_t) ). \label{eq:tildeJGra}
\end{align}


We next obtain the Lipschitz properties of $\nabla_{\theta}\tilde{J}(\theta;s,a)$ and $\nabla_{\theta} J(\theta)$ in the following lemma.
\begin{lemma}(Lipschitz property of policy gradient)\label{lem:pgLip}
Under Assumptions \ref{asp:policyLip} and \ref{asp:markov}, the policy gradient $\nabla_{\theta} J(\theta)$ defined in \eqref{eq:expetedPG} is Lipschitz continuous with the parameter $c_J$, i.e., $\forall \theta_1, \theta_2 \in \mathbb{R}^d$,
\begin{equation}
    \norm{ \nabla_{\theta_1} J(\theta_1) - \nabla_{\theta_2} J(\theta_2) }\leq c_J \norm{\theta_1 - \theta_2},
\end{equation}
where the constant coefficient $c_J = \frac{R_{\max} L}{1-\gamma} + \frac{(1+c_{\Theta}) R_{\max}}{1-\gamma}\cdot |\mathcal{A}|L_\pi\left( 1 + \lceil \log_{\rho}\sigma^{-1} \rceil + \frac{1}{1-\rho} \right)$. Further, the policy gradient estimator $\nabla_{\theta} \tilde{J}(\theta;s,a)$ defined in \eqref{eq:tildeJGra} is also Lipschitz continuous with the parameter $c_{\tilde{J}}$, i.e., $\forall \theta_1, \theta_2 \in \mathbb{R}^d, \forall (s,a)\in\mathcal{S}\times\mathcal{A}$,
\begin{equation}\label{eq:tildeJLip}
    \norm{ \nabla_{\theta_1} \tilde{J}(\theta_1;s,a)\! -\! \nabla_{\theta_2} \tilde{J}(\theta_2;s,a) }\leq c_{\tilde{J}} \norm{\theta_1\! -\! \theta_2},
\end{equation}
where $c_{\tilde{J}} = \frac{R_{\max} L}{1-\gamma} + c_{\Theta}|\mathcal{A}|L_\pi\left( 1 + \lceil \log_{\rho}\sigma^{-1} \rceil + \frac{1}{1-\rho} \right) $.
\end{lemma}

The following theorem characterizes the convergence of PG-AMSGrad with a constant stepsize. Recall that the stepsize refers to the parameter $\alpha$ in the AMSGrad update \eqref{eq:AMSGrad4}, not the overall learning rate of the algorithm.
\begin{theorem}\label{thm:AMSGradPGcons}
(Convergence of PG-AMSGrad with constant stepsize)  Fix $\beta_1,\beta_2$ in Algorithm \ref{alg:AMSGradPG}. Initialize Algorithm \ref{alg:AMSGradPG} such that $|g_{1,i}|\geq G_0$ for all $i\in[d]$ with some $G_0>0$. Suppose Assumptions \ref{asp:policyLip} and \ref{asp:markov} hold. Let $\alpha_t=\alpha$ for $t=1,\dots,T$. Then after running $T$ steps of PG-AMSGrad as given in Algorithm \ref{alg:AMSGradPG}, we have:
\begin{align*}
    \underset{t\in [T]}{\min}\mathbb{E}\left[\norm{\nabla_{\theta_t} J(\theta_t)}^2\right] \leq  \frac{C_1}{T} + \alpha C_2,
\end{align*}
where
\begin{align*}
    &C_1 = \frac{G_{\infty}\mathbb{E}[J(z_1)]}{\alpha} +  \frac{d G_{\infty}^3}{G_0(1-\beta_1)}  +  \frac{2G_\infty\tau^*}{G_0} G_{\infty}^2 +  \frac{dc_J\alpha G_\infty(3\beta_1^2+2(1-\beta_1)(1-\beta_1/\beta_2))}{(1-\beta_1)(1-\beta_2)(1-\beta_1/\beta_2)},\\
    &C_2 = \frac{G_\infty^3}{G_0} \left[ \frac{(3c_{J}\!+\!c_{\tilde{J}}) \tau^*}{G_0} \!+\! d \!+\! \frac{d L_\pi G_{\infty} (2\tau^*\!+\!(\tau^*)^2)}{2G_0} \right]
\end{align*}
with $c_J,c_{\tilde{J}}$ defined in Lemma \ref{lem:pgLip}, $\tau^*=\min\{\tau: m\rho^\tau \leq \alpha\} $ and $G_\infty = \frac{c_{\Theta}R_{\max}}{1-\gamma} $.
\end{theorem}

Theorem \ref{thm:AMSGradPGcons} indicates that under the constant stepsize, PG-AMSGrad converges to a neighborhood of a stationary point at a rate of $\mathcal{O}\left(\frac{1}{T}\right)$. The size of the neighborhood can be controlled by the stepsize $\alpha$. One can observe that $\alpha$ controls a tradeoff between the convergence rate and the convergence accuracy. Decreasing $\alpha$ improves the convergence accuracy, but slows down the convergence, 
since the coefficient $C_1$ contains $\alpha$ in the denominator. 
To balance such a tradeoff, we set the stepsize $\alpha_t=\frac{1}{\sqrt{T}}$. In this case, the mixing time becomes $\tau^*=\mathcal{O}(\log T)$ and thus PG-AMSGrad converges to a stationary point with a rate of $\mathcal{O}\left( \frac{\log^2 T}{\sqrt{T}} \right)$.

In the following, we adopt a diminishing stepsize to eliminate the convergence error and obtain the exact convergence.

\begin{theorem}\label{thm:AMSGradPGdimi}
(Convergence of PG-AMSGrad with diminishing stepsize) Suppose the same conditions of Theorem \ref{thm:AMSGradPGcons} hold, and let $\alpha_t=\frac{\alpha}{\sqrt{t}}$ for $t=1,\dots,T$. Then running $T$ steps of PG-AMSGrad as given in Algorithm \ref{alg:AMSGradPG}, we have:
\begin{align*}
    \underset{t\in [T]}{\min}\mathbb{E}\left[\norm{\nabla_{\theta_t} J(\theta_t)}^2\right] \leq \frac{C_1}{T} + \frac{C_2}{\sqrt{T}},
\end{align*}
where
\begin{align*}
    & C_1 = \frac{f_1 G_\infty}{\alpha} + \frac{2dc_J\alpha G_\infty}{1-\beta_2} + \frac{2\tau^* G_\infty}{G_0} G_{\infty}^2 + \frac{3dc_J\beta_1^2\alpha G_\infty}{(1-\beta_1)(1-\beta_2)(1-\beta_1/\beta_2)}\\
    & C_2 = \frac{R_{\max}G_\infty}{\alpha(1-\gamma)} + \frac{ dG_{\infty}^3}{ G_0(1-\beta_1)}+ \frac{\alpha  G_\infty^3}{G_0} \left[ \frac{2(3c_{J}\!+\!c_{\tilde{J}}) \tau^*}{G_0} + d \left( 1 + \frac{L_\pi G_{\infty} (\tau^*+(\tau^*)^2)}{G_0}  \right)\right]
\end{align*}
with $c_J,c_{\tilde{J}}$ defined in Lemma \ref{lem:pgLip}, $\tau^*\!=\!\min\{\tau: m\rho^\tau\! \leq\! \alpha_T\!=\!\frac{\alpha}{\sqrt{T}}\} $ and $G_\infty = \frac{c_{\Theta}R_{\max}}{1-\gamma} $.
\end{theorem}

Under the diminishing stepsize, PG-AMSGrad can converge exactly to a stationary point. Since $\tau^*=\mathcal{O}(\log T)$, the convergence rate is given by $\mathcal{O}\left(\frac{\log^2 T}{\sqrt{T}}\right)$. 

Theorems \ref{thm:AMSGradPGcons} and \ref{thm:AMSGradPGdimi} indicate that under both a constant stepsize and diminishing stepsize, PG-AMSGrad finds a stationary point efficiently with guaranteed convergence. However, there is a tradeoff between the convergence rate and the convergence accuracy. With a constant stepsize, PG-AMSGrad can converge faster but only to a neighborhood of a stationary point whose size is controlled by the stepsize, whereas a diminishing stepsize yields a better convergence accuracy, but attains a lower convergence rate.

\subsection{Implication on Vanilla PG under Markovian Data}\label{subsec:comparisonPG}

Although our focus in this paper is on the Adam-type PG, slight modification of our analysis yields the convergence rate of the vanilla PG under infinite horizon Markovian sampling. In the following, we first present such results and then compare them with two recent studies \cite{zhang2019global,karimi2019non} on the same model to illustrate the novelty of our analysis and benefit of our sampling strategy.




We consider the vanilla PG algorithm that uses the same gradient estimator and sampling strategy as those of Algorithm \ref{alg:AMSGradPG}, but adopts the SGD update (i.e., $\theta_{t+1} = \theta_t-\alpha_t g_t$) rather than the AMSGrad update. We call such an algorithm as PG-SGD. The following proposition characterizes the convergence rate for PG-SGD.
\begin{proposition}\label{prop:PGSGD}
Suppose Assumptions \ref{asp:policyLip} and \ref{asp:markov} hold. After running $T$ steps of PG-SGD with a constant stepsize $\alpha_t=\alpha$ for $t=1,\dots,T$, we have:
\begin{align*}
    \underset{t\in [T]}{\min}\mathbb{E}\left[\norm{\nabla_{\theta_t} J(\theta_t)}^2\right] \leq  \frac{J(\theta_1)/\alpha + 2G_{\infty}^2\tau^*}{T} + \alpha C_1,
\end{align*}
where
\begin{align*}
    &C_1\! =\! G_\infty^2 \left[ \frac{c_J}{2}\! + \! (3c_{J}\!+\!c_{\tilde{J}}) \tau^*\! +\!  1\! +\! \frac{L_\pi G_{\infty} (2\tau^*\!+\!(\tau^*)^2)}{2}  \right],
\end{align*}
with $c_J,c_{\tilde{J}}$ defined in Lemma \ref{lem:pgLip}, $\tau^*=\min\{\tau: m\rho^\tau \leq \alpha\} $ and $G_\infty = \frac{c_{\Theta}R_{\max}}{1-\gamma} $. Furthermore, after running $T$ steps of PG-SGD with a diminishing stepsize $\alpha_t=\frac{1-\gamma}{\sqrt{t}}$ for $t=1,\dots,T$, we have:
\begin{align*}
    \underset{t\in [T]}{\min}\mathbb{E}\left[\norm{\nabla_{\theta_t} J(\theta_t)}^2\right] \leq  \frac{J(\theta_1) + 2(1-\gamma)^2 G_{\infty}^2\tau^*}{(1-\gamma)T} + \frac{C_2}{(1-\gamma)^2\sqrt{T}},
\end{align*}
where
\begin{align*}
    C_2 = R_{\max} + (1-\gamma)^3 G_\infty^2 \left[ c_J +  2(3c_{J}+c_{\tilde{J}}) \tau^* +  1 + L_\pi G_{\infty} (\tau^*+(\tau^*)^2)  \right],
\end{align*}
with $\tau^*\!=\!\min\{\tau: m\rho^\tau\! \leq\! \alpha_T\!=\!\frac{1-\gamma}{\sqrt{T}}\} $.
\end{proposition}

We next compare \Cref{prop:PGSGD} with two recent studies on the same non-i.i.d.\ model. First, \cite{karimi2019non} studied infinite-horizon PG with a biased gradient estimator. 
In their analysis, they did not bound the gradient bias using the stepsize, and hence their convergence has a non-zero error even with a diminishing stepsize. In contrast, 
we obtain a fine-grained bound on the bias and characterize its dependence on the stepsize. We show that PG exactly converges to a stationary point under the diminishing stepsize.

Another closely related study on the infinite-horizon PG under non-i.i.d.\ sampling was by \cite{zhang2019global}, but their algorithm adopts an unbiased gradient estimator at the cost of using more samples. As a comparison,
\Cref{prop:PGSGD} indicates that PG-SGD with a biased gradient estimator attains a better convergence rate and accuracy.
More specifically, under the constant stepsize, \cite[Corollary 4.4]{zhang2019global} showed that their PG algorithm 
converges with an error bound of $\mathcal{O}\left( \frac{1}{(1-\gamma)^5(1-\sqrt{\gamma})^2} \right)$, whereas PG-SGD with a biased gradient estimator
achieves a much smaller error bound $\mathcal{O}\left( \frac{1}{(1-\gamma)^2} \right)$ by taking $\alpha = 1-\gamma$. 
Similarly, under the diminishing stepsize, \cite[Theorem 4.3]{zhang2019global} showed that their PG algorithm converges at a rate of $\mathcal{O}\left( \frac{1}{(1-\gamma)^5(1-\sqrt{\gamma})^2\sqrt{T}} \right)$, 
whereas our PG-SGD converges at a rate of $\mathcal{O}\left( \frac{\log^2(\sqrt{T}/(1-\gamma))}{(1-\gamma)^2\sqrt{T}} \right)$, which is much faster since $\gamma$ is usually close to 1, and $\log T$ is considered to be less influential in practice.

\section{Convergence of TD-AMSGrad under Markovian Sampling}\label{sec:AMSGradTD}

In this section, we adopt AMSGrad to TD learning and analyze its convergence under Markovian sampling. 
The proof techniques of bounding the bias and the nature of the convergence are very different from those of PG-AMSGrad.

\subsection{TD Learning and TD-AMSGrad}\label{sec:AMSGradTDalg}

Policy evaluation is a fundamental task in RL, and often plays a critical role in other algorithms such as PG that we study in Section \ref{sec:AMSGradPG}. The goal of policy evaluation is to obtain an accurate estimation of the accumulated reward function known as the value function $V:\mathcal{S}\mapsto\mathbb{R}$ for a given policy $\pi$ defined as
\begin{equation}\nonumber
    V^\pi(s) = \mathbb{E}\left[\sum_{t=1}^{\infty} \gamma^t R(s_t, a_t) \Big|s_1=s\right].
\end{equation}
The Bellman operator corresponding to a value function $V$ is defined by $\forall s\in\mathcal{S}$,
\begin{equation}\label{eq:bellman}
    (T^\pi V)(s) = \sum_{a\in\mathcal{A}}\pi(a|s)R(s,a) + \gamma \sum_{s'\in\mathcal{S}}\mathbb{P}(s'|s)V(s').
\end{equation}

The key to approximate the value function is the observation that it satisfies the projected bellman equation $(\Pi T^\pi V)(s) = V(s)$ when projected onto some convex space.
Under the function approximation, the value function $V(s)$ is parameterized by $\theta\in\mathbb R^d$ and denoted by $V(s;\theta)$. 
As many recent finite-time analysis of TD \cite{bhandari2018finite,xu2019twotimescale,srikant2019finite}, we consider the linear approximation class of the value function $V(s;\theta)$ defined as
\begin{equation}
    \label{eq:linearApprox}
    V(s;\theta) = \phi(s)^T\theta, 
\end{equation}
where $\theta\in\mathbb{R}^d$, and $\phi:\mathcal{S}\rightarrow \mathbb{R}^d$ is a vector function with the dimension $d$, and the elements of $\phi$ represent the nonlinear kernel (feature) functions. 
Then the vanilla TD algorithm follows a stochastic iterative method 
given by
\begin{equation}\label{eq:PGSGD}
    \theta_{t+1} = \theta_t - \alpha_t g_t,
\end{equation}
where $\alpha_t$ is the stepsize, and $g_t$ is defined as
\begin{align}
    g_t &:= g(\theta_t;s_t,a_t,s_{t+1}) = \left(\phi^T(s_t)\theta_t - R(s_t,a_t) - \gamma\phi^T(s_{t+1})\theta_t \right)\phi(s_t).\label{eq:gradientTD}
\end{align}
Here, $g_t$ serves as a stochastic pseudo-gradient, and is an estimator of the full pseudo-gradient 
given by
\begin{equation}\label{eq:meanGra}
    \bar g(\theta_t)\! =\! \underset{\nu}{\mathbb{E}}\left[ \left(\phi^T(s_t)\theta_t\! - \!R(s_t,\pi(s_t)) \!-\!\gamma\phi^T(s_{t+1})\theta_t \right)\phi(s_t) \right]
\end{equation}
where the expectation is taken over the stationary distribution of the states. We note that $\bar g(\theta_t)$ is not a gradient of a loss function, but plays a similar role as the gradient in the gradient descent algorithm.

Then TD-AMSGrad is obtained by replacing the update \eqref{eq:PGSGD} by the AMSGrad update given in \eqref{eq:AMSGrad1}-\eqref{eq:AMSGrad4} as in Algorithm \ref{alg:AMSGradTD}.

\begin{algorithm}
 	\caption{TD-AMSGrad} \label{alg:AMSGradTD} 
 	\begin{algorithmic}[1]
 		\STATE 	{\bf Input:}   $\alpha, \lambda, \theta_{1}, \beta_1, \beta_2, m_0 = 0, \hat v_0 = 0, s_1\sim \zeta(\cdot)$.
		\FOR{ $t=1, 2, \ldots, T $}
		\STATE Assign $\alpha_t$, $\beta_{1t} = \beta_1 \lambda^t$.
		\STATE Sample $a_t\sim\pi, s_{t+1}\sim\mathbb{P}(\cdot|s_t,a_t)$.
		\STATE Compute $g_t$ as \eqref{eq:gradientTD}.
		\STATE $m_{t} = (1-\beta_{1t})m_{t-1} + \beta_{1t} g_t$.
		\STATE $v_{t} = (1-\beta_2)\hat v_{t-1} + \beta_2 g_t^2$.
		\STATE $\hat v_t = \max (\hat{v}_{t-1}, v_t),\ \hat{V}_t = diag(\hat{v}_{t,1},\dots,\hat{v}_{t,d}) $.
		\STATE $\theta_{t+1} = \Pi_{\mathcal{D}, \hat V_t^{1/4}} \left( \theta_t - \alpha_t \hat{V}_t^{-\frac{1}{2}}m_t\right)$,\\ where $\Pi_{\mathcal{D}, \hat V_t^{1/4}}(\theta') = \underset{\theta\in\mathcal{D}}{\min}\norm{\hat V_t^{1/4}\left(\theta' - \theta\right)}$. 
		\ENDFOR
        \STATE 	{ {\bf Output:} $\frac{1}{T}\sum_{t=1}^T \theta_t $}.
 	\end{algorithmic}
\end{algorithm}

As seen in Algorithm \ref{alg:AMSGradTD}, the state-action pairs are sampled as a trajectory under the transition probability $\mathbb{P}$ with unknown policy $\pi$. Therefore, the samples along the trajectory are dependent, and hence we need to analyze the convergence of TD-AMSGrad under Markovian sampling.

\subsection{Technical Assumptions}\label{subsec:tdAsp}

In this section, we introduce some standard technical assumptions for our analysis.

We first give the following standard assumption on the kernel function \cite{tsitsiklis1997analysis,bhandari2018finite,xu2019twotimescale,Chen2019finiteQ}.

\begin{assumption}\label{asp:boundPhi}
For any state $s\in\mathcal{S}$, the kernel function $\phi:\mathcal{S}\rightarrow\mathbb{R}^{ d}$ is uniformly bounded, i.e., $\norm{\phi(s)} \leq 1,\;\forall s\in\mathcal{S}$. In addition, we define a feature matrix $\Phi$ as
$$ \Phi = \begin{bmatrix}
\phi^T(s_1)\\ \phi^T(s_2)\\ \vdots 
\end{bmatrix} = \begin{bmatrix}
&\phi_1(s_1) &\cdots &\phi_d(s_1) \\ &\phi_1(s_2) &\cdots &\phi_d(s_2)\\ &\vdots &\vdots &\vdots 
\end{bmatrix},
$$
and assume that the columns of $\Phi$ are linearly independent.
\end{assumption}
The boundedness assumption is mild since we can always normalize the kernel functions. 

We next define the following norm of the value function as:
\begin{equation}\nonumber
    \norm{V_{\theta}}_D = \sqrt{\sum_{s\in\mathcal{S}}\nu(s)\left( \phi(s)^T\theta \right)^2  }=\norm{\theta}_{\Sigma},
\end{equation}
where $D=\text{diag}(\nu(s_1),\nu(s_2),\cdots)$ is a diagonal matrix whose elements are determined by the stationary distribution $\nu$, and $\Sigma\in \mathbb{R}^{d\times d}$ is the steady-state feature covariance matrix given by $\Sigma = \sum_{s\in\mathcal{S}}\nu(s)\phi(s)\phi(s)^T$.



The next assumption of the bounded domain is standard in theoretical analysis of the Adam-type algorithms \cite{reddi2019convergence,tran2019convergence}. 
\begin{assumption}
    \label{asp:boundedDomain}
    The domain $\mathcal{D}\subset \mathbb{R}^d$ of approximation parameters is a ball originating at $\theta = 0$ with a bounded diameter containing $\theta^\star$. That is, there exists $D_\infty$, such that $\theta^\star\in\mathcal{D}$, and $\norm{\theta_m - \theta_n} < D_\infty$, for any $\theta_m, \theta_n\in \mathcal{D}$.
\end{assumption}
The boundedness parameter $D_\infty$ can be chosen as discussed in \cite{bhandari2018finite}.

\subsection{Convergence of TD-AMSGrad}

In the following, we provide the convergence results for TD-AMSGrad with linear function approximation under Markovian sampling. 

First consider the full pseudo-gradient $\bar g(\theta)$ in \eqref{eq:meanGra}. 
We define $\theta^*$ as the fixed point of $\bar g(\theta)$, i.e., $\bar g(\theta^*) = 0$. Then $\theta^*$ is the unique fixed point under Assumption \ref{asp:boundPhi} following from the contraction property of the projected Bellman operator~\cite{tsitsiklis1997analysis}. The following lemma is useful for our convergence analysis.
\begin{lemma}\label{lem:TDStrongConv}\cite[Lemma 3]{bhandari2018finite}
Let $\omega>0$ be the minimum eigenvalue of the matrix $\Sigma$. Under Assumption \ref{asp:boundPhi} and for any $\theta\in\mathbb{R}^d$, we have
$$ (\theta-\theta^*)^T\bar g(\theta) \geq (1-\gamma)\sqrt{\omega}\norm{\theta - \theta^*}^2.
$$
\end{lemma}
Lemma \ref{lem:TDStrongConv} indicates that the update of TD learning exhibits a property similar to strongly convex optimization.

Since our analysis is under Markovian sampling, non-zero bias arises in our convergence analysis when estimating the gradient, i.e.,
$ \mathbb{E}[g_t] - \bar g(\theta_t) \neq 0$. The following lemma provides an upper bound on such a bias error for TD-AMSGrad.  
\begin{lemma}\label{lem:TDbias}
Consider a sequence of non-increasing stepsizes $\{\alpha_t\}_{t= 1}^T$ for Algorithm \ref{alg:AMSGradTD} and fix $\tau^*=\min\{\tau: m\rho^\tau \leq \alpha_T\} $. Initialize Algorithm \ref{alg:AMSGradTD} such that $|g_{1,i}|\geq G_0$ for all $i\in[d]$ with some $G_0>0$. Under Assumptions \ref{asp:markov}-\ref{asp:boundedDomain}, we have for $t\leq \tau^*$:
\begin{align*}
    &\mathbb{E}[(\theta_t-\theta^*)^T(g_t-\bar g(\theta_t))] \leq 2((1+\gamma)D_{\infty}+G_{\infty})\cdot \frac{G_{\infty} }{G_0}\tau^*\alpha_0;
\end{align*}
and for $t > \tau^*$:
\begin{align*}
    &\mathbb{E}[(\theta_t-\theta^*)^T(g_t-\bar g(\theta_t))] \leq 4D_{\infty}G_{\infty}\alpha_T + 2((1+\gamma)D_{\infty}+G_{\infty})\cdot \frac{G_{\infty} }{G_0}\sum_{i=t-\tau}^{t-1} a_i,
\end{align*}
where $G_\infty = R_{\max} + (1+\gamma) D_{\infty}$.
\end{lemma}

The following theorem provides the convergence of TD-AMSGrad under a constant stepsize coupled with diminishing hyper-parameters in the AMSGrad update.
\begin{theorem}\label{thm:AMSGradTDcons}
    (Convergence of TD-AMSGrad with constant stepsize) Let $ \beta_{1t}=\beta_1\lambda^t$ and $\delta=\beta_1/\beta_2$ with $\delta,\lambda\in(0,1)$ in Algorithm \ref{alg:AMSGradTD}. Initialize Algorithm \ref{alg:AMSGradTD} such that $|g_{1,i}|\geq G_0$ for all $i\in[d]$ with some $G_0>0$. Let $\alpha_t=\alpha,t=1,\dots,T$, and suppose Assumptions \ref{asp:markov}-\ref{asp:boundedDomain} hold. Then the output of Algorithm \ref{alg:AMSGradTD} satisfies:
$$ \mathbb{E}\norm{\theta_{out} - \theta^\star}^2 \leq \frac{C_1}{T} + \alpha C_2,
$$
where
\begin{align*}
     C_1 =& \frac{ G_\infty D_\infty^2}{\alpha c(1-\beta)} + \frac{\beta_1 \lambda G_\infty D_\infty^2}{2\alpha c(1-\lambda)(1-\beta)} +  2((1+\gamma)D_{\infty}+G_{\infty})\cdot \frac{G_{\infty} }{c G_0(1-\beta)}(\tau^*)^2\alpha,\\
     C_2 =& \frac{4D_{\infty}G_{\infty}}{c(1-\beta)} + \frac{2G_{\infty}\tau^* ((1+\gamma)D_{\infty}+G_{\infty})}{c G_0(1-\beta)}  + \frac{(1+\beta)G_{\infty}^2}{2c G_0(1-\beta)}
\end{align*}
with $c=(1-\gamma)\sqrt{\omega}$ and $G_\infty = R_{\max} + (1+\gamma) D_{\infty}$.
\end{theorem}

In Theorem \ref{thm:AMSGradTDcons}, $C_1, C_2 $ are constants and time-independent. Therefore, under the choice of the stepsize and hyper-parameters in the theorem, Algorithm \ref{alg:AMSGradTD} converges to a neighborhood of the global optimum at a rate of $\mathcal{O}\left( \frac{1}{T} \right)$. The size of the neigborhood is controlled by the stepsize $\alpha$. Similarly to Theorem \ref{thm:AMSGradPGcons}, we can balance the tradeoff between the convergence rate and the convergence accuracy, i.e., the size of neighborhood, by setting the stepsize $\alpha_t=\frac{1}{\sqrt{T}}$, which yields a convergence to the global optimal solution at the rate of $\mathcal{O}\left( \frac{\log^2 T}{\sqrt{T}} \right)$.

Next, we provide the convergence result with a diminishing stepsize in the following theorem.
\begin{theorem}\label{thm:AMSGradTDdimi}
(Convergence of TD-AMSGrad with diminishing stepsize) Suppose the same conditions of Theorem \ref{thm:AMSGradTDcons} hold, and let $\alpha_t=\frac{\alpha}{\sqrt{t}}$ for $t=1,\dots,T$. Then the output of Algorithm \ref{alg:AMSGradTD} satisfies:
$$ \mathbb{E}\norm{\theta_{out} - \theta^\star}^2 \leq \frac{C_1}{\sqrt{T}} + \frac{C_2}{T},
$$
where
\begin{align*}
    & C_1 = \frac{ G_\infty D_\infty^2 }{2c\alpha (1-\beta)} +  \frac{\alpha (1+\beta_1)G_{\infty}^2}{c G_0(1-\beta)} + \frac{4\alpha D_{\infty}G_{\infty}}{c (1-\beta)} + \frac{4\tau^*\alpha G_{\infty}((1+\gamma)D_{\infty}+G_{\infty}) }{c G_0(1-\beta)}, \\
    & C_2 = \frac{ G_\infty D_\infty^2 }{\sqrt{2}c\alpha (1-\beta)} + \frac{\beta G_\infty D_\infty^2}{2c\alpha(1-\lambda)^2(1-\beta)} + \frac{2G_{\infty}\alpha(\tau^*)^2 ((1+\gamma)D_{\infty}+G_{\infty}) }{c G_0(1-\beta)}
\end{align*}
with $c=(1-\gamma)\sqrt{\omega}$ and $G_\infty = R_{\max} + (1+\gamma) D_{\infty}$.
\end{theorem}

Comparing with Theorem \ref{thm:AMSGradTDcons} and observing $\tau^*=\mathcal{O}(\log T)$, we conclude that TD-ASMGrad with the diminishing stepsize converges exactly to the global optimum at the rate of $\mathcal{O}\left( \frac{\log T}{\sqrt{T}} \right)$, rather than to a neighborhood.

\section{Conclusion}

This paper provides the first convergence analysis of the Adam-type RL algorithms under Markovian sampling. For PG-AMSGrad with nonlinear function approximation for the policy, 
we show that the algorithm converges to a stationary point with a diminishing stepsize. With a constant size, we show that PG-AMSGrad converges only to a neighborhood of a stationary point but at a faster rate. Furthermore, we also provide the finite-time convergence results for TD-AMSGrad to the global optimum under proper choices of the stepsize.

Several future directions along this topic are interesting. For example, the optimality of the convergence result of PG-AMSGrad is of importance to study. More general value function approximation, and convergence results for constant hyper-parameters in TD-AMSGrad are also of interest.

\section*{Acknowledgements}

The work was supported in part by the U.S. National Science Foundation under Grants CCF-1761506, ECCS-1818904, CCF-1909291 and CCF-1900145.

\newpage

\bibliography{Z_AMSGradPG}
\bibliographystyle{apalike}

\newpage
\onecolumn

\appendix

\textbf{\Large Supplementary Materials}

\section{EstQ algorithm}\label{sec:appdixEstQ}

We introduce the unbiased Q-function estimating algorithm \textit{EstQ} below.
\begin{algorithm}
 	\caption{EstQ \cite{zhang2019global}} \label{alg:EstQ} 
 	\begin{algorithmic}[1]
 		\STATE 	{\bf Input:}   $s,a,\theta$. Initialize $\hat{Q}=0,s^q_1=s,a^q_1=a$
 		\STATE Draw $T\sim\text{Geom}(1-\gamma^{1/2})$.
		\FOR{ $t=1, 2, \ldots, T-1 $}
		\STATE Collect reward $R(s^q_t,a^q_t)$ and update the Q-function $\hat{Q}\leftarrow\hat{Q}+\gamma^{t/2} R(s^q_t,a^q_t)$.
		\STATE Sample $s^q_{t+1}\sim\mathbb{P}(\cdot|s^q_t,a^q_t),a^q_{t+1}\sim\pi_\theta(\cdot|s^q_{t+1})$.
		\ENDFOR
		\STATE Collect reward $R(s^q_T,a^q_T)$ and update the Q-function $\hat{Q}\leftarrow\hat{Q}+\gamma^{T/2} R(s^q_T,a^q_T)$.
        \STATE 	{ {\bf Output:} $\hat Q^{\pi_\theta}\leftarrow\hat{Q} $}.
 	\end{algorithmic}
\end{algorithm}

In Algorithm \ref{alg:EstQ}, we emphasize that the samples $\{(s^q_t,a^q_t)\}$ are used only to estimate the Q-function, and are independent from those used to evaluate the score function in PG-AMSGrad except the initial pair. \cite{zhang2019global} showed that this algorithm provided an unbiased estimator of the Q-function $Q^{\pi_{\theta}}(s,a)$.

\section{Proof of Lemma \ref{lem:pgLip} }

For notational simplicity, we denote $\nabla J(\theta) = \nabla_\theta J(\theta)$ and $\nabla \tilde{J}(\theta;s,a) = \nabla_\theta \tilde{J}(\theta;s,a)$ in the sequel. We first introduce two technical lemmas which are useful for the proof of Lemma \ref{lem:pgLip}.
\begin{lemma}\label{lem:pgGradBound}
The gradient estimator is uniformly bounded, that is,
\begin{equation}
    \norm{g_t}_\infty \leq \norm{g_t} \leq \frac{c_{\Theta}R_{\max}}{1-\gamma}.
\end{equation}
Similarly, we also have 
$$ \norm{\nabla J(\theta)} \leq \frac{c_{\Theta}R_{\max}}{1-\gamma};\qquad \norm{\nabla \tilde{J}(\theta;s,a)} \leq \frac{c_{\Theta}R_{\max}}{1-\gamma}.
$$
\end{lemma}
\begin{proof}
The proof can be proceeded by the definition of $g_t$.
\begin{align*}
    g_t &= \hat{Q}_t(s_T,a_t)\nabla_{\theta_t}\log(\pi_{\theta_t}(a_t|s_t))\\
    &= \left( \sum_{t'=t}^{\infty}\gamma^{t'-t}R(s_{t'},a_{t'}) \right)\cdot \log(\pi_{\theta_t}(a_t|s_t))\\
    &\leq \norm{\sum_{t'=t}^{\infty}\gamma^{t'-t}R(s_{t'},a_{t'})}\cdot ]\norm{\log(\pi_{\theta_t}(a_t|s_t))}\\
    &\overset{\text{(i)}}{\leq} \frac{c_{\Theta}R_{\max}}{1-\gamma},
\end{align*}
where (i) follows from Assumption \ref{asp:policyLip} and from the fact that the reward function is uniformly bounded. The proof of the second claim is similar to the first one.
\end{proof}

\begin{lemma}\label{lem:muTV}
\cite[Lemma 3]{zou2019finite} For any $\theta_1,\theta_2\in\mathbb{R}^d$, under Assumption \ref{asp:markov} we have
\begin{align*}
    \lTV{\mu_{\theta_1}-\mu_{\theta_2}} \leq |\mathcal{A}|L_\pi\left( 1 + \lceil \log_{\rho}\sigma^{-1} \rceil + \frac{1}{1-\rho} \right)\norm{\theta_1-\theta_2}.
\end{align*}
\end{lemma}

\textbf{Proof of Lemma \ref{lem:pgLip}:} To derive the Lipschitz continuity of $\nabla J(\theta)$, we first use the policy gradient theorem \cite{sutton2000policy} which gives
$$
\begin{aligned}  
\nabla J(\theta) &= \underset{(s,a)}{\int} \sum_{t=1}^\infty \gamma^t p(s_t=s|s_1,\pi_\theta) \cdot Q^{\pi_\theta}(s,a)\nabla\pi_\theta (a|s)dsda\\
&:=  \underset{(s,a)}{\int} Q^{\pi_\theta}(s,a) \nabla_{\theta}\log (\pi_{\theta}(a|s)) d\mu_\theta,
\end{aligned}
$$
where $\mu_\theta$ is denoted as the discounted visitation distribution. We further proceed the proof of Lipschitz continuity of $\nabla J(\theta)$ as
\begin{align}\label{eq:amsgradpg_pf1}
    &\norm{ \nabla J(\theta_1) - \nabla J(\theta_2) }\nonumber\\
    &\quad=  \underset{(s,a)}{\int}Q^{\pi_{\theta_1}}(s,a) \nabla_{\theta_1}\log (\pi_{\theta_1}(a|s))d\mu_{\theta_1} - \underset{(s,a)}{\int}Q^{\pi_{\theta_2}}(s,a) \nabla_{\theta_2}\log (\pi_{\theta_2}(a|s))d\mu_{\theta_2}\nonumber\\
    &\quad= \left[ \underset{(s,a)}{\int}Q^{\pi_{\theta_1}}(s,a) \nabla_{\theta_1}\log (\pi_{\theta_1}(a|s))d\mu_{\theta_1} -  \underset{(s,a)}{\int}Q^{\pi_{\theta_1}}(s,a) \nabla_{\theta_2}\log (\pi_{\theta_2}(a|s))d\mu_{\theta_1} \right]\nonumber\\
    &\quad\quad\ + \left[ \underset{(s,a)}{\int}Q^{\pi_{\theta_1}}(s,a) \nabla_{\theta_2}\log (\pi_{\theta_2}(a|s))d\mu_{\theta_1} -  \underset{(s,a)}{\int}Q^{\pi_{\theta_2}}(s,a) \nabla_{\theta_2}\log (\pi_{\theta_2}(a|s))d\mu_{\theta_1} \right]\nonumber\\
    &\quad\quad\ + \left[ \underset{(s,a)}{\int}Q^{\pi_{\theta_2}}(s,a) \nabla_{\theta_2}\log (\pi_{\theta_2}(a|s))d\mu_{\theta_1} -  \underset{(s,a)}{\int}Q^{\pi_{\theta_2}}(s,a) \nabla_{\theta_2}\log (\pi_{\theta_2}(a|s))d\mu_{\theta_2} \right]\nonumber\\
    &\quad=  \underset{(s,a)}{\int}Q^{\pi_{\theta_1}}(s,a)\left( \nabla_{\theta_1}\log (\pi_{\theta_1}(a|s)) - \nabla_{\theta_2}\log (\pi_{\theta_2}(a|s)) \right)d\mu_{\theta_1} \nonumber\\
    &\quad\quad\ +  \underset{(s,a)}{\int} (Q^{\pi_{\theta_1}}(s,a) - Q^{\pi_{\theta_2}}(s,a) ) \nabla_{\theta_2}\log (\pi_{\theta_2}(a|s))d\mu_{\theta_1} \nonumber\\
    &\quad\quad\ +\left[ \underset{(s,a)}{\int}Q^{\pi_{\theta_2}}(s,a) \nabla_{\theta_2}\log (\pi_{\theta_2}(a|s))d\mu_{\theta_1} -  \underset{(s,a)}{\int}Q^{\pi_{\theta_2}}(s,a) \nabla_{\theta_2}\log (\pi_{\theta_2}(a|s))d\mu_{\theta_2} \right]\nonumber\\
    &\quad \overset{\text{(i)}}{\leq} \frac{R_{\max} L}{1-\gamma}\norm{\theta_1 - \theta_2} \underset{(s,a)}{\int} d\mu_{\theta_1} +  c_{\Theta} \underset{(s,a)}{\int} (Q^{\pi_{\theta_1}}(s,a) - Q^{\pi_{\theta_2}}(s,a) ) d\mu_{\theta_1}\nonumber\\
    &\quad\quad\ + \left[ \underset{(s,a)}{\int}Q^{\pi_{\theta_2}}(s,a) \nabla_{\theta_2}\log (\pi_{\theta_2}(a|s))d\mu_{\theta_1} -  \underset{(s,a)}{\int}Q^{\pi_{\theta_2}}(s,a) \nabla_{\theta_2}\log (\pi_{\theta_2}(a|s))d\mu_{\theta_2} \right]\nonumber\\
    &\quad \overset{\text{(ii)}}{=} \frac{R_{\max} L}{1-\gamma}\norm{\theta_1 - \theta_2} +  c_{\Theta} \underset{(s,a)}{\int} (Q^{\pi_{\theta_1}}(s,a) - Q^{\pi_{\theta_2}}(s,a) ) d\mu_{\theta_1}\nonumber\\
    &\quad\quad\ + \left[ \underset{(s,a)}{\int}Q^{\pi_{\theta_2}}(s,a) \nabla_{\theta_2}\log (\pi_{\theta_2}(a|s))d\mu_{\theta_1} -  \underset{(s,a)}{\int}Q^{\pi_{\theta_2}}(s,a) \nabla_{\theta_2}\log (\pi_{\theta_2}(a|s))d\mu_{\theta_2} \right],
\end{align}
where (i) follows from Assumption \ref{asp:policyLip} and boundedness of $Q^{\pi_\theta}(s,a)$, and (ii) follows since $\underset{(s,a)}{\int} d\mu_{\theta_1}=1$. In the following, we bound the last two terms in the right hand side. First, we show that $Q^{\pi_{\theta}}(s,a)$ is also Lipschitz continuous with respect to $\theta$. To see that, we have
\begin{align*}
    |Q^{\pi_{\theta_1}}(s,a) - Q^{\pi_{\theta_2}}(s,a)| &= \underset{\nu_{\theta_1}}{\mathbb{E}} \left[ \sum_{t=0}^\infty\gamma^t R(s_t,a_t) \right] - \underset{\nu_{\theta_2}}{\mathbb{E}} \left[ \sum_{t=0}^\infty\gamma^t R(s_t,a_t) \right]\\
    &\leq \left\lvert \sum_{t=0}^\infty\gamma^t R(s_t,a_t) \right\rvert \norm{ \nu_{\theta_1} - \nu_{\theta_2} }_{TV}\\
    &\leq \frac{R_{\max}}{1-\gamma} \norm{ \nu_{\theta_1} - \nu_{\theta_2} }_{TV},
\end{align*}
where $\nu_\theta$ is the distribution of the state-action trajectory corresponding to the transition probability $\mathbb P$ and the policy $\pi$. Although $\nu_\theta$ is different from the state-action visitation distribution $\mu_\theta$, both of them share the same property. That is, for a different $\theta$, $\nu_\theta$ is determined by the same transition kernel and only differs in the policy. Therefore, Lemma \ref{lem:muTV} applies to two distributions $\nu_{\theta_1}$ and $\nu_{\theta_2}$, which yields
\begin{equation}\label{eq:amsgradpg_pf2}
    |Q^{\pi_{\theta_1}}(s,a) - Q^{\pi_{\theta_2}}(s,a)|\leq  \frac{R_{\max}}{1-\gamma}\cdot |\mathcal{A}|L_\pi\left( 1 + \lceil \log_{\rho}\sigma^{-1} \rceil + \frac{1}{1-\rho} \right)\norm{\theta_1-\theta_2}.
\end{equation}

Then we bound the last term of the right hand side of \eqref{eq:amsgradpg_pf1} via using the similar techniques above. 
\begin{align}\label{eq:amsgradpg_pf3}
    &\underset{(s,a)}{\int}Q^{\pi_{\theta_2}}(s,a) \nabla_{\theta_2}\log (\pi_{\theta_2}(a|s))d\mu_{\theta_1} -  \underset{(s,a)}{\int}Q^{\pi_{\theta_2}}(s,a) \nabla_{\theta_2}\log (\pi_{\theta_2}(a|s))d\mu_{\theta_2}\nonumber\\
    &\quad \leq \norm{Q^{\pi_{\theta_2}}(s,a) \nabla_{\theta_2}\log (\pi_{\theta_2}(a|s))}\norm{ \mu_{\theta_1} - \mu_{\theta_2} }_{TV}\nonumber\\
    &\quad \overset{\text{(i)}}{\leq} \frac{c_{\Theta} R_{\max}}{1-\gamma} \norm{ \mu_{\theta_1} - \mu_{\theta_2} }_{TV}\nonumber\\
    &\quad \overset{\text{(ii)}}{\leq} \frac{c_{\Theta} R_{\max}}{1-\gamma}\cdot |\mathcal{A}|L_\pi\left( 1 + \lceil \log_{\rho}\sigma^{-1} \rceil + \frac{1}{1-\rho} \right)\norm{\theta_1-\theta_2},
\end{align}
where (i) follows from Assumption \ref{asp:policyLip} and boundedness of $Q^{\pi_{\theta}}(s,a)$, and (ii) follows from Lemma \ref{lem:muTV}. By substituting the bounds from \eqref{eq:amsgradpg_pf2} and \eqref{eq:amsgradpg_pf3} to \eqref{eq:amsgradpg_pf1}, we obtain the Lipschitz condition of $\nabla J(\theta)$, i.e.,
\begin{align*}
    \norm{ \nabla J(\theta_1) - \nabla J(\theta_2) } \leq& \frac{R_{\max} L}{1-\gamma}\norm{\theta_1 - \theta_2} +  c_{\Theta} \underset{(s,a)}{\int} (Q^{\pi_{\theta_1}}(s,a) - Q^{\pi_{\theta_2}}(s,a) ) d\mu_{\theta_1}\\
    &+ \left[ \underset{(s,a)}{\int}Q^{\pi_{\theta_2}}(s,a) \nabla_{\theta_2}\log (\pi_{\theta_2}(a|s))d\mu_{\theta_1} -  \underset{(s,a)}{\int}Q^{\pi_{\theta_2}}(s,a) \nabla_{\theta_2}\log (\pi_{\theta_2}(a|s))d\mu_{\theta_2} \right]\\
    \leq& \left[ \frac{R_{\max} L}{1-\gamma} + \frac{(1+c_{\Theta}) R_{\max}}{1-\gamma}\cdot |\mathcal{A}|L_\pi\left( 1 + \lceil \log_{\rho}\sigma^{-1} \rceil + \frac{1}{1-\rho} \right) \right]\norm{\theta_1-\theta_2}.
\end{align*}

Next we prove \eqref{eq:tildeJLip}. Note that the Q-function is bounded. Then we have
\begin{align*}
    \norm{ \nabla \tilde{J}(\theta_1;s,a) - \nabla \tilde{J}(\theta_2;s,a) }=& Q^{\pi_{\theta_1}}(s,a) \nabla_{\theta_1}\log (\pi_{\theta_1}(a|s)) - Q^{\pi_{\theta_2}}(s,a) \nabla_{\theta_2}\log (\pi_{\theta_2}(a|s))\\
    =& Q^{\pi_{\theta_1}}(s,a) \left( \nabla_{\theta_1}\log (\pi_{\theta_1}(a|s)) - \nabla_{\theta_2}\log (\pi_{\theta_2}(a|s)) \right)\\
    & + \nabla_{\theta_2}\log (\pi_{\theta_2}(a|s))\left( Q^{\pi_{\theta_1}}(s,a)-Q^{\pi_{\theta_2}}(s,a) \right)\\
    \overset{\text{(i)}}{\leq}& \frac{R_{\max} L}{1-\gamma}\norm{\theta_1-\theta_2} +  c_{\Theta} \left( Q^{\pi_{\theta_1}}(s,a)-Q^{\pi_{\theta_2}}(s,a) \right)\\
    \overset{\text{(ii)}}{\leq}& \left[ \frac{R_{\max} L}{1-\gamma} + c_{\Theta}|\mathcal{A}|L_\pi\left( 1 + \lceil \log_{\rho}\sigma^{-1} \rceil + \frac{1}{1-\rho} \right)  \right]\norm{\theta_1-\theta_2},
\end{align*}
where (i) follows from Assumption \ref{asp:policyLip} and the boundedness of $Q^{\pi_{\theta}}(s,a)$, and (ii) follows from \eqref{eq:amsgradpg_pf2}.

\section{Proof of Theorem \ref{thm:AMSGradPGcons}  }

In this appendix we will first introduce some useful lemmas in Appendix \ref{subsec:lemmas for proof}, and then provide technical proofs in Appendix \ref{subsec:proof of Theorem 1}.

\subsection{Supporting Lemmas}\label{subsec:lemmas for proof}

We first provide two lemmas to deal with the bias of the gradient estimation.

\begin{lemma}\label{lem:vecMatPro}
Given $x,y\in\mathbb{R}^d$ and a positive diagonal matrix $V\in\mathbb{R}^{d\times d}$, then
\begin{equation}
    x^TVy \leq Tr(V)\norm{x}_\infty\norm{y}_\infty.
\end{equation}
\end{lemma}
\begin{proof}
Since $V$ is a diagonal matrix with $V_{i,i} > 0, \forall i\in [d]$, we have
\begin{align*}
    x^TVy &= \sum_{i=1}^d V_{i,i}x_iy_i \overset{\text{(i)}}{\leq} \sum_{i=1}^d V_{i,i}\norm{x}_\infty\norm{y}_\infty = Tr(V)\norm{x}_\infty\norm{y}_\infty,
\end{align*}
where (i) follows because $V_{i,i} > 0, \forall i\in [d]$.
\end{proof}

\begin{lemma}\label{lem:pgBias}
Fix time $t$ and any $\tau < t$. Initialize Algorithm \ref{alg:AMSGradPG} such that $|g_{1,i}|\geq G_0$ for all $i\in[d]$ with some $G_0>0$. Suppose Assumption \ref{asp:policyLip} and \ref{asp:markov} hold. Under Assumption \ref{asp:policyLip} and Assumption \ref{asp:markov}, we have
\begin{equation}
    \norm{\nabla J(\theta_{t-\tau})-\mathbb{E}\left[\nabla \tilde{J} (\theta_{t-\tau};s_t,a_t)|\mathcal{F}_{t-\tau} \right]}\leq G_\infty\left[ \sigma\rho^\tau + \frac{L_\pi G_{\infty}}{2G_0} \left(\sum_{k=t-\tau}^{t-1}\alpha_k + \sum_{i=t-\tau}^{t-1}\sum_{k=t-\tau}^{i}\alpha_k\right) \right],
\end{equation}
where $G_\infty = \frac{c_{\Theta} R_{\max}}{1-\gamma}$.
\end{lemma}
\begin{proof}

For notational brevity we denote $O_t = (s_t,a_T)$. Observe that
\begin{align*}
    &\norm{\nabla J(\theta_{t-\tau})-\mathbb{E}\left[\nabla \tilde{J} (\theta_{t-\tau};O_t)|\mathcal{F}_{t-\tau} \right]}\\
    &\quad = \norm{ \underset{(s,a)}{\int}Q^{\pi_{\theta_{t-\tau}}}\nabla_{\theta}\log (\pi_{\theta_{t-\tau}}(a|s))d\mu_{\theta_{t-\tau}}  - \underset{(s,a)}{\int}Q^{\pi_{\theta_{t-\tau}}}\nabla_{\theta}\log (\pi_{\theta_{t-\tau}}(a|s))dP(s_t,a_t|\mathcal{F}_{t-\tau})  }\\
    &\quad \overset{\text{(i)}}{\leq} \frac{c_\Theta R_{\max}}{1-\gamma} \lTV{\mu_{\theta_{t-\tau}}(\cdot,\cdot)-P(s_t,a_t|\mf_{t-\tau})}\\
    &\quad =G_\infty\lTV{\mu_{\theta_{t-\tau}}(\cdot,\cdot)-P(s_t,a_t|\mf_{t-\tau})},
\end{align*}
where (i) follows from the boundedness of $Q^{\pi_\theta}$. In the following, we bound $\lTV{\mu_{\theta_{t-\tau}}(\cdot,\cdot)-P(s_t,a_t|\mf_{t-\tau})}$. To this end, we need to build a new Markov chain in which the policy is fixed after time $t-\tau$. To be specific, in Algorithm \ref{alg:AMSGradPG}, we sample actions via a changing policy $\pi_{\theta_t}(a_t|s_t)$, which yields a Markov chain:
\begin{equation}
    O_{t-\tau} \rightarrow O_{t-\tau+1} \cdots \rightarrow  O_{t} \rightarrow  O_{t+1}.
\end{equation}
Now, from time $t-\tau$, we repeatedly use the policy $\pi_{\theta_{t-\tau}}$, which yields an auxiliary Markov chain:
\begin{equation}\label{eq:newMarChain}
    O_{t-\tau} \rightarrow \tilde{O}_{t-\tau+1} \cdots \rightarrow  \tilde{O}_{t} \rightarrow  \tilde{O}_{t+1},
\end{equation}
where we denote $\tilde{O}_t=(\tilde{s}_t,\tilde{a}_t)$ correspondingly. Clearly we have for $k=t-\tau+1,\dots,t$,
\begin{equation}
    P(\tilde{s}_k=\cdot|\theta_{t-\tau},\tilde{s}_{k-1}) = \sum_{a\in\mathcal{A}} \pi_{\theta_{t-\tau}}(\tilde{a}_{k-1}=a|\tilde{s}_{k-1})\mathbb P(\tilde{s}_k=\cdot|\tilde{a}_{k-1}=a,\tilde{s}_{k-1}).
\end{equation}

Since we use the same policy from time $t-\tau$, the Markov chain given in \eqref{eq:newMarChain} in this time slot is uniformly ergodic, and thus satisfies Assumption \ref{asp:markov}. Therefore, we can bound $\lTV{\mu_{\theta_{t-\tau}}(\cdot,\cdot)-P(\tilde{s}_t,\tilde{a}_t|\mf_{t-\tau})}$ as
\begin{equation}\label{incre13}
   \lTV{\mu_{\theta_{t-\tau}}(\cdot,\cdot)-P(\tilde{s}_t,\tilde{a}_t|\mf_{t-\tau})}\leq \sigma\rho^{\tau}.  
\end{equation}

Observe that
\begin{align}\label{eq:pfReduce1}
    &\lTV{\mu_{\theta_{t-\tau}}(\cdot,\cdot)-P(s_t,a_t|\mf_{t-\tau})}\nonumber\\
    &\quad\leq \lTV{\mu_{\theta_{t-\tau}}(\cdot,\cdot)-P(\tilde{s}_t,\tilde{a}_t|\mf_{t-\tau})} + \lTV{P(\tilde{s}_t,\tilde{a}_t|\mf_{t-\tau})-P(s_t,a_t|\mf_{t-\tau})}.
\end{align}
It remains to deal with $\lTV{P(\tilde{s}_t,\tilde{a}_t|\mf_{t-\tau})-P(s_t,a_t|\mf_{t-\tau})}$. We proceed as follows 
\begin{align}\label{incre14}
&\lTV{P(\tilde{s}_t,\tilde{a}_t|\mf_{t-\tau})-P(s_t,a_t|\mf_{t-\tau})}\nonumber\\
&\quad= \lTV{P(\tilde{s}_t|\mf_{t-\tau})\pi_{\theta_{t-\tau}}(\tilde{a}_t|\tilde{s}_t)-P(s_t|\mf_{t-\tau})\pi_{\theta_t}(a_t|s_t)}\nonumber\\
&\quad=\frac{1}{2}\underset{(s,a)}{\int} \lone{P(\tilde{s}_t=s|\mf_{t-\tau})\pi_{\theta_{t-\tau}}(\tilde{a}_t=a|\tilde{s}_t=s)-P(s_t=s|\mf_{t-\tau})\pi_{\theta_t}(a_t=a|s_t=s)}dsda\nonumber\\
&\quad=\frac{1}{2}\underset{(s,a)}{\int} \Big|P(\tilde{s}_t=s|\mf_{t-\tau})\pi_{\theta_{t-\tau}}(\tilde{a}_t=a|\tilde{s}_t=s)-P(\tilde{s}_t=s|\mf_{t-\tau})\pi_{\theta_t}(a_t=a|s_t=s) \nonumber\\
&\quad\quad\ + P(\tilde{s}_t=s|\mf_{t-\tau})\pi_{\theta_t}(a_t=a|s_t=s) - P(s_t=s|\mf_{t-\tau})\pi_{\theta_t}(a_t=a|s_t=s) \Big| dsda\nonumber\\ 
&\quad\leq\frac{1}{2}\underset{(s,a)}{\int}P(\tilde{s}_t=s|\mf_{t-\tau}) \lone{\pi_{\theta_{t-\tau}}(\tilde{a}_t=a|\tilde{s}_t=s)-\pi_{\theta_t}(a_t=a|s_t=s)}dads\nonumber\\
&\quad\quad\ + \frac{1}{2}\int_{s}\lone{P(\tilde{s}_t=s|\mf_{t-\tau})- P(s_t=s|\mf_{t-\tau})}\int_{a} \pi_{\theta_t}(a_t=a|s_t=s)
dads\nonumber\\
&\quad\overset{\text{(i)}}{\leq} \frac{1}{2}L_\pi\norm{\theta_t - \theta_{t-\tau}} + \frac{1}{2}\int_{s}\lone{P(\tilde{s}_t=s|\mf_{t-\tau})- P(s_t=s|\mf_{t-\tau})}ds \nonumber\\
&\quad\leq \frac{1}{2}L_\pi\sum_{k=t-\tau}^{t-1}\norm{\theta_{k+1}-\theta_k} + \frac{1}{2}\int_{s}\lone{P(\tilde{s}_t=s|\mf_{t-\tau})- P(s_t=s|\mf_{t-\tau})}ds \nonumber\\
&\quad\overset{\text{(ii)}}{\leq} \frac{L_\pi G_{\infty}}{2G_0} \sum_{k=t-\tau}^{t-1}\alpha_k + \frac{1}{2}\int_{s}\lone{P(\tilde{s}_t=s|\mf_{t-\tau})- P(s_t=s|\mf_{t-\tau})}ds\nonumber\\
&\quad= \frac{L_\pi G_{\infty}}{2G_0} \sum_{k=t-\tau}^{t-1}\alpha_k + \lTV{P(\tilde{s}_t|\mf_{t-\tau})- P(s_t|\mf_{t-\tau})},
\end{align}
where (i) follows from Assumption \ref{asp:policyLip} and (ii) follows from Lemma \ref{lem:conThetaBound}. It remains to bound the second term of the right hand side. To this end, we first write $P(\tilde{s}_t=\cdot|\mf_{t-\tau})$ as
\begin{align} \label{incre15}
P(\tilde{s}_t=\cdot|\mf_{t-\tau}) \nonumber
&=\int_{s} P(\tilde{s}_{t-1}=s|\mf_{t-\tau})P(\tilde{s}_t=\cdot|\tilde{s}_{t-1}=s)ds \nonumber\\
&=\int_{s} P(\tilde{s}_{t-1}=s|\mf_{t-\tau})\int_{a} P(\tilde{s}_t=\cdot|\tilde{s}_{t-1}=s,\tilde{a}_{t-1}=a)\pi_{\theta_{t-\tau}}(\tilde{a}_{t-1}=a|\tilde{s}_{t-1}=s) dads.
\end{align}

Similarly, for $P(s_t=\cdot|\mf_{t-\tau})$ we have
\begin{align}\label{incre16}
P(s_t=\cdot|\mf_{t-\tau})\nonumber &=\int_{s} P(s_{t-1}=s|\mf_{t-\tau}) P(s_t=\cdot|s_{t-1}=s)ds\nonumber\\
&=\int_{s} P(s_{t-1}=s|\mf_{t-\tau})\int_{a} P(s_t=\cdot|s_{t-1}=s,a_{t-1}=a)\pi_{\theta_{t-1}}(a_{t-1}=a|s_{t-1}=s) dads.
\end{align}

Then we obtain the following bound
\begin{align*}
&\lone{P(\tilde{s}_t=\cdot|\mf_{t-\tau})- P(s_t=\cdot|\mf_{t-\tau})} \\
&\quad=\Bigg| \int_{s} P(\tilde{s}_{t-1}=s|\mf_{t-\tau})\int_{a} P(\tilde{s}_t=\cdot|\tilde{s}_{t-1}=s,\tilde{a}_{t-1}=a)\pi_{\theta_{t-\tau}}(\tilde{a}_{t-1}=a|\tilde{s}_{t-1}=s) dads \\
&\quad\qquad - \int_{s} P(s_{t-1}=s|\mf_{t-\tau})\int_{a} P(s_t=\cdot|s_{t-1}=s,a_{t-1}=a)\pi_{\theta_{t-1}}(a_{t-1}=a|s_{t-1}=s) dads \Bigg|  \\
&\quad\leq  \int_{s} P(\tilde{s}_{t-1}=s|\mf_{t-\tau})  \Bigg|\int_{a} P(\tilde{s}_t=\cdot|\tilde{s}_{t-1}=s,\tilde{a}_{t-1}=a)\pi_{\theta_{t-\tau}}(\tilde{a}_{t-1}=a|\tilde{s}_{t-1}=s)  \\
&\quad\quad\ - \int_{a} P(s_t=\cdot|s_{t-1}=s,a_{t-1}=a)\pi_{\theta_{t-1}}(a_{t-1}=a|s_{t-1}=s) \Bigg| da ds \\
&\quad\quad\ + \int_{s} \Bigg|P(\tilde{s}_{t-1}=s|\mf_{t-\tau})  - P(s_{t-1}=s|\mf_{t-\tau}) \Bigg| \cdot\\
&\quad\quad\qquad\int_{a} P(s_t=\cdot|s_{t-1}=s,a_{t-1}=a)\pi_{\theta_{t-1}}(a_{t-1}=a|s_{t-1}=s) dads \\
&\quad\leq  \int_{s} P(\tilde{s}_{t-1}=s|\mf_{t-\tau}) \cdot\\
&\quad\quad\qquad \int_{a} P(s_t=\cdot|s_{t-1}=s,a_{t-1}=a) \lone{\pi_{\theta_{t-\tau}}(\tilde{a}_{t-1}=a|\tilde{s}_{t-1}=s) - \pi_{\theta_{t-1}}(a_{t-1}=a|s_{t-1}=s)}  dads  \\
&\quad\quad + \int_{s} \lone{P(\tilde{s}_{t-1}=s|\mf_{t-\tau})  - P(s_{t-1}=s|\mf_{t-\tau})} ds \\
&\quad\leq  \int_{s} P(\tilde{s}_{t-1}=s|\mf_{t-\tau})  \int_{a} \lone{\pi_{\theta_{t-\tau}}(\tilde{a}_{t-1}=a|\tilde{s}_{t-1}=s) - \pi_{\theta_{t-1}}(a_{t-1}=a|s_{t-1}=s)}  dads  \\
&\quad\quad + \int_{s} \lone{P(\tilde{s}_{t-1}=s|\mf_{t-\tau})  - P(s_{t-1}=s|\mf_{t-\tau})} ds \\
&\quad\leq L_\pi\norm{\theta_{t-\tau}-\theta_{t-1}} + \int_{s} \lone{P(\tilde{s}_{t-1}=s|\mf_{t-\tau})  - P(s_{t-1}=s|\mf_{t-\tau})} ds \\
&\quad\leq L_\pi\norm{\theta_{t-\tau}-\theta_{t-1}} +  2\lTV{P(\tilde{s}_{t-1}|\mf_{t-\tau})  - P(s_{t-1}|\mf_{t-\tau})}. 
\end{align*}

Then we have a dynamical form as
\begin{align}
\lTV{P(\tilde{s}_t|\mf_{t-\tau})- P(s_t|\mf_{t-\tau})}
&=\frac{1}{2}\int_{s}\lone{P(\tilde{s}_t=s|\mf_{t-\tau})- P(s_t=s|\mf_{t-\tau})}ds \nonumber\\
&\leq \frac{1}{2}L_\pi\norm{\theta_{t-\tau}-\theta_{t-1}} +  \lTV{P(\tilde{s}_{t-1}|\mf_{t-\tau})  - P(s_{t-1}|\mf_{t-\tau})}. \label{incre17}
\end{align}

Applying \eqref{incre17} recursively yields

\begin{align} \label{incre18}
\lTV{P(\tilde{s}_t|\mf_{t-\tau})- P(s_t|\mf_{t-\tau})} 
&\leq \frac{1}{2}L_\pi\sum_{i=t-\tau}^{t-1}\norm{\theta_{t-\tau}-\theta_i}\nonumber \\
&\leq \frac{L_\pi G_\infty}{2G_0} \sum_{i=t-\tau}^{t-1}\sum_{k=t-\tau}^{i}\alpha_k.
\end{align}

Substituting \eqref{incre18} into \eqref{incre14} and according to the fact $\tau\geq 1$, we have
\begin{equation}\label{incre19} 
\lTV{P(\tilde{s}_t,\tilde{a}_t|\mf_{t-\tau})-P(s_t,a_t|\mf_{t-\tau})}\leq \frac{L_\pi G_{\infty}}{2G_0} \left(\sum_{k=t-\tau}^{t-1}\alpha_k + \sum_{i=t-\tau}^{t-1}\sum_{k=t-\tau}^{i}\alpha_k\right).
\end{equation}

Combining \eqref{incre13} and \eqref{incre19}, we obtain
\begin{align*}
\norm{\nabla J(\theta_{t-\tau})-\mathbb{E}\left[\nabla \tilde{J} (\theta_{t-\tau};O_t)|\mathcal{F}_{t-\tau} \right]}
&\leq G_\infty\lTV{\mu_{\theta_{t-\tau}}(\cdot,\cdot)-P_t(\tilde{s}_t,\tilde{a}_t|\mf_{t-\tau})}\\
&\leq G_\infty\left[ \sigma\rho^\tau + \frac{L_\pi G_{\infty}}{2G_0} \left(\sum_{k=t-\tau}^{t-1}\alpha_k + \sum_{i=t-\tau}^{t-1}\sum_{k=t-\tau}^{i}\alpha_k\right) \right]. 
\end{align*}

\end{proof}

The next a few lemmas are closely related to the update rule of AMSGrad.

\begin{lemma}\cite[Lemma A.1]{zhou2018convergence}
\label{lem:mvbound}
Let $\{g_t, m_t, \hat{v}_t\}$ for $t=1,2,\dots$ be sequences generated by Algorithm \ref{alg:AMSGradPG}, and denote $ G_\infty = \frac{c_{\Theta}R_{\max}}{1-\gamma} $. Given $g_t\leq G_\infty$ as in Lemma \ref{lem:pgGradBound}, $\norm{m_t} \leq G_\infty, \norm{\hat{v}_t} \leq G_\infty^2$.
\end{lemma}

\begin{lemma}\label{lem:conThetaBound}
Let $\{m_t, \hat{v}_t\}$ for $t=1,2,\dots$ be sequences generated by Algorithm \ref{alg:AMSGradPG} and dentoe $G_\infty = R_{\max} + (1+\gamma) D_{\infty}$. Then 
\begin{equation}
    \norm{\theta_{t+1} -\theta_t } = \norm{\alpha_t \hat V_t^{\frac{1}{2}} m_t } \leq \frac{\alpha_t G_\infty}{G_0}.
\end{equation} 
\end{lemma}
\begin{proof}
The proof can proceed easily given Lemma \ref{lem:mvbound2} and $\hat{v}_{t,i} \geq \hat{v}_{0,i} \geq G_0^2, \forall t, \forall i$.
\begin{equation}
    \begin{aligned}
    \norm{\theta_{t+1} -\theta_t } = \norm{\alpha_t \hat V_t^{\frac{1}{2}} m_t }
    = \alpha_t \sqrt{ \sum_{i=1}^d\frac{m_{t,i}^2}{v_{t,i}} }
    \leq \frac{\alpha_t}{G_0}\norm{m_t}
    \leq \frac{\alpha_t G_\infty}{G_0}.
    \end{aligned}
\end{equation}
\end{proof}

\begin{lemma}\label{lem:pgAMSGradVm}
\cite[Lemma A.2]{zhou2018convergence}
Let $\alpha_t$ be the stepsize in Algorithm \ref{alg:AMSGradPG} and $\beta_1,\beta_2$ be the constant hyper-parameters with $\beta_1 < \beta_2$. Then we have 
\begin{equation}
    \sum_{t=1}^T\mathbb{E}\norm{\hat{V}_{t}^{-\frac{1}{2}}m_t}^2 \leq \frac{d(1-\beta_1)}{(1-\beta_2)(1-\beta_1/\beta_2)}.
\end{equation}
In addition, taking $\beta_1 = 0$ yields $m_t = g_t$. Thus we further have
\begin{equation}
    \sum_{t=1}^T\mathbb{E}\norm{\hat{V}_{t}^{-\frac{1}{2}}g_t}^2 \leq \frac{d}{1-\beta_2}.
\end{equation}
\end{lemma}
\begin{proof}
We refer the reader to the proof of Lemma A.2 in \cite{zhou2018convergence} for more details by reducing their proof to the case where $p=\frac{1}{2},q=1$.
\end{proof}

Since the update rule of AMSGrad is complicated, it is often useful to introduce an auxiliary sequences $z_t$. If we define $\theta_0 = \theta_1$, then for $t\geq 1$, let
\begin{equation}\label{eq:zt}
    z_t = \theta_t + \frac{\beta_1}{1-\beta_1}\theta_{t-1}.
\end{equation}
The following lemma captures the property of $z_t$ and its connection with $\theta_t$.
\begin{lemma}\label{lem:auxseqprop}
\cite[Lemma A.3-A.5]{zhou2018convergence}
Given $z_t$ defined in \eqref{eq:zt}, we have for $t\geq 2$ 
\begin{equation}
    z_{t+1} - z_t = \frac{\beta_1}{1-\beta_1}\left( \alpha_{t-1}\hat{V}_{t-1}^{-\frac{1}{2}} - \alpha_t\hat{V}_{t}^{-\frac{1}{2}} \right)m_{t-1} - \alpha_t\hat{V}_{t}^{-\frac{1}{2}}g_t,
\end{equation}
and for $t=1$,
\begin{equation}
    z_2-z_1 = - \alpha_1\hat{V}_{1}^{-\frac{1}{2}}g_1.
\end{equation}
Further, we can bound $\norm{z_{t+1} - z_t}$ as
\begin{equation}
    \norm{z_{t+1} - z_t} \leq \norm{\alpha_t\hat{V}_{t}^{-\frac{1}{2}}g_t} + \frac{\beta_1}{1-\beta_1}\norm{\theta_{t+1} - \theta_t},
\end{equation}
and we can also bound $\norm{\nabla J(z_t) - \nabla J(\theta_t)}$ as
\begin{equation}
    \norm{\nabla J(z_t) - \nabla J(\theta_t)} \leq \frac{c_J\beta_1}{1-\beta_1}\norm{\theta_{t+1} - \theta_t}.
\end{equation}
\end{lemma}

\subsection{Proof of Theorem \ref{thm:AMSGradPGcons}}\label{subsec:proof of Theorem 1}

In the remaining, we can complete this proof by taking the following three steps.

\textbf{Step1: Establishing convergent sequence}

First, we observe that $\nabla J(\theta)$ is Lipschitz continuous according to Lemma \ref{lem:pgLip}. Then we have
\begin{align}\label{eq:pgthempf1}
    J(z_{t+1}) \leq& J(z_t) + \nabla J(z_t)^T(z_{t+1} - z_t) + \frac{c_J}{2}\norm{z_{t+1} - z_t}^2\nonumber\\
    =& J(z_t) + \nabla J(\theta_t)^T(z_{t+1} - z_t) + (\nabla J(z_t) - \nabla J(\theta_t))^T(z_{t+1} - z_t) + \frac{c_J}{2}\norm{z_{t+1} - z_t}^2\nonumber\\
    =& J(z_t) + \nabla J(\theta_t)^T\left[ \frac{\beta_1}{1-\beta_1}\left( \alpha_{t-1}\hat{V}_{t-1}^{-\frac{1}{2}} - \alpha_t\hat{V}_{t}^{-\frac{1}{2}} \right)m_{t-1} - \alpha_t\hat{V}_{t}^{-\frac{1}{2}}g_t \right]\nonumber\\
    &+ (\nabla J(z_t) - \nabla J(\theta_t))^T(z_{t+1} - z_t) + \frac{c_J}{2}\norm{z_{t+1} - z_t}^2\nonumber\\
    =& J(z_t) + \underbrace{\frac{\beta_1}{1-\beta_1}\nabla J(\theta_t)^T \left( \alpha_{t-1}\hat{V}_{t-1}^{-\frac{1}{2}} - \alpha_t\hat{V}_{t}^{-\frac{1}{2}} \right)m_{t-1}}_{T_1} \underbrace{ - \alpha_t\nabla J(\theta_t)^T\hat{V}_{t}^{-\frac{1}{2}}g_t}_{T_2} \nonumber\\
    &+ \underbrace{(\nabla J(z_t) - \nabla J(\theta_t))^T(z_{t+1} - z_t)}_{T_3} + \underbrace{\frac{c_J}{2}\norm{z_{t+1} - z_t}^2}_{T_4}.
\end{align}

Next we bound the tail terms. The term $T_1$ can be bounded as
\begin{align*}
    T_1 &\overset{\text{(i)}}{\leq} \frac{\beta_1}{1-\beta_1} Tr\left( \alpha_{t-1}\hat{V}_{t-1}^{-\frac{1}{2}} - \alpha_t\hat{V}_{t}^{-\frac{1}{2}} \right) \norm{\nabla J(\theta_t)}_\infty \norm{m_{t-1}}_\infty\\
    &= \frac{\beta_1}{1-\beta_1} G_{\infty}^2 Tr\left( \alpha_{t-1}\hat{V}_{t-1}^{-\frac{1}{2}} - \alpha_t\hat{V}_{t}^{-\frac{1}{2}} \right)\\
    &= \frac{\beta_1}{1-\beta_1} G_{\infty}^2 \left[Tr\left( \alpha_{t-1}\hat{V}_{t-1}^{-\frac{1}{2}}\right) - Tr\left(\alpha_t\hat{V}_{t}^{-\frac{1}{2}} \right)\right],
\end{align*}
where (i) follows from Lemma \ref{lem:vecMatPro}.

The term $T_2$ is the key to deal with under non-i.i.d.\ sampling, where a non-zero bias arises in the gradient estimation. We bound this term as
\begin{align*}
    T_2 =& - \alpha_{t-1}\nabla J(\theta_t)^T\hat{V}_{t-1}^{-\frac{1}{2}}g_t + \nabla J(\theta_t)^T\left( \alpha_{t-1}\hat{V}_{t-1}^{-\frac{1}{2}} - \alpha_t\hat{V}_{t}^{-\frac{1}{2}} \right)g_t\\
    \overset{\text{(i)}}{\leq}& - \alpha_{t-1}\nabla J(\theta_t)^T\hat{V}_{t-1}^{-\frac{1}{2}}g_t + Tr\left( \alpha_{t-1}\hat{V}_{t-1}^{-\frac{1}{2}} - \alpha_t\hat{V}_{t}^{-\frac{1}{2}} \right) \norm{\nabla J(\theta_t)}_\infty \norm{g_t}_\infty\\
    =&  - \alpha_{t-1}\nabla J(\theta_t)^T\hat{V}_{t-1}^{-\frac{1}{2}}g_t + G_{\infty}^2 \left[Tr\left( \alpha_{t-1}\hat{V}_{t-1}^{-\frac{1}{2}}\right) - Tr\left(\alpha_t\hat{V}_{t}^{-\frac{1}{2}} \right)\right]\\
    =& - \alpha_{t-1}\nabla J(\theta_t)^T\hat{V}_{t-1}^{-\frac{1}{2}}\nabla J(\theta_t) + \alpha_{t-1}\nabla J(\theta_t)^T\hat{V}_{t-1}^{-\frac{1}{2}}(\nabla J(\theta_t)-g_t)\\
    &+ G_{\infty}^2 \left[Tr\left( \alpha_{t-1}\hat{V}_{t-1}^{-\frac{1}{2}}\right) - Tr\left(\alpha_t\hat{V}_{t}^{-\frac{1}{2}} \right)\right]\\
    \overset{\text{(ii)}}{\leq}& - \frac{\alpha_{t-1}}{G_\infty}\norm{\nabla J(\theta_t)}^2 + \alpha_{t-1}\nabla J(\theta_t)^T\hat{V}_{t-1}^{-\frac{1}{2}}(\nabla J(\theta_t)-g_t)\\
    &+ G_{\infty}^2 \left[Tr\left( \alpha_{t-1}\hat{V}_{t-1}^{-\frac{1}{2}}\right) - Tr\left(\alpha_t\hat{V}_{t}^{-\frac{1}{2}} \right)\right],
\end{align*}
where (i) follows from Lemma \ref{lem:vecMatPro} and (ii) follows because $\hat{V}_t$ is positive diagonal matrix and each entry is bounded as in Lemma \ref{lem:mvbound}. 

Next we bound the term $T_3$ as
\begin{align*}
    T_3 &\overset{\text{(i)}}{\leq} \norm{\nabla J(z_t) - \nabla J(\theta_t)}\norm{z_{t+1} - z_t}\\
    &\overset{\text{(ii)}}{\leq} \left( \norm{\alpha_t\hat{V}_{t}^{-\frac{1}{2}}g_t} + \frac{\beta_1}{1-\beta_1}\norm{\theta_{t+1} - \theta_t} \right)\frac{c_J\beta_1}{1-\beta_1}\norm{\theta_{t+1} - \theta_t}\\
    &= \frac{c_J\beta_1}{1-\beta_1}\norm{\theta_{t+1} - \theta_t}\norm{\alpha_t\hat{V}_{t}^{-\frac{1}{2}}g_t} + c_J\left(\frac{\beta_1}{1-\beta_1}\right)^2\norm{\theta_{t+1} - \theta_t}^2\\
    &= \sqrt{c_J}\norm{\alpha_t\hat{V}_{t}^{-\frac{1}{2}}g_t}\cdot\frac{\sqrt{c_J}\beta_1}{1-\beta_1}\norm{\theta_{t+1} - \theta_t} + c_J\left(\frac{\beta_1}{1-\beta_1}\right)^2\norm{\theta_{t+1} - \theta_t}^2\\
    &\overset{\text{(iii)}}{\leq} c_J\norm{\alpha_t\hat{V}_{t}^{-\frac{1}{2}}g_t}^2 + 2c_J\left(\frac{\beta_1}{1-\beta_1}\right)^2\norm{\theta_{t+1} - \theta_t}^2\\
    &\overset{\text{(iv)}}{=} c_J\norm{\alpha_t\hat{V}_{t}^{-\frac{1}{2}}g_t}^2 + 2c_J\left(\frac{\beta_1}{1-\beta_1}\right)^2\norm{\alpha_t\hat{V}_{t}^{-\frac{1}{2}}m_t}^2,
\end{align*}
where (i) follows from Cauchy-Schwarz inequality, (ii) follows from Lemma \ref{lem:auxseqprop}, (iii) follows because $xy\leq x^2 + y^2$ and (iv) follows by the update rule of Algorithm \ref{alg:AMSGradPG}.

Last, we bound the term $T_4$ as
\begin{align*}
    T_4 &\overset{\text{(i)}}{\leq} \frac{c_J}{2} \left[ \norm{\alpha_t\hat{V}_{t}^{-\frac{1}{2}}g_t} + \frac{\beta_1}{1-\beta_1}\norm{\theta_{t+1} - \theta_t} \right]^2\\
    &\overset{\text{(ii)}}{\leq} c_J\norm{\alpha_t\hat{V}_{t}^{-\frac{1}{2}}g_t}^2 + c_J\left(\frac{\beta_1}{1-\beta_1}\right)^2\norm{\theta_{t+1} - \theta_t}^2\\
    &= c_J\norm{\alpha_t\hat{V}_{t}^{-\frac{1}{2}}g_t}^2 + c_J\left(\frac{\beta_1}{1-\beta_1}\right)^2\norm{\alpha_t\hat{V}_{t}^{-\frac{1}{2}}m_t}^2,
\end{align*}
where (i) follows from Lemma \ref{lem:auxseqprop}, and (ii) follows due to the fact $(x+y)^2 \leq 2x^2 + 2y^2$.

Substituting the upper bounds of the terms $T_1,T_2,T_3$ and $T_4$ in \eqref{eq:pgthempf1} yields
\begin{align*}
    J(z_{t+1}) \leq& J(z_t) - \frac{\alpha_{t-1}}{G_\infty}\norm{\nabla J(\theta_t)}^2 + \alpha_{t-1}\nabla J(\theta_t)^T\hat{V}_{t-1}^{-\frac{1}{2}}(\nabla J(\theta_t)-g_t)\\
    &+ \frac{1}{1-\beta_1} G_{\infty}^2 \left[Tr\left( \alpha_{t-1}\hat{V}_{t-1}^{-\frac{1}{2}}\right) - Tr\left(\alpha_t\hat{V}_{t}^{-\frac{1}{2}} \right)\right]\\
    &+ 2c_J\norm{\alpha_t\hat{V}_{t}^{-\frac{1}{2}}g_t}^2 + 3c_J\left(\frac{\beta_1}{1-\beta_1}\right)^2\norm{\alpha_t\hat{V}_{t}^{-\frac{1}{2}}m_t}^2.
\end{align*}

Next we rearrange the above inequality and take expectation over all the randomness to obtain
\begin{align}\label{eq:pgthempf2}
     \frac{\alpha_{t-1}}{G_\infty}\mathbb{E}\left[\norm{\nabla J(\theta_t)}^2\right]
     \leq& \left(\mathbb{E}[J(z_t)] - \mathbb{E}[J(z_{t+1})] \right) + \frac{G_{\infty}^2}{1-\beta_1}  \left(\mathbb{E}\left[Tr\left( \alpha_{t-1}\hat{V}_{t-1}^{-\frac{1}{2}}\right)\right] -\mathbb{E}\left[ Tr\left(\alpha_t\hat{V}_{t}^{-\frac{1}{2}} \right)\right] \right)\nonumber\\
    &+ 2c_J\mathbb{E}\norm{\alpha_t\hat{V}_{t}^{-\frac{1}{2}}g_t}^2 + 3c_J\left(\frac{\beta_1}{1-\beta_1}\right)^2\mathbb{E}\norm{\alpha_t\hat{V}_{t}^{-\frac{1}{2}}m_t}^2\nonumber\\
    &+ \alpha_{t-1}\mathbb{E}\left[\nabla J(\theta_t)^T\hat{V}_{t-1}^{-\frac{1}{2}}(\nabla J(\theta_t)-g_t) \right]\nonumber\\
    =& \left(\mathbb{E}[J(z_t)] - \mathbb{E}[J(z_{t+1})] \right) + \frac{G_{\infty}^2}{1-\beta_1}  \left(\mathbb{E}\left[Tr\left( \alpha_{t-1}\hat{V}_{t-1}^{-\frac{1}{2}}\right)\right] -\mathbb{E}\left[ Tr\left(\alpha_t\hat{V}_{t}^{-\frac{1}{2}} \right)\right] \right)\nonumber\\
    &+ 2c_J\mathbb{E}\norm{\alpha_t\hat{V}_{t}^{-\frac{1}{2}}g_t}^2 + 3c_J\left(\frac{\beta_1}{1-\beta_1}\right)^2\mathbb{E}\norm{\alpha_t\hat{V}_{t}^{-\frac{1}{2}}m_t}^2\nonumber\\
    &+ \alpha_{t-1}\mathbb{E}\left[ \mathbb{E}\left[\nabla J(\theta_t)^T\hat{V}_{t-1}^{-\frac{1}{2}}(\nabla J(\theta_t)-g_t)\rvert \mathcal{F}_{t}^q \right] \right]\nonumber\\
    =& \left(\mathbb{E}[J(z_t)] - \mathbb{E}[J(z_{t+1})] \right) + \frac{G_{\infty}^2}{1-\beta_1}  \left(\mathbb{E}\left[Tr\left( \alpha_{t-1}\hat{V}_{t-1}^{-\frac{1}{2}}\right)\right] -\mathbb{E}\left[ Tr\left(\alpha_t\hat{V}_{t}^{-\frac{1}{2}} \right)\right] \right)\nonumber\\
    &+ 2c_J\mathbb{E}\norm{\alpha_t\hat{V}_{t}^{-\frac{1}{2}}g_t}^2 + 3c_J\left(\frac{\beta_1}{1-\beta_1}\right)^2\mathbb{E}\norm{\alpha_t\hat{V}_{t}^{-\frac{1}{2}}m_t}^2\nonumber\\
    &+ \alpha_{t-1}\mathbb{E}\left[\nabla J(\theta_t)^T\hat{V}_{t-1}^{-\frac{1}{2}}(\nabla J(\theta_t)-\nabla \tilde{J}(\theta_t;O_t)) \right],
\end{align}
where in the last equation we denote $O_t=(s_t,a_t)$ for brevity, and this notation is used in the sequel as well.
We emphasize that the filtration $\mathcal{F}_{t}^q$ contains all the samples up to time $t$ except the samples for estimating $Q^{\pi_{\theta_t}}(s_t,a_t)$. Thus we have $\mathbb{E} [g_t\rvert \mathcal{F}_{t}^q] = \nabla \tilde{J}(\theta_t;O_t)$ where the expectation is taken over the randomness in EstQ algorithm.

\textbf{Step2: Bounding bias term}

In the following, we bound the bias term $\mathbb{E}\left[\nabla J(\theta_t)^T\hat{V}_{t-1}^{-\frac{1}{2}}(\nabla J(\theta_t)-\nabla \tilde{J}(\theta_t;O_t)) \right]$. Observe that
\begin{align*}
    \mathbb{E}\left[\right.&\left.\nabla J(\theta_t)^T\hat{V}_{t-1}^{-\frac{1}{2}}(\nabla J(\theta_t)-\nabla \tilde{J}(\theta_t;O_t)) \right] \\
    =& \mathbb{E}\left[ \sum_{i=1}^d \frac{1}{\sqrt{\hat{v}_{t-1,i}}} \nabla_i J(\theta_t)\cdot(\nabla_i J(\theta_t)-\nabla \tilde{J}_i(\theta_t;O_t)) \right]\\
    =& \mathbb{E}\left[ \sum_{i=1}^d \frac{1}{\sqrt{\hat{v}_{t-1,i}}} \left[ \nabla_i J(\theta_{t-\tau})\cdot(\nabla_i J(\theta_{t-\tau})-\nabla \tilde{J}_i(\theta_{t-\tau};O_t)) + \nabla_i J(\theta_{t-\tau})\cdot(\nabla \tilde{J}_i(\theta_{t-\tau};O_t)-\nabla \tilde{J}_i(\theta_{t};O_t))\right.\right.\\
    &+ \left.\left.  \nabla_i J(\theta_{t-\tau})\cdot(\nabla_i J(\theta_{t})-\nabla_i J(\theta_{t-\tau})) + (\nabla_i J(\theta_{t})-\nabla_i J(\theta_{t-\tau}))\cdot(\nabla_i J(\theta_t)-\nabla \tilde{J}_i(\theta_t;O_t)) \right] \right]\\
    =& \mathbb{E} \left[ \underbrace{\nabla J(\theta_{t-\tau})^T \hat{V}_{t-1}^{-\frac{1}{2}}(\nabla J(\theta_{t-\tau})-\nabla \tilde{J}(\theta_{t-\tau};O_t))}_{T5} + \underbrace{\nabla J(\theta_{t-\tau})^T \hat{V}_{t-1}^{-\frac{1}{2}}(\nabla \tilde{J}(\theta_{t-\tau};O_t)-\nabla \tilde{J}(\theta_{t};O_t))}_{T_6}  \right. \\
    & \left.+ \underbrace{\nabla J(\theta_{t-\tau})^T \hat{V}_{t-1}^{-\frac{1}{2}}(\nabla J(\theta_{t})-\nabla J(\theta_{t-\tau}))}_{T_7}  + \underbrace{(\nabla J(\theta_{t})-\nabla J(\theta_{t-\tau}))^T \hat{V}_{t-1}^{-\frac{1}{2}}(\nabla J(\theta_t)-\nabla \tilde{J}(\theta_t;O_t))}_{T_8} \right].
\end{align*}

It turns out that the terms $T_6,T_7,T_8$ are easier to bound and the term $T_5$ is the key to bound the bias. For the clarity of presentation, we first bound the terms $T_6,T_7,T_8$.

To bound the term $T_6$, we have
\begin{align*}
    T_6 =& \nabla J(\theta_{t-\tau})^T \hat{V}_{t-1}^{-\frac{1}{2}}(\nabla \tilde{J}(\theta_{t-\tau};O_t)-\nabla \tilde{J}(\theta_{t};O_t))\\
    =& \left(\nabla J(\theta_{t-\tau})^T \hat{V}_{t-1}^{-\frac{1}{4}} \right) \cdot \left(\hat{V}_{t-1}^{-\frac{1}{4}}(\nabla \tilde{J}(\theta_{t-\tau};O_t)-\nabla \tilde{J}(\theta_{t};O_t)) \right)\\
    \overset{\text{(i)}}{\leq}& \norm{\nabla J(\theta_{t-\tau})^T \hat{V}_{t-1}^{-\frac{1}{4}}  } \cdot \norm{\hat{V}_{t-1}^{-\frac{1}{4}}(\nabla \tilde{J}(\theta_{t-\tau};O_t)-\nabla \tilde{J}(\theta_{t};O_t))  }\\
    \overset{\text{(ii)}}{\leq}& \frac{1}{G_0} \norm{\nabla J(\theta_{t-\tau}) } \cdot \norm{(\nabla \tilde{J}(\theta_{t-\tau};O_t)-\nabla \tilde{J}(\theta_{t};O_t))  }\\
    \overset{\text{(iii)}}{\leq}& \frac{G_\infty}{G_0}\cdot c_{\tilde{J}}\norm{\theta_t-\theta_{t-\tau}}\\
    \overset{\text{(iv)}}{\leq}& \frac{G_\infty}{G_0}\cdot c_{\tilde{J}} \sum_{k=t-\tau}^{t-1}\norm{\theta_{k+1}-\theta_{k}}\\
    \overset{\text{(v)}}{\leq}& \frac{G_\infty}{G_0}\cdot \frac{c_{\tilde{J}} G_\infty}{G_0} \sum_{k=t-\tau}^{t-1}\alpha_k,
\end{align*}
where (i) follows from Cauchy-Schwarz inequality, (ii) follows because $\hat{v}_{t,i}\geq G_0^2$, (iii) follows from Lemma \ref{lem:pgLip} and Lemma \ref{lem:mvbound}, (iv) follows by the triangle inequality, and (v) follows due to Lemma \ref{lem:conThetaBound}.

Similarly, we can bound the term $T_7$ as follows:
\begin{align*}
    T_7 =& \nabla J(\theta_{t-\tau})^T \hat{V}_{t-1}^{-\frac{1}{2}}(\nabla J(\theta_{t-\tau})-\nabla J(\theta_{t}))\\
    \leq& \norm{\nabla J(\theta_{t-\tau})^T \hat{V}_{t-1}^{-\frac{1}{4}}  } \cdot \norm{\hat{V}_{t-1}^{-\frac{1}{4}}(\nabla J(\theta_{t-\tau})-\nabla J(\theta_{t}))  }\\
    \leq& \frac{1}{G_0} \norm{\nabla J(\theta_{t-\tau}) } \cdot \norm{(\nabla J(\theta_{t-\tau})-\nabla J(\theta_{t}))  }\\
    \leq& \frac{G_\infty}{G_0}\cdot c_{J}\norm{\theta_t-\theta_{t-\tau}}\\
    \leq& \frac{G_\infty}{G_0}\cdot c_{J} \sum_{k=t-\tau}^{t-1}\norm{\theta_{k+1}-\theta_{k}}\\
    \leq& \frac{G_\infty}{G_0}\cdot \frac{c_{J} G_\infty}{G_0} \sum_{k=t-\tau}^{t-1}\alpha_k.
\end{align*}

Next, we bound the term $T_8$ and obtain
\begin{align*}
    T_8 =& (\nabla J(\theta_{t})-\nabla J(\theta_{t-\tau}))^T \hat{V}_{t-1}^{-\frac{1}{2}}(\nabla J(\theta_t)-\nabla \tilde{J}(\theta_t;O_t))\\
    \leq& \norm{ (\nabla J(\theta_{t})-\nabla J(\theta_{t-\tau}))^T \hat{V}_{t-1}^{-\frac{1}{4}} }\cdot \norm{ \hat{V}_{t-1}^{-\frac{1}{4}}(\nabla J(\theta_t)-\nabla \tilde{J}(\theta_t;O_t)) }\\
    \leq& \frac{1}{G_0} \norm{\nabla J(\theta_{t-\tau})-\nabla J(\theta_{t})}\cdot \norm{\nabla J(\theta_{t})-\nabla \tilde{J}(\theta_{t};O_t)}\\
    \leq& \frac{1}{G_0} \left(\norm{\nabla J(\theta_{t})} + \norm{\nabla \tilde{J}(\theta_{t};O_t)}\right) \cdot  c_{J}\norm{\theta_t-\theta_{t-\tau}}\\
    \leq& \frac{2G_\infty}{G_0}\cdot \frac{c_{J} G_\infty}{G_0} \sum_{k=t-\tau}^{t-1}\alpha_k.
\end{align*}

Last, it remains to bound the term $\mathbb{E}[T_5]$. Observe that $\mathbb{E}[T_5] = \mathbb{E}[\mathbb{E}[T_5|\mathcal{F}_{t-\tau}]]$. We first deal with $\mathbb{E}[T_5|\mathcal{F}_{t-\tau}]$ as
\begin{align*}
    \mathbb{E}[T_5|\mathcal{F}_{t-\tau}] =& \mathbb{E}\left[ \nabla J(\theta_{t-\tau})^T \hat{V}_{t-1}^{-\frac{1}{2}}(\nabla J(\theta_{t-\tau})-\nabla \tilde{J}(\theta_{t-\tau};O_t))|\mathcal{F}_{t-\tau} \right]\\
    \overset{\text{(i)}}{\leq}& Tr\left( \mathbb{E}\left[\hat{V}_{t-1}^{-\frac{1}{2}}\rvert\mathcal{F}_{t-\tau} \right] \right) \norm{\nabla J(\theta_{t-\tau})}_\infty\cdot \norm{\nabla J(\theta_{t-\tau})-\mathbb{E}\left[\nabla \tilde{J} (\theta_{t-\tau};O_t)|\mathcal{F}_{t-\tau} \right]}_\infty\\
    \leq& \frac{d}{G_0} \norm{\nabla J(\theta_{t-\tau})}\cdot \norm{\nabla J(\theta_{t-\tau})-\mathbb{E}\left[\nabla \tilde{J} (\theta_{t-\tau};O_t)|\mathcal{F}_{t-\tau} \right]}\\
    \leq& \frac{d G_\infty}{G_0}\cdot \norm{\nabla J(\theta_{t-\tau})-\mathbb{E}\left[\nabla \tilde{J} (\theta_{t-\tau};O_t)|\mathcal{F}_{t-\tau} \right]},
\end{align*}
where (i) follows from Lemma \ref{lem:vecMatPro}.
The key to bound the bias is to bound the term $\norm{\nabla J(\theta_{t-\tau})-\mathbb{E}\left[\nabla \tilde{J} (\theta_{t-\tau};O_t)|\mathcal{F}_{t-\tau} \right]}$, which has been done in Lemma \ref{lem:pgBias}.

To conclude, the bias term ca be bounded for a fixed $\tau < t$ as:
\begin{align}\label{eq:pgthempf3}
    \mathbb{E}&\left[\nabla J(\theta_t)^T\hat{V}_{t-1}^{-\frac{1}{2}}(\nabla J(\theta_t)-\nabla \tilde{J}(\theta_t;O_t)) \right] \nonumber\\
    \leq&  \frac{G_\infty}{G_0}\left[ \frac{(3c_{J}+c_{\tilde{J}}) G_\infty}{G_0} \sum_{k=t-\tau}^{t-1}\alpha_k + d G_{\infty}\left[ \sigma\rho^\tau + \frac{L_\pi G_{\infty}}{2G_0} \left(\sum_{k=t-\tau}^{t-1}\alpha_k + \sum_{i=t-\tau}^{t-1}\sum_{k=t-\tau}^{i}\alpha_k\right) \right]\right].
\end{align}

\textbf{Step3: Establishing convergence to stationary point}

For the case with a constant stepsize $\alpha_t=\alpha$, we choose $\tau^{*} = \min\{\tau:\sigma\rho^{\tau} \leq \alpha \} $. To take the summation over the time steps, notice that the bound in \eqref{eq:pgthempf3} holds only when $t > \tau^*$, and hence we separate the summation of the bias term into two parts as follows:
\begin{align}\label{eq:pgthempf4}
     \sum_{t=1}^T&\frac{\alpha}{G_\infty}\mathbb{E}\left[\norm{\nabla J(\theta_t)}^2\right]\nonumber\\
     \leq& \sum_{t=1}^T\left(\mathbb{E}[J(z_t)] - \mathbb{E}[J(z_{t+1})] \right) + \frac{G_{\infty}^2}{1-\beta_1}  \sum_{t=1}^T\left(\mathbb{E}\left[Tr\left( \alpha\hat{V}_{t-1}^{-\frac{1}{2}}\right)\right] -\mathbb{E}\left[ Tr\left(\alpha\hat{V}_{t}^{-\frac{1}{2}} \right)\right] \right)\nonumber\\
    &+ 2c_J\sum_{t=1}^T\mathbb{E}\norm{\alpha\hat{V}_{t}^{-\frac{1}{2}}g_t}^2 + 3c_J\left(\frac{\beta_1}{1-\beta_1}\right)^2\sum_{t=1}^T\mathbb{E}\norm{\alpha\hat{V}_{t}^{-\frac{1}{2}}m_t}^2\nonumber\\
    &+ \sum_{t=1}^T\alpha\mathbb{E}\left[\nabla J(\theta_t)^T\hat{V}_{t-1}^{-\frac{1}{2}}(\nabla J(\theta_t)-\nabla \tilde{J}(\theta_t;O_t)) \right]\nonumber\\
    =& \sum_{t=1}^T\left(\mathbb{E}[J(z_t)] - \mathbb{E}[J(z_{t+1})] \right) + \frac{G_{\infty}^2}{1-\beta_1}  \sum_{t=1}^T\left(\mathbb{E}\left[Tr\left( \alpha\hat{V}_{t-1}^{-\frac{1}{2}}\right)\right] -\mathbb{E}\left[ Tr\left(\alpha\hat{V}_{t}^{-\frac{1}{2}} \right)\right] \right)\nonumber\\
    &+ 2c_J\sum_{t=1}^T\mathbb{E}\norm{\alpha\hat{V}_{t}^{-\frac{1}{2}}g_t}^2 + 3c_J\left(\frac{\beta_1}{1-\beta_1}\right)^2\sum_{t=1}^T\mathbb{E}\norm{\alpha\hat{V}_{t}^{-\frac{1}{2}}m_t}^2\nonumber\\
    &+ \sum_{t=1}^{\tau^*}\alpha\mathbb{E}\left[\nabla J(\theta_t)^T\hat{V}_{t-1}^{-\frac{1}{2}}(\nabla J(\theta_t)-\nabla \tilde{J}(\theta_t;O_t)) \right] + \sum_{t=\tau^*+1}^{T}\alpha\mathbb{E}\left[\nabla J(\theta_t)^T\hat{V}_{t-1}^{-\frac{1}{2}}(\nabla J(\theta_t)-\nabla \tilde{J}(\theta_t;O_t)) \right],
\end{align}
Then applying \eqref{eq:pgthempf3} to \eqref{eq:pgthempf4} yields
\begin{align}\label{eq:pgthm1res}
     \sum_{t=1}^T\frac{\alpha}{G_\infty}\mathbb{E}\left[\norm{\nabla J(\theta_t)}^2\right]
     \leq& \sum_{t=1}^T\left(\mathbb{E}[J(z_t)] - \mathbb{E}[J(z_{t+1})] \right) + \frac{G_{\infty}^2}{1-\beta_1}  \sum_{t=1}^T\left(\mathbb{E}\left[Tr\left( \alpha\hat{V}_{t-1}^{-\frac{1}{2}}\right)\right] -\mathbb{E}\left[ Tr\left(\alpha\hat{V}_{t}^{-\frac{1}{2}} \right)\right] \right)\nonumber\\
    &+ 2c_J\sum_{t=1}^T\mathbb{E}\norm{\alpha\hat{V}_{t}^{-\frac{1}{2}}g_t}^2 + 3c_J\left(\frac{\beta_1}{1-\beta_1}\right)^2\sum_{t=1}^T\mathbb{E}\norm{\alpha\hat{V}_{t}^{-\frac{1}{2}}m_t}^2\nonumber\\
    &+ \sum_{t=1}^T\alpha\mathbb{E}\left[\nabla J(\theta_t)^T\hat{V}_{t-1}^{-\frac{1}{2}}(\nabla J(\theta_t)-\nabla \tilde{J}(\theta_t;O_t)) \right]\nonumber\\
    \overset{\text{(i)}}{\leq}& \mathbb{E}[J(z_1)] + \frac{G_{\infty}^2}{1-\beta_1} \mathbb{E}\left[Tr\left( \alpha_{1}\hat{V}_{1}^{-\frac{1}{2}}\right)\right] + \frac{2dc_J\alpha^2}{1-\beta_2} + \frac{3dc_J\beta_1^2\alpha^2}{(1-\beta_1)(1-\beta_2)(1-\beta_1/\beta_2)}\nonumber\\
    &+ \sum_{t=1}^{\tau^*}\alpha\mathbb{E}\left[\nabla J(\theta_t)^T\hat{V}_{t-1}^{-\frac{1}{2}}(\nabla J(\theta_t)-\nabla \tilde{J}(\theta_t;O_t)) \right]\nonumber\\
    &+ \frac{\alpha G_\infty}{G_0} \sum_{t=\tau^*+1}^T\left[ \frac{(3c_{J}+c_{\tilde{J}}) G_\infty\tau^*}{G_0} \alpha + d G_{\infty}\left( \sigma\rho^{\tau^*} + \frac{L_\pi G_{\infty} (2\tau^*+(\tau^*)^2)}{2G_0}\alpha  \right)\right]\nonumber\\
    \overset{\text{(ii)}}{\leq}& \mathbb{E}[J(z_1)] + \frac{\alpha d G_{\infty}^2}{G_0(1-\beta_1)}  + \frac{2dc_J\alpha^2}{1-\beta_2} + \frac{3dc_J\beta_1^2\alpha^2}{(1-\beta_1)(1-\beta_2)(1-\beta_1/\beta_2)}\nonumber\\
    & + \sum_{t=1}^{\tau^*}\alpha\mathbb{E}\left[\nabla J(\theta_t)^T\hat{V}_{t-1}^{-\frac{1}{2}}(\nabla J(\theta_t)-\nabla \tilde{J}(\theta_t;O_t)) \right]\nonumber\\
    &+ \frac{\alpha^2 G_\infty}{G_0} \left[ \frac{(3c_{J}+c_{\tilde{J}}) G_\infty\tau^*}{G_0}  + d G_{\infty}\left( 1 + \frac{L_\pi G_{\infty} (2\tau^*+(\tau^*)^2)}{2G_0}  \right)\right](T-\tau^*)\nonumber\\
    \overset{\text{(iii)}}{\leq}& \mathbb{E}[J(z_1)] +  \frac{\alpha d G_{\infty}^2}{G_0(1-\beta_1)}  +  \frac{dc_J\alpha^2(3\beta_1^2+2(1-\beta_1)(1-\beta_1/\beta_2))}{(1-\beta_1)(1-\beta_2)(1-\beta_1/\beta_2)}  +  \frac{2\alpha\tau^*}{G_0} G_{\infty}^2\nonumber\\
    &+ \frac{\alpha^2 G_\infty}{G_0} \left[ \frac{(3c_{J}+c_{\tilde{J}}) G_\infty\tau^*}{G_0}  + d G_{\infty}\left( 1 + \frac{L_\pi G_{\infty} (2\tau^*+(\tau^*)^2)}{2G_0}  \right)\right](T-\tau^*),
\end{align}
where (i) follows from Lemma \ref{lem:pgAMSGradVm}, (ii) follows from the definition of $\tau^*$ and (iii) follows since 
\begin{align*}
\nabla J(\theta_t)^T\hat{V}_{t-1}^{-\frac{1}{2}}(\nabla J(\theta_t)-\nabla \tilde{J}(\theta_t;O_t)) &\leq \norm{\nabla J(\theta_t)^T \hat{V}_{t-1}^{-\frac{1}{4}}  } \cdot \norm{\hat{V}_{t-1}^{-\frac{1}{4}}(\nabla J(\theta_{t})-\nabla \tilde{J}(\theta_{t}))  }\\
&\leq \frac{1}{G_0}\norm{ \nabla J(\theta_t) }\left( \norm{ \nabla J(\theta_t) } + \norm{ \nabla \tilde{J}(\theta_t;O_t) } \right)\\
&\leq \frac{2}{G_0} G_{\infty}^2.
\end{align*}

Finally, we complete our proof by letting both sides of \eqref{eq:pgthm1res} be divided by $T$, which yields
\begin{align*}
    \underset{t\in [T]}{\min}\mathbb{E}\left[\norm{\nabla J(\theta_t)}^2\right] \leq \frac{1}{T}\sum_{t=1}^T\mathbb{E}\left[\norm{\nabla J(\theta_t)}^2\right]\leq \frac{C_1}{T} + \frac{\alpha C_2(T-\tau^*)}{T}\leq  \frac{C_1}{T} + \alpha C_2,
\end{align*}
where 
\begin{align*}
    &C_1 = \frac{G_{\infty}\mathbb{E}[J(z_1)]}{\alpha} +  \frac{d G_{\infty}^3}{G_0(1-\beta_1)}  +  \frac{dc_J\alpha G_\infty(3\beta_1^2+2(1-\beta_1)(1-\beta_1/\beta_2))}{(1-\beta_1)(1-\beta_2)(1-\beta_1/\beta_2)}  +  \frac{2G_\infty\tau^*}{G_0} G_{\infty}^2,\\
    &C_2 = \frac{G_\infty^2}{G_0} \left[ \frac{(3c_{J}+c_{\tilde{J}}) G_\infty\tau^*}{G_0}  + d G_{\infty}\left( 1 + \frac{L_\pi G_{\infty} (2\tau^*+(\tau^*)^2)}{2G_0}  \right)\right].
\end{align*}

\section{Proof of Theorem \ref{thm:AMSGradPGdimi} }

The proof of Theorem \ref{thm:AMSGradPGdimi} starts from steps similar to those of Theorem \ref{thm:AMSGradPGcons}. The difference starts from \eqref{eq:pgthempf4}. Here we consider $\alpha_t$ is not constant. Then we divide both sides of \eqref{eq:pgthempf4} by $\alpha_{t-1}$ and obtain
\begin{align*}
     \frac{1}{G_\infty}&\mathbb{E}\left[\norm{\nabla J(\theta_t)}^2\right]\\
     &\leq \frac{\left(\mathbb{E}[J(z_t)] - \mathbb{E}[J(z_{t+1})] \right)}{\alpha_{t-1}} + \frac{G_{\infty}^2}{\alpha_{t-1}(1-\beta_1)}  \left(\mathbb{E}\left[Tr\left( \alpha_{t-1}\hat{V}_{t-1}^{-\frac{1}{2}}\right)\right] -\mathbb{E}\left[ Tr\left(\alpha_t\hat{V}_{t}^{-\frac{1}{2}} \right)\right] \right)\\
    &\quad + \frac{2c_J}{\alpha_{t-1}}\mathbb{E}\norm{\alpha_t\hat{V}_{t}^{-\frac{1}{2}}g_t}^2 + \frac{3c_J}{\alpha_{t-1}}\left(\frac{\beta_1}{1-\beta_1}\right)^2\mathbb{E}\norm{\alpha_t\hat{V}_{t}^{-\frac{1}{2}}m_t}^2\\
    &\quad + \mathbb{E}\left[\nabla J(\theta_t)^T\hat{V}_{t-1}^{-\frac{1}{2}}(\nabla J(\theta_t)-\nabla \tilde{J}(\theta_t;O_t)) \right]\\
    &:= \frac{f_t-f_{t+1}}{\alpha_{t-1}} + \frac{2c_J}{\alpha_{t-1}}\mathbb{E}\norm{\alpha_t\hat{V}_{t}^{-\frac{1}{2}}g_t}^2 + \frac{3c_J}{\alpha_{t-1}}\left(\frac{\beta_1}{1-\beta_1}\right)^2\mathbb{E}\norm{\alpha_t\hat{V}_{t}^{-\frac{1}{2}}m_t}^2\\
    &\quad + \mathbb{E}\left[\nabla J(\theta_t)^T\hat{V}_{t-1}^{-\frac{1}{2}}(\nabla J(\theta_t)-\nabla \tilde{J}(\theta_t;O_t)) \right],
\end{align*}
where 
$$ f_t = \mathbb{E}[J(z_t)] + \frac{G_{\infty}^2}{1-\beta_1}  \mathbb{E}\left[Tr\left( \alpha_{t-1}\hat{V}_{t-1}^{-\frac{1}{2}}\right)\right].
$$

We choose $\tau^{*} = \min\{\tau:\sigma\rho^{\tau} \leq \alpha_T \} $. Again we choose $\tau = t$ if $t\leq\tau^{*}$ and $\tau = \tau^{*}$ if $t>\tau^{*}$. If $t>\tau^*$, then choice of $\alpha_t$ yields
\begin{align*}
    &\mathbb{E}\left[\nabla J(\theta_t)^T\hat{V}_{t-1}^{-\frac{1}{2}}(\nabla J(\theta_t)-\nabla \tilde{J}(\theta_t;O_t)) \right]\\
    &\quad\leq \frac{G_\infty}{G_0}\left[ \frac{(3c_{J}+c_{\tilde{J}}) G_\infty}{G_0}\sum_{k=t-\tau^* }^{t-1}\alpha_k + d G_{\infty}\left[ \sigma\rho^{\tau^*} + \frac{L_\pi G_{\infty}}{2G_0} \left(\sum_{k=t-\tau^* }^{t-1}\alpha_k + \sum_{i=t-\tau^*}^{t-1}\sum_{k=t-\tau^*  }^{i}\alpha_k\right) \right]\right]\\
    &\quad\leq \frac{G_\infty}{G_0}\left[ \frac{(3c_{J}+c_{\tilde{J}}) G_\infty}{G_0} \frac{\alpha\tau^*}{\sqrt{t-\tau^{*}}} + d G_{\infty}\left( \frac{\alpha}{\sqrt{T}} + \frac{L_\pi G_{\infty} (1+\tau^*)}{2G_0} \frac{\alpha\tau^*}{\sqrt{t-\tau^{*}}} \right)\right],
\end{align*}
where the last inequality follows because
\begin{equation}\label{eq:alphaSum}
    \sum_{k=t-\tau^{*}}^{t}\alpha_k = \alpha\sum_{k=t-\tau^{*}}^{t}\frac{1}{\sqrt{k}} \leq \alpha\int_{x=t-\tau^{*}+1}^{t+1} \frac{1}{\sqrt{x-1}}dx\leq 2\alpha(\sqrt{t}-\sqrt{t-\tau^{*}})\leq \frac{\alpha\tau^*}{\sqrt{t-\tau^{*}}} .
\end{equation}

By taking the summation over time steps, we obtain
\begin{align}\label{eq:pgthempf6}
     \frac{1}{G_\infty}\sum_{t=1}^T&\mathbb{E}\left[\norm{\nabla J(\theta_t)}^2\right]\nonumber\\
     \overset{\text{(i)}}{\leq}& \sum_{t=1}^T\frac{f_t-f_{t+1}}{\alpha_{t-1}} + 2c_J\alpha\sum_{t=1}^T\mathbb{E}\norm{\hat{V}_{t}^{-\frac{1}{2}}g_t}^2 + 3c_J\alpha\sum_{t=1}^T\left(\frac{\beta_1}{1-\beta_1}\right)^2\mathbb{E}\norm{\hat{V}_{t}^{-\frac{1}{2}}m_t}^2\nonumber\\
     &+ \sum_{t=1}^{\tau^*}\mathbb{E}\left[\nabla J(\theta_t)^T\hat{V}_{t-1}^{-\frac{1}{2}}(\nabla J(\theta_t)-\nabla \tilde{J}(\theta_t;O_t)) \right]\nonumber\\
    & + \frac{\alpha G_\infty}{G_0} \sum_{t=\tau^*+1}^T \left[ \frac{(3c_{J}+c_{\tilde{J}}) G_\infty}{G_0} \frac{\tau^*}{\sqrt{t-\tau^*}} \!+\! d G_{\infty}\left( \frac{1}{\sqrt{T}} \!+\! \frac{L_\pi G_{\infty} (1+\tau^*)}{2G_0} \frac{\tau^*}{\sqrt{t-\tau^*}} \right)\right]\nonumber\\
    \overset{\text{(ii)}}{\leq}& \sum_{t=1}^T\frac{f_t-f_{t+1}}{\alpha_{t-1}} +\frac{2dc_J\alpha}{1-\beta_2} + \frac{3dc_J\beta_1^2\alpha}{(1-\beta_1)(1-\beta_2)(1-\beta_1/\beta_2)}\nonumber\\
    & + \sum_{t=1}^{\tau^*}\mathbb{E}\left[\nabla J(\theta_t)^T\hat{V}_{t-1}^{-\frac{1}{2}}(\nabla J(\theta_t)-\nabla \tilde{J}(\theta_t;O_t)) \right]\nonumber\\
    & + \frac{\alpha G_\infty}{G_0} \left[ \frac{2(3c_{J}+c_{\tilde{J}}) G_\infty\tau^*\sqrt{T-\tau^*}}{G_0} \!+\! d G_{\infty}\left( \sqrt{T} \!+\! \frac{L_\pi G_{\infty} (\tau^*+(\tau^*)^2)\sqrt{T-\tau^*}}{G_0}  \right)\right]\nonumber\\
    \leq& \sum_{t=1}^T\frac{f_t-f_{t+1}}{\alpha_{t-1}} +\frac{2dc_J\alpha}{1-\beta_2} + \frac{3dc_J\beta_1^2\alpha}{(1-\beta_1)(1-\beta_2)(1-\beta_1/\beta_2)} + \frac{2\tau^*}{G_0} G_{\infty}^2\nonumber\\
    & + \frac{\alpha G_\infty}{G_0} \left[ \frac{2(3c_{J}+c_{\tilde{J}}) G_\infty\tau^*\sqrt{T-\tau^*}}{G_0} \!+\! d G_{\infty}\left( \sqrt{T} \!+\! \frac{L_\pi G_{\infty} (\tau^*+(\tau^*)^2)\sqrt{T-\tau^*}}{G_0}  \right)\right]\nonumber\\
    \leq& \sum_{t=1}^T\frac{f_t-f_{t+1}}{\alpha_{t-1}} +\frac{2dc_J\alpha}{1-\beta_2} + \frac{3dc_J\beta_1^2\alpha}{(1-\beta_1)(1-\beta_2)(1-\beta_1/\beta_2)} + \frac{2\tau^*}{G_0} G_{\infty}^2\nonumber\\
    & + \frac{\alpha G_\infty}{G_0} \left[ \frac{2(3c_{J}+c_{\tilde{J}}) G_\infty\tau^*}{G_0} \!+\! d G_{\infty}\left( 1 \!+\! \frac{L_\pi G_{\infty} (\tau^*+(\tau^*)^2)}{G_0}  \right)\right]\cdot \sqrt{T},
\end{align}
where (i) follows since $\alpha_t$ is decreasing and from the definition of $\tau^*$, and (ii) follows from Lemma \ref{lem:pgAMSGradVm}.

For clarity we bound $\sum_{t=1}^T\frac{f_t-f_{t+1}}{\alpha_{t-1}}$ separately. The key observation is that $f_t$ is uniformly bounded as
\begin{equation}
    f_t = \mathbb{E}[J(z_t)] + \frac{G_{\infty}^2}{1-\beta_1}  \mathbb{E}\left[Tr\left( \alpha_{t-1}\hat{V}_{t-1}^{-\frac{1}{2}}\right)\right]
    \leq \frac{R_{\max}}{1-\gamma} + \frac{\alpha d G_{\infty}^2}{G_0(1-\beta_1)}. 
\end{equation}
Thus
\begin{align}\label{eq:pgthempf7}
    \sum_{t=1}^T\frac{f_t-f_{t+1}}{\alpha_{t-1}} &= \frac{f_1}{\alpha_0} + \sum_{t=2}^{T}f_t\left( \frac{1}{\alpha_{t-1}} - \frac{1}{\alpha_{t-2}} \right) -\frac{f_{T+1}}{\alpha_{T-1}}\nonumber\\
    &\leq \frac{f_1}{\alpha_0} + \left( \frac{R_{\max}}{1-\gamma} + \frac{\alpha dG_{\infty}^2}{G_0(1-\beta_1)} \right)\sum_{t=2}^{T}\left( \frac{1}{\alpha_{t-1}} - \frac{1}{\alpha_{t-2}} \right)\nonumber\\
    &\leq \frac{f_1}{\alpha_0} + \left( \frac{R_{\max}}{1-\gamma} + \frac{\alpha dG_{\infty}^2}{G_0(1-\beta_1)} \right)/\alpha_{T-1}\nonumber\\
    &= \frac{f_1}{\alpha} + \left( \frac{R_{\max}}{1-\gamma} + \frac{\alpha dG_{\infty}^2}{G_0(1-\beta_1)} \right)\frac{\sqrt{T-1}}{\alpha}.
\end{align}

Finally, we complete our proof by substituting \eqref{eq:pgthempf7} in \eqref{eq:pgthempf6} and letting both sides of \eqref{eq:pgthempf6} be divided by $T$, which yields
\begin{align*}
    \underset{t\in [T]}{\min}\mathbb{E}\left[\norm{\nabla J(\theta_t)}^2\right] \leq
    \frac{1}{T}\sum_{t=1}^T\mathbb{E}\left[\norm{\nabla J(\theta_t)}^2\right]\leq \frac{C_1}{T} + \frac{\alpha C_2(T-\tau^*)}{T}\leq  \frac{C_1}{T} + \frac{C_2}{\sqrt{T}},
\end{align*}
where
\begin{align*}
     C_1 &= \frac{f_1 G_\infty}{\alpha} + \frac{2dc_J\alpha G_\infty}{1-\beta_2} + \frac{3dc_J\beta_1^2\alpha G_\infty}{(1-\beta_1)(1-\beta_2)(1-\beta_1/\beta_2)} + \frac{2\tau^* G_\infty}{G_0} G_{\infty}^2\\
     C_2 &= \frac{R_{\max}G_\infty}{\alpha(1-\gamma)} + \frac{ dG_{\infty}^3}{ G_0(1-\beta_1)} + \frac{\alpha  G_\infty^3}{G_0} \left[ \frac{2(3c_{J}+c_{\tilde{J}}) \tau^*}{G_0} \!+\! d \left( 1 + \frac{L_\pi G_{\infty} (\tau^*+(\tau^*)^2)}{G_0}  \right)\right].
\end{align*}

\section{Proof of Proposition \ref{prop:PGSGD}}

In the following, we show how to adapt the proof techniques in analyzing PG-AMSGrad to PG-SGD. We first reduce Lemma \ref{lem:pgBias} to the vanilla PG case.
\begin{lemma}\label{lem:PGSGDBias}
Fix time $t$ and any $\tau < t$. Suppose Assumptions \ref{asp:policyLip} and \ref{asp:markov} hold for PG-SGD. Then we have
\begin{equation}
    \norm{\nabla J(\theta_{t-\tau})-\mathbb{E}\left[\nabla \tilde{J} (\theta_{t-\tau};s_t,a_t)|\mathcal{F}_{t-\tau} \right]}\leq G_\infty\left[ \sigma\rho^\tau + \frac{L_\pi G_\infty}{2} \left(\sum_{k=t-\tau}^{t-1}\alpha_k + \sum_{i=t-\tau}^{t-1}\sum_{k=t-\tau}^{i}\alpha_k\right) \right],
\end{equation}
where $G_\infty = \frac{c_{\Theta} R_{\max}}{1-\gamma}$.
\end{lemma}
\begin{proof}
since the major part of the proof is similar to that of Lemma \ref{lem:pgBias}, we only emphasize the different steps. 

For notational brevity, we denote $O_t = (s_t,a_T)$. Then we still have
\begin{align}
    &\norm{\nabla J(\theta_{t-\tau})-\mathbb{E}\left[\nabla \tilde{J} (\theta_{t-\tau};O_t)|\mathcal{F}_{t-\tau} \right]} \leq G_\infty\lTV{\mu_{\theta_{t-\tau}}(\cdot,\cdot)-P(s_t,a_t|\mf_{t-\tau})}. \label{eq:PGSGDpf1}
\end{align}
Then we use the steps similar to those in building an auxiliary Markov chain \eqref{eq:newMarChain} which is generated by the policy $\pi_{\theta_{t-\tau}}$ from time $t-\tau$. Similarly to \eqref{eq:pfReduce1}, we have
\begin{align}
    &\lTV{\mu_{\theta_{t-\tau}}(\cdot,\cdot)-P(s_t,a_t|\mf_{t-\tau})}\nonumber\\
    &\quad\leq \lTV{\mu_{\theta_{t-\tau}}(\cdot,\cdot)-P(\tilde{s}_t,\tilde{a}_t|\mf_{t-\tau})} + \lTV{P(\tilde{s}_t,\tilde{a}_t|\mf_{t-\tau})-P(s_t,a_t|\mf_{t-\tau})}\nonumber\\
    &\quad\overset{\text{(i)}}{\leq} \sigma \rho^{\tau} + \frac{1}{2}L_\pi\sum_{i=t-\tau}^{t-1}\norm{\theta_{i+1}-\theta_i} + \frac{1}{2}L_\pi\sum_{i=t-\tau}^{t-1}\sum_{k=t-\tau}^{i}\norm{\theta_{k+1}-\theta_k}, \label{eq:PGSGDpf1_1}
\end{align}
where (i) follows from \eqref{incre13}, \eqref{incre14} and \eqref{incre18}. Observe that in PG-SGD, for any $t$ we have,
$$ \norm{\theta_{t+1}-\theta_t} = \alpha_t\norm{g_t}\leq G_\infty\alpha_t.
$$
Thus, we complete the proof by substituting the above observation to \eqref{eq:PGSGDpf1_1} and then \eqref{eq:PGSGDpf1}.

\end{proof}

\textbf{Proof of Proposition \ref{prop:PGSGD}:} Following from the Lipschitz condition of $\nabla J(\theta)$ in Lemma \ref{lem:pgLip}, we obtain
\begin{align*}
    J(\theta_{t+1}) \leq& J(\theta_t) + \nabla J(\theta_t)^T(\theta_{t+1} - \theta_t) + \frac{c_J}{2}\norm{\theta_{t+1} - \theta_t}^2\nonumber\\
    =& J(\theta_t) - \alpha_t\nabla J(\theta_t)^T g_t + \frac{c_J}{2}\norm{\theta_{t+1} - \theta_t}^2\nonumber\\
    =& J(\theta_t) - \alpha_t\nabla J(\theta_t)^T \nabla J(\theta_t) + \alpha_t\nabla J(\theta_t)^T(\nabla J(\theta_t) - g_t) + \frac{c_J}{2}\norm{\theta_{t+1} - \theta_t}^2\nonumber.
\end{align*}
Then we rearrange the above inequality and take expectation over all the randomness to have
\begin{align}
    \alpha_t \mathbb E\left[ \norm{\nabla J(\theta_t)}^2 \right] \leq& 
    \mathbb E[J(\theta_t)] - \mathbb E[J(\theta_{t+1})] + \alpha_t\mathbb E\left[ \nabla J(\theta_t)^T(\nabla J(\theta_t) - g_t)\right] + \frac{c_J}{2}\mathbb E\left[ \norm{\theta_{t+1} - \theta_t}^2\right]\nonumber\\
    \leq& \mathbb E[J(\theta_t)] - \mathbb E[J(\theta_{t+1})] + \alpha_t\mathbb E\left[ \nabla J(\theta_t)^T(\nabla J(\theta_t) - g_t)\right] + \frac{c_J G_\infty^2}{2}\alpha_t^2\nonumber\\
    =& \mathbb E[J(\theta_t)] - \mathbb E[J(\theta_{t+1})] + \alpha_t\mathbb E\left[ \mathbb{E}\left[\nabla J(\theta_t)^T(\nabla J(\theta_t)-g_t)\rvert \mathcal{F}_{t}^q \right] \right] + \frac{c_J G_\infty^2}{2}\alpha_t^2\nonumber\\
    =& \mathbb E[J(\theta_t)] - \mathbb E[J(\theta_{t+1})] + \alpha_t\mathbb E\left[ \nabla J(\theta_t)^T(\nabla J(\theta_t) - \nabla \tilde{J}(\theta_t;O_t))\right] + \frac{c_J G_\infty^2}{2}\alpha_t^2, \label{eq:PGSGDpf2}
\end{align}
where the filtration $\mathcal{F}_{t}^q$ contains all the samples up to time $t$ except the samples for estimating $Q^{\pi_{\theta_t}}(s_t,a_t)$ in the EstQ algorithm. Thus we have $\mathbb{E} [g_t\rvert \mathcal{F}_{t}^q] = \nabla \tilde{J}(\theta_t;O_t)$ where the expectation is taken over the randomness in EstQ algorithm.
Observe that if $\tau < t$, we have
\begin{align*}
    &\mathbb E\left[ \nabla J(\theta_t)^T(\nabla J(\theta_t) - \nabla \tilde{J}(\theta_t;O_t))\right]\\
    &\quad= \mathbb E\left[\underbrace{\nabla J(\theta_{t-\tau})^T (\nabla J(\theta_{t-\tau})-\nabla \tilde{J}(\theta_{t-\tau};O_t))}_{T1} + \underbrace{\nabla J(\theta_{t-\tau})^T (\nabla \tilde{J}(\theta_{t-\tau};O_t)-\nabla \tilde{J}(\theta_{t};O_t))}_{T_2}\right.  \\
    &\quad\quad \left. + \underbrace{\nabla J(\theta_{t-\tau})^T (\nabla J(\theta_{t})-\nabla J(\theta_{t-\tau}))}_{T_3}  + \underbrace{(\nabla J(\theta_{t})-\nabla J(\theta_{t-\tau}))^T (\nabla J(\theta_t)-\nabla \tilde{J}(\theta_t;O_t))}_{T_4}\right].
\end{align*}
Next we bound the terms $T_1,T_2,T_3$ and $T_4$.

Observe that $\mathbb{E}[T_1] = \mathbb{E}[\mathbb{E}[T_1|\mathcal{F}_{t-\tau}]]$. We deal with $\mathbb{E}[T_1|\mathcal{F}_{t-\tau}]$ and have
\begin{align*}
    \mathbb{E}[T_1|\mathcal{F}_{t-\tau}] =& \mathbb{E}\left[ \nabla J(\theta_{t-\tau})^T (\nabla J(\theta_{t-\tau})-\nabla \tilde{J}(\theta_{t-\tau};O_t))|\mathcal{F}_{t-\tau} \right]\\
    \leq& \norm{\nabla J(\theta_{t-\tau})}\cdot \norm{\nabla J(\theta_{t-\tau})-\mathbb{E}\left[\nabla \tilde{J} (\theta_{t-\tau};O_t)|\mathcal{F}_{t-\tau} \right]}\\
    \leq& G_\infty\cdot\norm{\nabla J(\theta_{t-\tau})-\mathbb{E}\left[\nabla \tilde{J} (\theta_{t-\tau};O_t)|\mathcal{F}_{t-\tau} \right]}\\
    \leq& G_\infty^2\left[ \sigma\rho^\tau + \frac{L_\pi G_\infty}{2} \left(\sum_{k=t-\tau}^{t-1}\alpha_k + \sum_{i=t-\tau}^{t-1}\sum_{k=t-\tau}^{i}\alpha_k\right) \right].
\end{align*}

Next we bound $T_2$ and obtain
\begin{align*}
    T_2 =& \nabla J(\theta_{t-\tau})^T(\nabla \tilde{J}(\theta_{t-\tau};O_t)-\nabla \tilde{J}(\theta_{t};O_t))\\
    \leq& \norm{\nabla J(\theta_{t-\tau}) } \cdot \norm{(\nabla \tilde{J}(\theta_{t-\tau};O_t)-\nabla \tilde{J}(\theta_{t};O_t))  }\\
    \leq& G_\infty\cdot c_{\tilde{J}}\norm{\theta_t-\theta_{t-\tau}}\\
    \leq& G_\infty\cdot c_{\tilde{J}} \sum_{k=t-\tau}^{t-1}\norm{\theta_{k+1}-\theta_{k}}\\
    \leq& c_{\tilde{J}} G_\infty^2 \sum_{k=t-\tau}^{t-1}\alpha_k.
\end{align*}

Similarly, we can bound term $T_3$ as
\begin{align*}
    T_3 =& \nabla J(\theta_{t-\tau})^T (\nabla J(\theta_{t-\tau})-\nabla J(\theta_{t}))\\
    \leq& \norm{\nabla J(\theta_{t-\tau}) } \cdot \norm{(\nabla J(\theta_{t-\tau})-\nabla J(\theta_{t}))  }\\
    \leq& G_\infty\cdot c_{J}\norm{\theta_t-\theta_{t-\tau}}\\
    \leq& G_\infty\cdot c_{J} \sum_{k=t-\tau}^{t-1}\norm{\theta_{k+1}-\theta_{k}}\\
    \leq& c_{J} G_\infty^2 \sum_{k=t-\tau}^{t-1}\alpha_k.
\end{align*}

Last, we bound term $T_4$ and obtain
\begin{align*}
    T_4 =& (\nabla J(\theta_{t})-\nabla J(\theta_{t-\tau}))^T (\nabla J(\theta_t)-\nabla \tilde{J}(\theta_t;O_t))\\
    \leq& \norm{\nabla J(\theta_{t-\tau})-\nabla J(\theta_{t})}\cdot \norm{\nabla J(\theta_{t})-\nabla \tilde{J}(\theta_{t};O_t)}\\
    \leq& \left(\norm{\nabla J(\theta_{t})} + \norm{\nabla \tilde{J}(\theta_{t};O_t)}\right) \cdot  c_{J}\norm{\theta_t-\theta_{t-\tau}}\\
    \leq& 2c_{J} G_\infty^2 \sum_{k=t-\tau}^{t-1}\alpha_k.
\end{align*}

Thus for any $\tau<t$, we have
\begin{align}
    &\mathbb E\left[ \nabla J(\theta_t)^T(\nabla J(\theta_t) - \nabla \tilde{J}(\theta_t;O_t))\right]\nonumber\\
    &\quad\leq G_\infty^2\left[(3c_{J}+c_{\tilde{J}})\sum_{k=t-\tau}^{t-1}\alpha_k + \sigma\rho^\tau + \frac{L_\pi G_\infty}{2} \left(\sum_{k=t-\tau}^{t-1}\alpha_k + \sum_{i=t-\tau}^{t-1}\sum_{k=t-\tau}^{i}\alpha_k\right) \right]. \label{eq:PGSGDpf3}
\end{align}


\textbf{Proof of convergence under constant stepsize:}

We first consider a constant stepsize $\alpha_t=\alpha$ for $t=1,\dots,T$. Choose $\tau^{*} = \min\{\tau:\sigma\rho^{\tau} \leq \alpha \} $. We rearrange \eqref{eq:PGSGDpf2} and take the summation over the time steps to obtain
\begin{align*}
     \alpha\sum_{t=1}^T\mathbb{E}\left[\norm{\nabla J(\theta_t)}^2\right]\nonumber
     \leq& \sum_{t=1}^T\left(\mathbb{E}[J(\theta_t)] - \mathbb{E}[J(\theta_{t+1})] \right) + \sum_{t=1}^T\frac{c_J G_\infty^2}{2}\alpha^2 + \alpha\sum_{t=1}^T\mathbb E\left[ \nabla J(\theta_t)^T(\nabla J(\theta_t) - \nabla \tilde{J}(\theta_t;O_t))\right]\nonumber\\
     \leq& J(\theta_1) + \frac{c_J G_\infty^2\alpha^2}{2} T + \alpha\sum_{t=1}^{\tau^*}\mathbb E\left[ \nabla J(\theta_t)^T(\nabla J(\theta_t) - \nabla \tilde{J}(\theta_t;O_t))\right]\nonumber\\
     &+ \alpha\sum_{t=\tau^*+1}^{T}\mathbb E\left[ \nabla J(\theta_t)^T(\nabla J(\theta_t) - \nabla \tilde{J}(\theta_t;O_t))\right]\nonumber\\
     \overset{\text{(i)}}{\leq}& J(\theta_1) + \frac{c_J G_\infty^2\alpha^2}{2} T + 2\alpha G_\infty^2\tau^* + \alpha\sum_{t=\tau^*+1}^{T}\mathbb E\left[ \nabla J(\theta_t)^T(\nabla J(\theta_t) - \nabla \tilde{J}(\theta_t;O_t))\right]\nonumber\\
     \overset{\text{(ii)}}{\leq}& J(\theta_1) + \frac{c_J G_\infty^2\alpha^2}{2} T + 2\alpha G_\infty^2\tau^* + \alpha^2 G_\infty^2 \sum_{t=\tau^*+1}^{T}\left[(3c_{J}+c_{\tilde{J}})\tau^* + 1 + \frac{L_\pi G_\infty\tau^*(2+\tau^*)}{2}  \right],
\end{align*}
where (i) follows because $  \nabla J(\theta_t)^T(\nabla J(\theta_t) - \nabla \tilde{J}(\theta_t;O_t))\leq \norm{\nabla J(\theta_t)}\cdot\norm{\nabla J(\theta_t) - \nabla \tilde{J}(\theta_t;O_t)}\leq 2G_\infty^2 $, and (ii) follows from \eqref{eq:PGSGDpf3}.
Then we divide both sides of the above inequality by $\alpha T$, and have
\begin{align*}
    \underset{t\in [T]}{\min}\mathbb{E}\left[\norm{\nabla J(\theta_t)}^2\right] &\leq
    \frac{1}{T}\sum_{t=1}^T\mathbb{E}\left[\norm{\nabla J(\theta_t)}^2\right]\\
    &\leq \frac{J(\theta_1)/\alpha + 2G_{\infty}^2\tau^*}{T} + \alpha  G_\infty^2 \left[ \frac{c_J}{2} +  (3c_{J}+c_{\tilde{J}}) \tau^* +  1 + \frac{L_\pi G_{\infty} (2\tau^*+(\tau^*)^2)}{2}  \right].
\end{align*}

\textbf{Proof of convergence under diminishing stepsize:}

Now we consider a diminishing stepsize $\alpha_t=\frac{1-\gamma}{\sqrt{t}}$ for $t=1,\dots,T$. Choose $\tau^{*} = \min\{\tau:\sigma\rho^{\tau} \leq \alpha_T=\frac{1-\gamma}{\sqrt{T}} \} $. We rearrange \eqref{eq:PGSGDpf2} and take the summation over the time steps to obtain
\begin{align*}
     \sum_{t=1}^T&\mathbb{E}\left[\norm{\nabla J(\theta_t)}^2\right]\nonumber\\
     \leq& \sum_{t=1}^T\frac{\mathbb{E}[J(\theta_t)] - \mathbb{E}[J(\theta_{t+1})]}{\alpha_t} + \frac{c_J G_\infty^2}{2}\sum_{t=1}^T\alpha_t + \sum_{t=1}^T\mathbb E\left[ \nabla J(\theta_t)^T(\nabla J(\theta_t) - \nabla \tilde{J}(\theta_t;O_t))\right]\nonumber\\
     \leq& \frac{J(\theta_1)}{\alpha_1} + \sum_{t=2}^T\mathbb{E}[J(\theta_t)]\left(\frac{1}{\alpha_t} - \frac{1}{\alpha_{t-1}} \right) + \frac{c_J G_\infty^2}{2}\sum_{t=1}^T\alpha_t + \sum_{t=1}^T\mathbb E\left[ \nabla J(\theta_t)^T(\nabla J(\theta_t) - \nabla \tilde{J}(\theta_t;O_t))\right]\nonumber\\
     \leq& \frac{J(\theta_1)}{\alpha_1} + \frac{R_{\max}}{1-\gamma}\sum_{t=2}^T\left(\frac{1}{\alpha_t} - \frac{1}{\alpha_{t-1}} \right) + \frac{c_J G_\infty^2}{2}\sum_{t=1}^T\alpha_t + \sum_{t=1}^T\mathbb E\left[ \nabla J(\theta_t)^T(\nabla J(\theta_t) - \nabla \tilde{J}(\theta_t;O_t))\right]\nonumber\\
     \leq& \frac{J(\theta_1)}{\alpha_1} + \frac{R_{\max}}{1-\gamma}\cdot \frac{1}{\alpha_T} + \frac{c_J G_\infty^2}{2}\sum_{t=1}^T\alpha_t + \sum_{t=1}^T\mathbb E\left[ \nabla J(\theta_t)^T(\nabla J(\theta_t) - \nabla \tilde{J}(\theta_t;O_t))\right]\nonumber\\
     \leq& \frac{J(\theta_1)}{\alpha_1} + \frac{R_{\max}\sqrt{T}}{(1-\gamma)^2} + (1-\gamma)c_J G_\infty^2\sqrt{T} + \sum_{t=1}^{\tau^*}\mathbb E\left[ \nabla J(\theta_t)^T(\nabla J(\theta_t) - \nabla \tilde{J}(\theta_t;O_t))\right]\nonumber\\
     & + \sum_{t=\tau^*+1}^{T}\mathbb E\left[ \nabla J(\theta_t)^T(\nabla J(\theta_t) - \nabla \tilde{J}(\theta_t;O_t))\right]\\
     \overset{\text{(i)}}{\leq}& \frac{J(\theta_1)}{\alpha_1} + \frac{R_{\max}\sqrt{T}}{(1-\gamma)^2} + (1-\gamma)c_J G_\infty^2\sqrt{T} + 2(1-\gamma)G_\infty^2\tau^*  \nonumber\\
     & + (1-\gamma)G_\infty^2\sum_{t=\tau^*+1}^{T}\left[\frac{(3c_{J}+c_{\tilde{J}})\tau^*}{\sqrt{t-\tau^*}} + \frac{1}{\sqrt{T}} + \frac{L_\pi G_\infty}{2} \cdot \frac{(1+\tau^*)\tau^*}{\sqrt{t-\tau^*}} \right]\\
     \leq& \frac{J(\theta_1)}{1-\gamma} + \frac{R_{\max}\sqrt{T}}{(1-\gamma)^2} + (1-\gamma)G_\infty^2\left[2\tau^*  + \left[ c_J +  2(3c_{J}+c_{\tilde{J}}) \tau^* +  1 + L_\pi G_{\infty} (\tau^*+(\tau^*)^2)  \right]\sqrt{T}\right],
\end{align*}
where (i) follows from \eqref{eq:alphaSum}.
Then we divide both sides of the above inequality by $T$, and have
\begin{align*}
    &\underset{t\in [T]}{\min}\mathbb{E}\left[\norm{\nabla J(\theta_t)}^2\right] \leq
    \frac{1}{T}\sum_{t=1}^T\mathbb{E}\left[\norm{\nabla J(\theta_t)}^2\right]\\
    &\quad\leq \frac{J(\theta_1) + 2(1-\gamma)^2G_{\infty}^2\tau^*}{(1-\gamma)T} + \frac{R_{\max} + (1-\gamma)^3 G_\infty^2 \left[ c_J +  2(3c_{J}+c_{\tilde{J}}) \tau^* +  1 + L_\pi G_{\infty} (\tau^*+(\tau^*)^2)  \right]}{(1-\gamma)^2\sqrt{T}}.
\end{align*}

\section{Proof of Lemma \ref{lem:TDbias}}

We bound the expectation of bias via constructing a new Markov chain and applying useful techniques from information theory. Before deriving the bound, we first introduce some technical lemmas.

\begin{lemma}\label{lem:gbound2}
Given Assumption \ref{asp:boundedDomain}, the gradient $g_t$ in Algorithm \ref{alg:AMSGradTD} is uniformly bounded as follows,
\begin{equation}
    \norm{g_t}_\infty \leq \norm{g_t} \leq R_{\max} + (1+\gamma) D_{\infty}.
\end{equation}
\end{lemma}
\begin{proof}
The proof follows easily from the boundedness of $\phi$ and $\theta$.
\begin{align*}
    \norm{g_t} =& \norm{ \left(\phi(s'_t)^T\theta - r_t - \gamma\phi(s'_t)^T\theta \right)\phi(s_t) }\\
    \leq& \norm{ \phi(s'_t)^T\theta - r_t - \gamma\phi(s'_t)^T\theta } \cdot \norm{\phi(s_t)}\\
    \overset{\text{(i)}}{\leq}& \norm{ \phi(s'_t)^T\theta } + R_{\max} + \gamma\norm{ \phi(s'_t)^T\theta }\\
    \overset{\text{(ii)}}{\leq}& R_{\max} + (1+\gamma) D_{\infty},
\end{align*}
where (i) follows since the reward function is uniformly bounded and (ii) follows from Assumption \ref{asp:boundPhi} and Assumption \ref{asp:boundedDomain}.
\end{proof}

\begin{lemma}\label{lem:mvbound2}\cite[Lemma A.1]{zhou2018convergence}
Let $\{g_t, m_t, \hat{v}_t\}$ for $t=1,2,\dots$ be sequences generated by Algorithm \ref{alg:AMSGradTD} and denote $G_\infty = R_{\max} + (1+\gamma) D_{\infty}$. Given $g_t\leq G_\infty$ as in Lemma \ref{lem:gbound2}, we have $\norm{\bar g_t}\leq G_\infty, \norm{m_t} \leq G_\infty, \norm{\hat{v}_t} \leq G_\infty^2$.
\end{lemma}

\begin{lemma}\label{lem:biasLip}
\cite[Lemma 11]{bhandari2018finite} Let $\xi(\theta;s,a,s') = (g(\theta;s,a,s') - \bar g(\theta))^T(\theta - \theta^\star)$. Fix $(s,a,s')$. Then $\xi$ is uniformly bounded by
$$ \lvert \xi(\theta;s,a,s') \rvert \leq 2D_{\infty}G_{\infty}, \quad\forall \theta\in\mathcal{D},$$
and it is Lipschitz continuous as given by
$$ \lvert \xi(\theta;s,a,s')-\xi(\theta';s,a,s') \rvert \leq 2((1+\gamma)D_{\infty}+G_{\infty})\norm{\theta-\theta'}_2, \quad\forall \theta,\theta'\in\mathcal{D}. $$
\end{lemma}

\begin{lemma}\label{lem:biasMarkov}
\cite[Lemma 9]{bhandari2018finite} Consider two random variables $X$ and $Y$ such that 
\begin{equation}\label{eq:lemMar}
    X\rightarrow s_t \rightarrow s_{t+\tau} \rightarrow Y,
\end{equation} 
for fixed $ t $ and $\tau>0$. Suppose Assumption \ref{asp:markov} holds. Let $X',Y'$ are independent copies drawn from the marginal distribution of $X$ and $Y$, that is $\mathbb{P}(X'=\cdot,Y'=\cdot) = \mathbb{P}(X=\cdot)  \mathbb{P}(Y=\cdot)$. Then, 
for any bounded $v$, we have
$$ \lvert \mathbb{E}[v(X,Y)] - \mathbb{E}[v(X',Y')] \rvert \leq 2 \norm{v}_{\infty}(\sigma\rho^{\tau}),
$$
where $\sigma$ and $\rho$ are defined in Assumption \ref{asp:markov}.
\end{lemma}
\begin{remark}
The notation $X\rightarrow Z\rightarrow Y$ indicates that the random variable $X$ and $Y$ are independent conditioned on $Z$, which is a standard notation in information theory. 
\end{remark}

\textbf{Proof of Lemma \ref{lem:TDbias}:} Now we bound the bias of the gradient estimation. We first develop the connection between $\xi(\theta_t;s_t,a_t,s_{t+1})$ and $\xi(\theta_{t-\tau};s_t,a_t,s_{t+1})$ using Lemma \ref{lem:biasLip}. For notational brevity, we define a random tuple $O_t = (s_t, a_t, s_{t+1})$ and clearly $g_t=g(\theta, O_t)$. We denote 
$$ \xi_t(\theta) := \xi(\theta, O_t) = (g(\theta, O_t) - \bar g(\theta))^T(\theta - \theta^\star),
$$
where $\bar g(\theta) = \mathbb{E}[g(\theta, O_t)]$ as defined in \eqref{eq:meanGra}.

Next, we apply Lemma \ref{lem:conThetaBound} to obtain
$$ \norm{\theta_t - \theta_{t-\tau}}_2 \leq \sum_{i=t-\tau}^{t-1} \norm{\theta_{i+1} - \theta_i}_2\leq \frac{G_{\infty} }{G_0}\sum_{i=t-\tau}^{t-1} a_i.
$$
Thus we can relate $\xi_t(\theta_t)$ and $\xi_t(\theta_{t-\tau})$ by using the Lipschitz property in Lemma \ref{lem:biasLip} as follows.
\begin{align}\label{eq:lembiaspfstep1}
    \xi_t(\theta_t) - \xi_t(\theta_{t-\tau}) \leq& \lvert \xi_t(\theta_t) - \xi_t(\theta_{t-\tau}) \rvert\nonumber\\
    \leq& 2((1+\gamma)D_{\infty}+G_{\infty})\norm{\theta_t-\theta_{t-\tau}}\nonumber\\
    \leq& 2((1+\gamma)D_{\infty}+G_{\infty})\cdot \frac{G_{\infty} }{G_0}\sum_{i=t-\tau}^{t-1} a_i.
\end{align}

Next, we seek to bound $\mathbb{E}[\xi_t(\theta_{t-\tau})]$ using Lemma \ref{lem:biasMarkov}. Observe that given any deterministic $\theta\in\mathcal{D}$, we have 
$$ \mathbb{E}[\xi(\theta, O_t)] = (\mathbb{E}[g(\theta, O_t)] - \bar g(\theta))^T(\theta - \theta^\star) = 0.
$$
Since $\theta_0$ is non-random, we have $\mathbb{E}[\xi(\theta_0, O_t)] = 0$.
Now we are ready to bound $ \mathbb{E}[\xi(\theta_{t-\tau},O_t)]$ with Lemma \ref{lem:biasMarkov} via constructing a random process satisfying \eqref{eq:lemMar}. To do so, consider random variables $\theta'_{t-\tau}$ and $O'_t$ drawn independently from the marginal distribution of $\theta_{t-\tau}$ and $O_t$, thus $\mathbb{P}(\theta'_{t-\tau}=\cdot,O'_t=\cdot) = \mathbb{P}(\theta_{t-\tau}=\cdot)  \mathbb{P}(O_t=\cdot)$. Further, we can obtain $\mathbb{E}[\xi(\theta'_{t-\tau},O'_t)]=\mathbb{E}[\mathbb{E}[\xi(\theta'_{t-\tau},O'_t)]|\theta'_{t-\tau}] = 0 $ since $\theta'_{t-\tau}$ and $O'_t$ are independent. Combining Lemma \ref{lem:biasLip} and Lemma \ref{lem:biasMarkov} yields
\begin{equation}\label{eq:lembiaspfstep2}
    \mathbb{E}[\xi(\theta_{t-\tau},O_t)]\leq 2 (2D_{\infty}G_{\infty}) (\sigma\rho^{\tau}).
\end{equation}

Finally, we are ready to bound the bias. First, we take expectation on both sides of \eqref{eq:lembiaspfstep1} and obtain
$$ \mathbb{E}[\xi_t(\theta_t)] \leq \mathbb{E}[\xi_t(\theta_{t-\tau})] + 2((1+\gamma)D_{\infty}+G_{\infty})\cdot \frac{G_{\infty} }{G_0}\sum_{i=t-\tau}^{t-1} a_i.
$$
Let $\tau^*=\min\{\tau: \sigma\rho^\tau \leq \alpha_T\} $. When $t \leq \tau^{*}$, we choose $\tau = t$ and have
$$\begin{aligned}
\mathbb{E}[\xi_t(\theta_t)] \leq& \mathbb{E}[\xi_t(\theta_{0})] + 2((1+\gamma)D_{\infty}+G_{\infty})\cdot \frac{G_{\infty} }{G_0}\cdot t\alpha_0\\
\leq& 2((1+\gamma)D_{\infty}+G_{\infty})\cdot \frac{G_{\infty} }{G_0}\tau^*\alpha_0.
\end{aligned}
$$
When $t > \tau^{*}$, we choose $\tau = \tau^* $ and have
\begin{align*}
    \mathbb{E}[\xi_t(\theta_t)] \leq& \mathbb{E}[\xi_t(\theta_{t-\tau^*})] + 2((1+\gamma)D_{\infty}+G_{\infty})\cdot \frac{G_{\infty} }{G_0}\sum_{i=t-\tau}^{t-1} a_i\\
    \overset{\text{(i)}}{\leq}& 4D_{\infty}G_{\infty}(\sigma\rho^{\tau^*}) + 2((1+\gamma)D_{\infty}+G_{\infty})\cdot \frac{G_{\infty} }{G_0}\sum_{i=t-\tau}^{t-1} a_i\\
    \overset{\text{(ii)}}{\leq}& 4D_{\infty}G_{\infty}\alpha_T + 2((1+\gamma)D_{\infty}+G_{\infty})\cdot \frac{G_{\infty} }{G_0}\sum_{i=t-\tau}^{t-1} a_i,
\end{align*}
where (i) follows from \eqref{eq:lembiaspfstep2}, and (ii) follows due to the definition of the mixing time.

\section{Proof of Theorem \ref{thm:AMSGradTDcons}}

Differently from the regret bound for AMSGrad in conventional optimization~\cite{reddi2019convergence}, our analysis here focuses on the convergence rate. In fact, a slight modification of our proof also provides the convergence rate of AMSGrad for conventional strongly convex optimization, which can be of independent interest. Moreover, we provide results under the constant stepsize and under Marovian sampling, neither of which has been studied in \cite{reddi2019convergence}.


To proceed the proof, we first observe that
$$\theta_{t+1} = \Pi_{\mathcal{D}, \hat V_t^{1/4}} \left( \theta_t - \alpha_t \hat{V}_t^{-\frac{1}{2}}m_t\right)=\underset{\theta\in\mathcal{D}}{\min}\norm{\hat V_t^{1/4}\left(\theta_t - \alpha_t \hat{V}_t^{-\frac{1}{2}}m_t - \theta\right)}.$$
Clearly $\Pi_{\mathcal{D}, \hat V_t^{1/4}}(\theta^\star)=\theta^\star$ due to Assumption \ref{asp:boundedDomain}. We start from the update of $\theta_t$ when $t\geq 2$.
\begin{align*}
    \norm{\hat V_t^{1/4}(\theta_{t+1}-\theta^\star)}^2
    &= \norm{\Pi_{\mathcal{D}, \hat V_t^{1/4}}\hat V_t^{1/4}\left(\theta_{t} - \theta^\star - \alpha_t\hat{V}_{t}^{-\frac{1}{2}}m_t\right)}^2\\
    &\leq \norm{\hat V_t^{1/4}\left(\theta_{t} - \theta^\star - \alpha_t\hat{V}_{t}^{-\frac{1}{2}}m_t\right)}^2\\
    &= \norm{\hat V_t^{1/4}(\theta_{t}-\theta^\star)}^2 + \norm{\alpha_t\hat{V}_{t}^{-1/4}m_t}^2 - 2\alpha_t(\theta_t - \theta^\star)^T m_t\\
    &= \norm{\hat V_t^{1/4}(\theta_{t}-\theta^\star)}^2 + \norm{\alpha_t\hat{V}_{t}^{-1/4}m_t}^2 - 2\alpha_t(\theta_t - \theta^\star)^T (\beta_{1t}m_{t-1} + (1-\beta_{1t})g_{t})\\
    &\overset{\text{(i)}}{\leq}  \norm{\hat V_t^{1/4}(\theta_{t}-\theta^\star)}^2 + \norm{\alpha_t\hat{V}_{t}^{-1/4}m_t}^2+ \alpha_t\beta_{1t}\left( \frac{1}{\alpha_t}\norm{\hat V_t^{1/4}(\theta_{t}-\theta^\star)}^2+\alpha_t\norm{ \hat{V}_{t}^{-1/4}m_{t-1} }^2 \right) \\
    &\quad - 2\alpha_t(1-\beta_{1t})(\theta_t - \theta^\star)^T g_{t}\\
    &\overset{\text{(ii)}}{\leq} \norm{\hat V_t^{1/4}(\theta_{t}-\theta^\star)}^2 + \norm{\alpha_t\hat{V}_{t}^{-1/4}m_t}^2 + \beta_{1t}\norm{\hat V_t^{1/4}(\theta_{t}-\theta^\star)}^2+\alpha_t^2\beta_{1t}\norm{ \hat{V}_{t-1}^{-1/4}m_{t-1} }^2\\
    &\quad  - 2\alpha_t(1-\beta_{1t})(\theta_t - \theta^\star)^T g_{t},
\end{align*}
where (i) follows from Cauchy-Schwarz inequality, and (ii) holds because $\hat{v}_{t+1,i} \geq \hat{v}_{t,i},\forall t, \forall i$. 

Next, we take the expectation over all samples used up to time step $t$ on both sides, which still preserves the inequality.
Since we consider the Markovian sampling case where the estimation of gradient is biased. That is,
\begin{equation}\label{eq:pf2}
    \mathbb{E}\left[(\theta_t - \theta^\star)^T g_{t}\right] = \mathbb{E}\left[(\theta_t - \theta^\star)^T \bar g_{t}\right] + \mathbb{E}\left[(\theta_t - \theta^\star)^T (g_t - \bar g_{t})\right] \neq \mathbb{E}\left[(\theta_t - \theta^\star)^T \bar g_{t}\right],
\end{equation}
where $\bar g_t:= \bar g(\theta_t)$ and this notation will be used in the remaining for simplicity.

Thus we have
\begin{align*}
    &\mathbb{E}\norm{\hat V_t^{1/4}(\theta_{t+1}-\theta^\star)}^2\\
    &\leq \mathbb{E}\norm{\hat V_t^{1/4}(\theta_{t}-\theta^\star)}^2 + \alpha_t^2\mathbb{E}\norm{\hat{V}_{t}^{-1/4}m_t}^2 + \beta_{1t}\mathbb{E}\norm{\hat V_t^{1/4}(\theta_{t}-\theta^\star)}^2+ \alpha_t^2\beta_{1t}\mathbb{E}\norm{ \hat{V}_{t-1}^{-1/4}m_{t-1} }^2\\
    &\quad  - 2\alpha_t(1-\beta_{1t})\mathbb{E}\left[(\theta_t - \theta^\star)^T g_{t}\right]\\
    &= \mathbb{E}\norm{\hat V_t^{1/4}(\theta_{t}-\theta^\star)}^2 + \alpha_t^2\mathbb{E}\norm{\hat{V}_{t}^{-1/4}m_t}^2 + \beta_{1t}\mathbb{E}\norm{\hat V_t^{1/4}(\theta_{t}-\theta^\star)}^2+ \alpha_t^2\beta_{1t}\mathbb{E}\norm{ \hat{V}_{t-1}^{-1/4}m_{t-1} }^2\\
    &\quad  - 2\alpha_t(1-\beta_{1t})\mathbb{E}\left[(\theta_t - \theta^\star)^T \bar g_{t}\right] + 2\alpha_t(1-\beta_{1t})\mathbb{E}\left[(\theta_t - \theta^\star)^T (\bar g_t - g_{t})\right]\\
    &\overset{\text{(i)}}{\leq} \mathbb{E}\norm{\hat V_t^{1/4}(\theta_{t}-\theta^\star)}^2 + \alpha_t^2\mathbb{E}\norm{\hat{V}_{t}^{-1/4}m_t}^2 + \beta_{1t}\mathbb{E}\norm{\hat V_t^{1/4}(\theta_{t}-\theta^\star)}^2+ \alpha_t^2\beta_{1t}\mathbb{E}\norm{ \hat{V}_{t-1}^{-1/4}m_{t-1} }^2\\
    &\quad  - 2\alpha_t c(1-\beta_{1t})\mathbb{E}\norm{\theta_t - \theta^\star}^2 + 2\alpha_t(1-\beta_{1t})\mathbb{E}\left[(\theta_t - \theta^\star)^T (\bar g_t - g_{t})\right]\\
    &\overset{\text{(ii)}}{\leq} \mathbb{E}\norm{\hat V_t^{1/4}(\theta_{t}-\theta^\star)}^2 + \alpha_t^2\mathbb{E}\norm{\hat{V}_{t}^{-1/4}m_t}^2 + \beta_{1t}\mathbb{E}\norm{\hat V_t^{1/4}(\theta_{t}-\theta^\star)}^2+ \alpha_t^2\beta_{1}\mathbb{E}\norm{ \hat{V}_{t-1}^{-1/4}m_{t-1} }^2\\
    &\quad  - 2\alpha_t c(1-\beta_{1})\mathbb{E}\norm{\theta_t - \theta^\star}^2 + 2\alpha_t(1-\beta_{1t})\mathbb{E}\left[(\theta_t - \theta^\star)^T (g_t - \bar g_{t})\right]\\
    &\overset{\text{(iii)}}{\leq} \mathbb{E}\norm{\hat V_t^{1/4}(\theta_{t}-\theta^\star)}^2 + \alpha_t^2\mathbb{E}\norm{\hat{V}_{t}^{-1/4}m_t}^2 + G_\infty D_\infty^2\beta_{1t}+ \alpha_t^2\beta_{1}\mathbb{E}\norm{ \hat{V}_{t-1}^{-1/4}m_{t-1} }^2\\
    &\quad  - 2\alpha_t c(1-\beta_{1})\mathbb{E}\norm{\theta_t - \theta^\star}^2 + 2\alpha_t(1-\beta_{1t})\mathbb{E}\left[(\theta_t - \theta^\star)^T (g_t - \bar g_{t})\right],
\end{align*}
where (i) follows from Lemma \ref{lem:TDStrongConv} and because $1-\beta_{1t}>0$, (ii) follows because $\beta_{1t}<\beta_1<1$ and $\mathbb{E}\norm{\theta_t - \theta^\star}^2>0$, and (iii) follows from $\norm{\hat V_t^{1/4}(\theta_{t}-\theta^\star)}^2\leq \norm{\hat V_t^{1/4}}_2^2\norm{\theta_{t}-\theta^\star}^2 \leq G_\infty D_\infty^2$ by Lemma \ref{lem:mvbound2} and Assumption \ref{asp:boundedDomain}.
By rearranging the terms in the above inequality and taking the summation over time steps, we have
\begin{align*}
2c(1&-\beta_1)\sum_{t=2}^T\mathbb{E}\norm{\theta_t - \theta^\star}^2\\
&\leq \sum_{t=2}^T \frac{1}{\alpha_t} \left( \mathbb{E}\norm{\hat V_t^{1/4}(\theta_{t}-\theta^\star)}^2\! -\! \mathbb{E}\norm{\hat V_t^{1/4}(\theta_{t+1}-\theta^\star)}^2 \right) + \sum_{t=2}^T \frac{\beta_{1t}G_\infty D_\infty^2}{\alpha_t}\\
&\quad +  \sum_{t=2}^{T} \alpha_t \mathbb{E} \norm{ \hat{V}_{t}^{-1/4}m_t }^2 + \sum_{t=2}^{T} \alpha_t\beta_{1} \mathbb{E} \norm{ \hat{V}_{t-1}^{-1/4}m_{t-1} }^2 \\
&\quad + 2\sum_{t=2}^{T}(1-\beta_{1t})\mathbb{E}\left[(\theta_t - \theta^\star)^T (g_t - \bar g_{t})\right]\\
&\overset{\text{(i)}}{\leq} \sum_{t=2}^T \frac{1}{\alpha_t} \left( \mathbb{E}\norm{\hat V_t^{1/4}(\theta_{t}-\theta^\star)}^2\! -\! \mathbb{E}\norm{\hat V_t^{1/4}(\theta_{t+1}-\theta^\star)}^2 \right) + \sum_{t=2}^T \frac{\beta_{1t}G_\infty D_\infty^2}{\alpha_t}\\
&\quad +  \sum_{t=2}^{T} \alpha_t \mathbb{E} \norm{ \hat{V}_{t}^{-1/4}m_t }^2 + \sum_{t=2}^{T} \alpha_{t-1}\beta_{1} \mathbb{E} \norm{ \hat{V}_{t-1}^{-1/4}m_{t-1} }^2 \\ 
&\quad + 2\sum_{t=2}^{T}(1-\beta_{1t})\mathbb{E}\left[(\theta_t - \theta^\star)^T (g_t - \bar g_{t})\right]\\
&\leq \sum_{t=2}^T \frac{1}{\alpha_t} \left( \mathbb{E}\norm{\hat V_t^{1/4}(\theta_{t}-\theta^\star)}^2\! -\! \mathbb{E}\norm{\hat V_t^{1/4}(\theta_{t+1}-\theta^\star)}^2 \right) + \sum_{t=2}^T \frac{\beta_{1t}G_\infty D_\infty^2}{\alpha_t}\\
&\quad +  (1+\beta_1)\sum_{t=1}^{T} \alpha_t \mathbb{E} \norm{ \hat{V}_{t}^{-1/4}m_t }^2 + 2\sum_{t=2}^{T}(1-\beta_{1t})\mathbb{E}\left[(\theta_t - \theta^\star)^T (g_t - \bar g_{t})\right],
\end{align*}
where (i) follows from $\alpha_t \leq \alpha_{t-1}$. By further implementing the first term in the right-hand side of the last inequality, we can then bound the sum as
\begin{align*}
2c(1&-\beta_1)\sum_{t=2}^T\mathbb{E}\norm{\theta_t - \theta^\star}^2\\
&\leq \sum_{t=2}^T \frac{1}{\alpha_t} \mathbb{E}\left( \norm{\hat V_t^{1/4}(\theta_{t}-\theta^\star)}^2 - \norm{\hat V_t^{1/4}(\theta_{t+1}-\theta^\star)}^2 \right) + \sum_{t=2}^T \frac{\beta_{1t}G_\infty D_\infty^2}{\alpha_t}\\
&\quad + (1+\beta_1)\sum_{t=1}^{T} \alpha_t \mathbb{E} \norm{ \hat{V}_{t}^{-1/4}m_t }^2 + 2\sum_{t=2}^{T}\mathbb{E}\left[(\theta_t - \theta^\star)^T (g_t - \bar g_{t})\right]\\
& = \frac{\mathbb{E}\norm{\hat V_2^{1/4}(\theta_{2}-\theta^\star)}^2}{\alpha_2} + \sum_{t=3}^T \mathbb{E}\left( \frac{\norm{\hat V_t^{1/4}(\theta_{t}-\theta^\star)}^2}{\alpha_t} - \frac{\norm{\hat V_{t-1}^{1/4}(\theta_{t}-\theta^\star)}^2}{\alpha_{t-1}}  \right)\\
&\quad - \frac{\mathbb{E}\norm{\hat V_{T}^{1/4}(\theta_{T+1}-\theta^\star)}^2}{\alpha_{T}} + \sum_{t=2}^T \frac{\beta_{1t}G_\infty D_\infty^2}{\alpha_t} + (1+\beta_1)\sum_{t=1}^{T} \alpha_t\mathbb{E}  \norm{ \hat{V}_{t}^{-1/4}m_t }^2\\
&\quad + 2\sum_{t=2}^{T}\mathbb{E}\left[(\theta_t - \theta^\star)^T (g_t - \bar g_{t})\right]\\
& = \frac{\mathbb{E}\norm{\hat V_2^{1/4}(\theta_{2}-\theta^\star)}^2}{\alpha_2} + \sum_{t=3}^T \mathbb{E}\left( \frac{\sum_{i=1}^d\hat{v}_{t,i}^{1/2}(\theta_{t,i} - \theta_{i}^\star)^2}{\alpha_t} - \frac{\sum_{i=1}^d\hat{v}_{t-1,i}^{1/2}(\theta_{t,i} - \theta_{i}^\star)^2}{\alpha_{t-1}}  \right)\\
&\quad - \frac{\mathbb{E}\norm{\hat V_{T}^{1/4}(\theta_{T+1}-\theta^\star)}^2}{\alpha_{T}} + \sum_{t=2}^T \frac{\beta_{1t}G_\infty D_\infty^2}{\alpha_t} + (1+\beta_1)\sum_{t=1}^{T} \alpha_t \mathbb{E} \norm{ \hat{V}_{t}^{-1/4}m_t }^2.\\
&\quad + 2\sum_{t=2}^{T}\mathbb{E}\left[(\theta_t - \theta^\star)^T (g_t - \bar g_{t})\right]\\
& = \frac{\mathbb{E}\norm{\hat V_2^{1/4}(\theta_{2}-\theta^\star)}^2}{\alpha_2} + \sum_{t=3}^T\sum_{i=1}^d \mathbb{E}(\theta_{t,i} - \theta_{i}^\star)^2\left( \frac{\hat{v}_{t,i}^{1/2}}{\alpha_t} - \frac{\hat{v}_{t-1,i}^{1/2}}{\alpha_{t-1}}  \right)\\
&\quad - \frac{\mathbb{E}\norm{\hat V_{T}^{1/4}(\theta_{T+1}-\theta^\star)}^2}{\alpha_{T}} + \sum_{t=2}^T \frac{\beta_{1t}G_\infty D_\infty^2}{\alpha_t} + (1+\beta_1)\sum_{t=1}^{T} \alpha_t \mathbb{E} \norm{ \hat{V}_{t}^{-1/4}m_t }^2\\
&\quad + 2\sum_{t=2}^{T}\mathbb{E}\left[(\theta_t - \theta^\star)^T (g_t - \bar g_{t})\right].
\end{align*}
Next, we further bound the above term as
\begin{align}\label{eq:tdThmpf1}
2c(1&-\beta_1)\sum_{t=2}^T\mathbb{E}\norm{\theta_t - \theta^\star}^2\nonumber\\
&\overset{\text{(i)}}{\leq}\frac{\mathbb{E}\norm{\hat V_2^{1/4}(\theta_{2}-\theta^\star)}^2}{\alpha_2} + D_\infty^2\sum_{t=3}^T\sum_{i=1}^d \mathbb{E}\left( \frac{\hat{v}_{t,i}^{1/2}}{\alpha_t} - \frac{\hat{v}_{t-1,i}^{1/2}}{\alpha_{t-1}}  \right)\nonumber\\
&\quad - \frac{\mathbb{E}\norm{\hat V_{T}^{1/4}(\theta_{T+1}-\theta^\star)}^2}{\alpha_{T}} + \sum_{t=2}^T \frac{\beta_{1t}G_\infty D_\infty^2}{\alpha_t} + (1+\beta_1)\sum_{t=1}^{T} \alpha_t \mathbb{E} \norm{ \hat{V}_{t}^{-1/4}m_t }^2\nonumber\\
&\quad + 2\sum_{t=2}^{T}\mathbb{E}\left[(\theta_t - \theta^\star)^T (g_t - \bar g_{t})\right]\nonumber\\
&\leq \frac{\mathbb{E}\norm{\hat V_2^{1/4}(\theta_{2}-\theta^\star)}^2}{\alpha_2} + D_\infty^2 \sum_{i=1}^d \mathbb{E} \frac{\hat{v}_{T,i}^{1/2}}{\alpha_T} + \sum_{t=2}^T \frac{\beta_{1t}G_\infty D_\infty^2}{\alpha_t} + (1+\beta_1)\sum_{t=1}^{T} \alpha_t \mathbb{E} \norm{ \hat{V}_{t}^{-1/4}m_t }^2\nonumber\\
&\quad + 2\sum_{t=2}^{T}\mathbb{E}\left[(\theta_t - \theta^\star)^T (g_t - \bar g_{t})\right],
\end{align}
where (i) follows from Assumption \ref{asp:boundedDomain} and because $\frac{\hat{v}_{t,i}^{1/2}}{\alpha_t} \geq \frac{\hat{v}_{t-1,i}^{1/2}}{\alpha_{t-1}}$. 

Since we consider the case with a constant stepsize $\alpha_t = \alpha$, we continue to bound \eqref{eq:tdThmpf1} as follows
\begin{align}\label{eq:tdThmpf2}
    2c(1&-\beta_1)\sum_{t=2}^T\mathbb{E}\norm{\theta_t - \theta^\star}^2\nonumber\\
    &\leq \frac{\mathbb{E}\norm{\hat V_2^{1/4}(\theta_{2}-\theta^\star)}^2}{\alpha_2} + D_\infty^2 \sum_{i=1}^d \mathbb{E} \frac{\hat{v}_{T,i}^{1/2}}{\alpha_T} + \sum_{t=2}^T \frac{\beta_{1t}G_\infty D_\infty^2}{\alpha_t} + (1+\beta_1)\sum_{t=1}^{T} \alpha_t \mathbb{E} \norm{ \hat{V}_{t}^{-1/4}m_t }^2\nonumber\\
    &\quad + 2\sum_{t=2}^{T}\mathbb{E}\left[(\theta_t - \theta^\star)^T (g_t - \bar g_{t})\right]\nonumber\\
    &\overset{\text{(i)}}{\leq} \frac{ G_\infty D_\infty^2}{\alpha} + \frac{ G_\infty D_\infty^2}{\alpha}  + \frac{\beta_1 \lambda G_\infty D_\infty^2}{\alpha(1-\lambda)} + \frac{\alpha(1+\beta_1)G_{\infty}^2}{G_0}\cdot T + 2\sum_{t=2}^{T}\mathbb{E}\left[(\theta_t - \theta^\star)^T (g_t - \bar g_{t})\right]\nonumber\\
    &= \frac{ 2G_\infty D_\infty^2}{\alpha}  + \frac{\beta_1 \lambda G_\infty D_\infty^2}{\alpha(1-\lambda)} + \frac{\alpha(1+\beta_1)G_{\infty}^2}{G_0}\cdot T + 2\sum_{t=2}^{T}\mathbb{E}\left[(\theta_t - \theta^\star)^T (g_t - \bar g_{t})\right],
\end{align}
where (i) follows because
\begin{align*}
    \norm{ \hat{V}_{t}^{-1/4}m_t }^2=\sum_{i=1}^d \frac{m_{t,i}^2}{\sqrt{\hat v_{t,i}}} \leq \frac{1}{G_0}\sum_{i=1}^d m_{t,i}^2\leq \frac{G_{\infty}^2}{G_0}. 
\end{align*}

Then we are ready to obtain the upper bound by applying Lemma \ref{lem:TDbias}. Choosing $\tau = t$ if $t\leq\tau^{*}$ and $\tau = \tau^{*}$ if $t>\tau^{*}$, we obtain
\begin{align*}
2c&(1-\beta_1)\sum_{t=2}^T\mathbb{E}\norm{\theta_t - \theta^\star}^2\\
& \leq \frac{ 2G_\infty D_\infty^2}{\alpha} + \frac{\beta_1 \lambda G_\infty D_\infty^2}{\alpha(1-\lambda)} + \frac{\alpha(1+\beta_1)G_{\infty}^2}{G_0}\cdot T + 2\sum_{t=2}^{T}\mathbb{E}\left[(\theta_t - \theta^\star)^T (g_t - \bar g_{t})\right]\\
& \leq \frac{ 2G_\infty D_\infty^2}{\alpha} + \frac{\beta_1 \lambda G_\infty D_\infty^2}{\alpha(1-\lambda)} + \frac{\alpha(1+\beta_1)G_{\infty}^2}{G_0}\cdot T +  4\sum_{t=2}^{\tau^*}\left[((1+\gamma)D_{\infty}+G_{\infty})\cdot \frac{G_{\infty} }{G_0}\tau^*\alpha\right]\\
&\quad + 4\sum_{t=\tau^*+1}^{T}\left[2D_{\infty}G_{\infty}\alpha + ((1+\gamma)D_{\infty}+G_{\infty})\cdot \frac{G_{\infty} }{G_0}\tau^*\alpha\right]\\
& \leq \frac{ 2G_\infty D_\infty^2}{\alpha} + \frac{\beta_1 \lambda G_\infty D_\infty^2}{\alpha(1-\lambda)} + \frac{\alpha(1+\beta_1)G_{\infty}^2}{G_0}\cdot T +  4((1+\gamma)D_{\infty}+G_{\infty})\cdot \frac{G_{\infty} }{G_0}(\tau^*)^2\alpha\\
&\quad + 4\left[2D_{\infty}G_{\infty}\alpha + ((1+\gamma)D_{\infty}+G_{\infty})\cdot \frac{G_{\infty} }{G_0}\tau^*\alpha\right](T-\tau^*).
\end{align*}

Finally, applying Jensen's inequality yields
\begin{align}\label{eq:pfthm1}
    \mathbb{E}\norm{\theta_{out} - \theta^\star}^2 \leq \frac{1}{T}\sum_{t=1}^T\mathbb{E}\norm{\theta_t - \theta^\star}^2\leq \frac{C_1}{T} + \alpha C_2,
\end{align}
where
\begin{align*}
    & C_1 = \frac{ G_\infty D_\infty^2}{\alpha c(1-\beta)} + \frac{\beta_1 \lambda G_\infty D_\infty^2}{2\alpha c(1-\lambda)(1-\beta)} +  2((1+\gamma)D_{\infty}+G_{\infty})\cdot \frac{G_{\infty} }{c G_0(1-\beta)}(\tau^*)^2\alpha, \\
    & C_2 = \frac{4D_{\infty}G_{\infty}}{c(1-\beta)} + \frac{2G_{\infty}\tau^* ((1+\gamma)D_{\infty}+G_{\infty})}{c G_0(1-\beta)}  + \frac{(1+\beta)G_{\infty}^2}{2c G_0(1-\beta)}.
\end{align*}

\section{Proof of Theorem \ref{thm:AMSGradTDdimi} }

We first provide the following useful lemma.
\begin{lemma}\label{lem:seqSum}
Let $\alpha_t=\frac{\alpha}{\sqrt{t}}$ and $\beta_{1t} = \beta_1 \lambda^t$ for $t=1,2,\dots$. Then
\begin{equation}
    \sum_{t=1}^T \frac{\beta_{1t}}{\alpha_t} \leq \frac{\beta_1}{\alpha(1-\lambda)^2}.
\end{equation}
\end{lemma}
\begin{proof}
The proof follows from taking the standard sum of geometric sequences.
\begin{equation}
    \sum_{t=1}^T \frac{\beta_{1t}}{\alpha_t} = \sum_{t=1}^T \frac{\beta_{1t}\sqrt{t}}{\alpha} \leq \sum_{t=1}^T \frac{\beta_{1}\lambda^{t-1}t} {\alpha} = \frac{\beta_1}{\alpha} \left( \frac{1}{(1-\lambda)}\sum_{t=1}^T \lambda^{t-1} -T\lambda^T \right) \leq \frac{\beta_1}{\alpha(1-\lambda)^2}.
\end{equation}
\end{proof}

\textbf{Proof of Theorem \ref{thm:AMSGradTDdimi}:}
The proof starts with steps similar to those of Theorem \ref{thm:AMSGradTDcons}. The difference begins from \eqref{eq:tdThmpf2}, where we now consider a diminishing stepsize $\alpha_t=\frac{\alpha}{\sqrt{t}}$. We then have
\begin{align*}
    2c&(1-\beta_1)\sum_{t=2}^T\mathbb{E}\norm{\theta_t - \theta^\star}^2\\
    &\leq \frac{\mathbb{E}\norm{\hat V_2^{1/4}(\theta_{2}-\theta^\star)}^2}{\alpha_2} + D_\infty^2 \sum_{i=1}^d \mathbb{E} \frac{\hat{v}_{T,i}^{1/2}}{\alpha_T} + \sum_{t=2}^T \frac{\beta_{1t}G_\infty D_\infty^2}{\alpha_t} + (1+\beta_1)\sum_{t=1}^{T} \alpha_t \mathbb{E} \norm{ \hat{V}_{t}^{-1/4}m_t }^2\\
    &\quad + 2\sum_{t=2}^{T}\mathbb{E}\left[(\theta_t - \theta^\star)^T (g_t - \bar g_{t})\right]\\
    & \overset{\text{(i)}}{\leq} \frac{G_\infty D_\infty^2}{\alpha_2} + \frac{ G_\infty D_\infty^2\sqrt{T}}{\alpha}  + \frac{\beta_1G_\infty D_\infty^2}{\alpha(1-\lambda)^2} + \frac{2\alpha(1+\beta_1)G_{\infty}^2}{G_0}\cdot \sqrt{T} + 2\sum_{t=2}^{T}\mathbb{E}\left[(\theta_t - \theta^\star)^T (g_t - \bar g_{t})\right],
\end{align*}
where (i) follows from Assumption \ref{asp:boundedDomain}, and Lemmas \ref{lem:mvbound2} and \ref{lem:seqSum}.

Then we are ready to obtain the upper bound by applying Lemma \ref{lem:TDbias}. Choosing $\tau = t$ if $t\leq\tau^{*}$ and $\tau = \tau^{*}$ if $t>\tau^{*}$, we obtain
\begin{align*}
2c&(1-\beta_1)\sum_{t=2}^T\mathbb{E}\norm{\theta_t - \theta^\star}^2\\
& \leq \frac{G_\infty D_\infty^2}{\alpha_2} + \frac{ G_\infty D_\infty^2\sqrt{T}}{\alpha}  + \frac{\beta_1G_\infty D_\infty^2}{\alpha(1-\lambda)^2} + \frac{2\alpha(1+\beta_1)G_{\infty}^2}{G_0}\cdot \sqrt{T} + 2\sum_{t=2}^{T}\mathbb{E}\left[(\theta_t - \theta^\star)^T (g_t - \bar g_{t})\right]\\
& \leq \frac{G_\infty D_\infty^2}{\alpha_2} + \frac{ G_\infty D_\infty^2\sqrt{T}}{\alpha}  + \frac{\beta_1G_\infty D_\infty^2}{\alpha(1-\lambda)^2} + \frac{2\alpha(1+\beta_1)G_{\infty}^2}{G_0}\cdot \sqrt{T} +4\sum_{t=2}^{\tau^*}\left[((1+\gamma)D_{\infty}+G_{\infty})\cdot \frac{G_{\infty} }{G_0}\tau^*\alpha\right]\\
&\quad + 4\sum_{t=\tau^*+1}^{T}\left[2D_{\infty}G_{\infty}\frac{\alpha}{\sqrt{T}} + ((1+\gamma)D_{\infty}+G_{\infty})\cdot \frac{G_{\infty} }{G_0}\frac{\tau^*\alpha}{\sqrt{t-\tau^*}}\right]\\
& \leq \frac{G_\infty D_\infty^2}{\alpha_2} + \frac{ G_\infty D_\infty^2\sqrt{T}}{\alpha}  + \frac{\beta_1G_\infty D_\infty^2}{\alpha(1-\lambda)^2} + \frac{2\alpha(1+\beta_1)G_{\infty}^2}{G_0}\cdot \sqrt{T} + 4\left[((1+\gamma)D_{\infty}+G_{\infty})\cdot \frac{G_{\infty} }{G_0}\alpha(\tau^*)^2 \right]\\
&\quad + 4\left[2D_{\infty}G_{\infty} \alpha\sqrt{T} + ((1+\gamma)D_{\infty}+G_{\infty})\cdot \frac{2G_{\infty}\tau^*\alpha\sqrt{T-\tau^*} }{G_0}\right].
\end{align*}

Finally, applying Jensen's inequality completes our proof as
\begin{align*}
    \mathbb{E}\norm{\theta_{out} - \theta^\star}^2 \leq \frac{1}{T}\sum_{t=1}^T\mathbb{E}\norm{\theta_t - \theta^\star}^2 \leq \frac{C_1}{\sqrt{T}} + \frac{C_2}{T},
\end{align*}
where
\begin{align*}
    & C_1 = \frac{ G_\infty D_\infty^2 }{2c\alpha (1-\beta)} +  \frac{\alpha (1+\beta_1)G_{\infty}^2}{c G_0(1-\beta)} + \frac{4\alpha D_{\infty}G_{\infty}}{c (1-\beta)} + \frac{4\tau^*\alpha G_{\infty}((1+\gamma)D_{\infty}+G_{\infty}) }{c G_0(1-\beta)}, \\
    & C_2 = \frac{ G_\infty D_\infty^2 }{\sqrt{2}c\alpha (1-\beta)} + \frac{\beta G_\infty D_\infty^2}{2c\alpha(1-\lambda)^2(1-\beta)} + \frac{2G_{\infty}\alpha(\tau^*)^2 ((1+\gamma)D_{\infty}+G_{\infty}) }{c G_0(1-\beta)}.
\end{align*}

\end{document}